\documentclass[journal,twoside]{IEEEtran}
\newcommand{\GD}{GD}

\newcommand{\DQGD}{DQ-GD}
\newcommand{\NQGD}{NQ-GD}

\usepackage{comment}

\usepackage{graphics,graphicx}
\usepackage{caption,subcaption}
\usepackage{color,xcolor}
\usepackage{tikz}
\usetikzlibrary{
   shapes,
   shapes.misc,
   arrows,
   arrows.meta,
   positioning
}

\usepackage{array}
\usepackage{multirow,multicol}
\usepackage{enumerate}

\usepackage{amsmath,amssymb}
\usepackage{amsthm}

\usepackage{mathrsfs}
\usepackage{dsfont}
\usepackage{bm,bbm}

\usepackage{empheq}
\usepackage{mathtools}
\usepackage[algo2e,ruled,linesnumbered]{algorithm2e}

\usepackage{booktabs}
\usepackage[noadjust]{cite}
\usepackage{url}
\usepackage{hyperref}
\hypersetup{backref, colorlinks=true, linkcolor=blue, citecolor=blue, urlcolor = blue}
\newtheorem{theorem}{Theorem}[section]

\newtheorem{lemma}{Lemma}[section]

\theoremstyle{definition}
\newtheorem{definition}{Definition}[section]

\theoremstyle{remark}
\newtheorem{remark}{Remark}[section]

\newcommand{\mbb}{\mathbb}
\newcommand{\msf}{\mathsf}
\newcommand{\mbf}{\mathbf}
\newcommand{\mrm}{\mathrm}

\newcommand{\mcal}{\mathcal}

\renewcommand{\vec}[1]{\bm{#1}}

\newcommand{\set}[1]{\mathcal{#1}}
\newcommand{\fcn}[1]{\mathsf{#1}}

\newcommand{\pr}[1]{\left( {#1}\right)}
\newcommand{\bk}[1]{\left[ {#1}\right]}
\newcommand{\bp}[1]{\left\{ {#1}\right\}}

\newcommand{\norm}[1]{\left\lVert {#1}\right\rVert}
\newcommand{\abs}[1]{\left\lvert {#1}\right\rvert}

\newcommand{\triDef}{\triangleq}

\newcommand{\R}{\mathbb{R}}
\newcommand{\N}{\mathbb{N}}

\newcommand{\T}{\top}
\newcommand{\grad}{\nabla}

\newcommand{\lp}{\left(}
\newcommand{\rp}{\right)}
\newcommand{\lb}{\left[}
\newcommand{\rb}{\right]}
\newcommand{\lbp}{\left\{}
\newcommand{\rbp}{\right\}}
\newcommand{\lba}{\left\lvert}
\newcommand{\rba}{\right\rvert}
\newcommand{\lnorm}{\left\lVert}
\newcommand{\rnorm}{\right\rVert}

\newcommand{\la}{\leftarrow}

\newcommand{\eqDef}{\triangleq}

\newcommand{\Mand}{\quad\text{ and }\quad}

\newcommand{\by}{\times}

\newcommand{\spn}{\mathrm{span}}

\newcommand{\minus}{\scalebox{0.75}[1.0]{$-$}}
\DeclarePairedDelimiterX{\infdivx}[2]{(}{)}{%
  #1\:\delimsize\|\:#2%
}

\DeclareMathOperator*{\argmin}{arg\,min}

\makeatletter

\newcommand{\Rmnum}[1]{\expandafter\@slowromancap\romannumeral #1@}
\makeatother

\makeatletter
\newcommand*{\transpose}{%
  {\mathpalette\@transpose{}}%
}
\newcommand*{\@transpose}[2]{%
  \raisebox{\depth}{$\m@th#1\intercal$}%
}
\makeatother

\newcommand{\thmref}[1]{Theorem~\ref{#1}}

\newcommand{\secref}[1]{Section~\ref{#1}}
\newcommand{\lemref}[1]{Lemma~\ref{#1}}

\newcommand{\figref}[1]{Fig.~\ref{#1}}

\ifCLASSINFOpdf
\else
\fi
\usepackage{psfrag}
\psfrag{NQ-GD}{\footnotesize{contraction factor}}

\newcommand{\hl}[1]{#1}

\begin{document}
\title{Differentially Quantized Gradient Methods}

\author{Chung-Yi~Lin,
        Victoria~Kostina,
        and~Babak~Hassibi% <-this % stops a space
 \thanks{Chung-Yi Lin (\href{mailto:hsnu1220@gmail.com}{hsnu1220@gmail.com}) is with Kronos
Research, Taiwan. V. Kostina (\href{mailto:vkostina@caltech.edu}{vkostina@caltech.edu}) and B. Hassibi (\href{mailto:hassibi@caltech.edu}{hassibi@caltech.edu}) are with California Institute of Technology. This work was supported in part by the National Science Foundation (NSF) under grants CCF-1751356, CCF-1956386, CNS-0932428, CCF-1018927, CCF-1423663 and CCF-1409204, by a grant from Qualcomm Inc., by NASA's Jet Propulsion Laboratory through the President and Director's 
 Fund, and by King Abdullah University of Science and Technology. This paper was presented in part at ISIT 2021 \cite{lin2021dq}. 
}% <-this % stops a space
}

\maketitle

\begin{abstract}
Consider the following distributed optimization scenario. A worker has access to training data that it uses to compute the gradients while a server decides when to stop iterative computation based on its target accuracy or delay constraints. The server receives all its information about the problem instance from the worker via a rate-limited noiseless communication channel. 

We introduce the principle we call \emph{differential quantization} (DQ) that prescribes compensating the past quantization errors to direct the descent trajectory of a quantized algorithm towards that of its unquantized counterpart. Assuming that the objective function is smooth and strongly convex, we prove that \emph{differentially quantized gradient descent} (DQ-GD) attains a linear contraction factor of $\max\{\sigma_{\mathrm{GD}}, \rho_n 2^{-R}\}$, where $\sigma_{\mathrm{GD}}$ is the contraction factor of unquantized gradient descent (GD), $\rho_n \geq 1$ is the covering efficiency of the quantizer, and $R$ is the bitrate per problem dimension $n$. Thus at any $R\geq\log_2 \rho_n /\sigma_{\mathrm{GD}}$ bits, the contraction factor of DQ-GD is the same as that of unquantized {\GD}, i.e., there is no loss due to quantization. We show a converse demonstrating that no algorithm within a certain class can converge faster than $\max\{\sigma_{\mathrm{GD}}, 2^{-R}\}$. Since quantizers exist with $\rho_n \to 1$ as $n \to \infty$ (Rogers, 1963), this means that DQ-GD is asymptotically optimal. In contrast, naively quantized GD where the worker directly quantizes the gradient barely attains $\sigma_{\mathrm{GD}} + \rho_n2^{-R}$. 

The principle of differential quantization continues to apply to gradient methods with momentum such as Nesterov's accelerated gradient descent, and Polyak's heavy ball method.  For these algorithms as well, if the rate is above a certain threshold, there is no loss in contraction factor obtained by the differentially quantized algorithm compared to its unquantized counterpart, and furthermore, the differentially quantized heavy ball method attains the optimal contraction achievable among all (even unquantized) gradient methods. 

Experimental results on least-squares problems validate our theoretical analysis. 
\end{abstract}

 \begin{IEEEkeywords}
 gradient descent, quantized gradient descent, accelerated gradient descent, heavy ball method, error compensation, error feedback, sigma-delta modulation, federated learning, linear convergence.
\end{IEEEkeywords}

% \fi
\IEEEpeerreviewmaketitle

\section{Introduction}
\label{sec:intro}

\subsection{Motivation and related work}

Distributed optimization plays a central role in large-scale machine learning where gradient descent (GD) and its stochastic variant SGD are employed to minimize an objective function \cite{Zinkevich-10,Recht-11,Bekkerman-11,Dean-12,Chilimbi-14,Sa-15,Konecny-16,Scaman-17}. Despite the scalability of parallel gradient training, the frequent exchange of high-dimensional gradients \hl{between distributed agents in the federated learning setting} has become a communication bottleneck that slows down the overall learning process \cite{Recht-11,Chilimbi-14,Seide-14,Li-14-communication,Strom-15,Zhang-15} .

\hl{
A natural approach to alleviating that communication bottleneck is to quantize the gradients with a limited number of bits per problem dimension. Its power was first demonstrated  by Seide et al. \cite{Seide-14}, where the gradient computed by stochastic gradient descent (SGD) \cite{RobbinsSutton-51} is quantized down to just one bit per dimension and the quantization error is carried forward across mini-batches, resulting in almost no loss in empirical convergence performance compared to the unquantized algorithm. Wen at al.~\cite{Wen-17} propose a ternary quantizer for SGD and prove that it converges almost surely under the assumption of bounded gradients. Bernstein et al. \cite{Bernstein-18} propose a sign-based quantizer for mini-batch SGD, give its convergence analysis on nonconvex problems, and extend it to accelerated gradient descent and to a multi-worker setting. Alistarh et al. \cite{Alistarh-17} propose a quantized SGD algorithm that compresses the gradient using a stochastic scalar quantizer with an adjustable number of quantization levels, and provide convergence guarantees that depend on this variable compression rate on smooth convex and non-convex functions. The quantizer in \cite{Seide-14,Wen-17,Bernstein-18,Alistarh-17} is a uniform scalar quantizer, which simply rounds the binary representation of each coordinate to a fixed number of bits, while  \cite{RamezaniKebrya-19} considers a non-uniform scalar quantizer, and  \cite{MayekarTyagi-20,Gandikota-19} construct vector quantizers from the convex hull of specifically structured point sets. 

A different approach to addressing the communication bottleneck in parallel SGD training is to sparsify the gradient vectors \cite{Strom-15,AjiHeafield-17,Stich-18,Wangni-18,Wang-18,Lin-18}. For example, the top-$k$ sparsifier (or \emph{compressor}) preserves the $k$ coordinates of the largest magnitude and sends them with full precision \cite{Strom-15,Dryden-16,AjiHeafield-17,Alistarh-18,Stich-18,yu2018gradiveq}. A user-specified parameter (e.g. $k$ for the top-$k$ compressor) serves as a proxy for the communication rate in this line of work. 

For an empirical risk minimization problem where the global objective function is the average of local objective functions, recent works \cite{AmiriGunduz-19,Zhu-19,Yang-20} perform analog gradient compression and communication by taking the physical superposition nature of the underlying multiple-access channel into the account.

The assumption of unbiased compression error \cite{Nemirovski-09,Bubeck-15,Bottou-16,Mishchenko-19,Horvath-19-stochastic,Horvath-19-natural,philippenko2020bidirectional,gorbunov2020unified,gorbunov2020linearly,khaled2020unified,gorbunov2021marina,islamov2021distributed} is commonly imposed to enable convergence analyses of compressed SGD. Employing biased compressors in compressed SGD can lead to divergence: for example, both the 1-bit SGD without mini-batching \cite{Seide-14,Karimireddy-19} and the top-$1$ compressed SGD \cite{Aleksandr-20} diverge on some problem instances. A set of sufficient conditions on the compression operators to ensure convergence of SGD is put forth in \cite{magnusson2020maintaining}.

The same paper - \cite{Seide-14} - that initiates the study of quantized SGD is also the first to introduce the idea of adding back previous quantization errors   before quantizing the gradient at the next step of iterative optimization, which fixes the divergence issue mentioned above. The idea, referred to as \emph{error compensation}, or \emph{error feedback}, in the federated learning literature, has been long known as $\Sigma$-$\Delta$ modulation \cite{Gray-89} in the information theory literature. Stich et al.~\cite{Stich-18} apply the mechanism of error feedback in \cite{Seide-14} to a more general setting of SGD and show that it converges with the same order as unquantized SGD on strongly convex and smooth functions, providing the first theoretical performance guarantee of that error feedback strategy.  Karimireddy et al. \cite{Karimireddy-19} extend the analysis of \cite{Stich-18} to the non-convex and weakly convex objective functions, while Zheng et al. \cite{Zheng-19} and Gorbunov et al. \cite{gorbunov2020linearly} prove its convergence in the multi-worker setting. Wu et al.~\cite{Wu-18} propose an error feedback mechanism different from \cite{Seide-14} and prove its convergence on quadratic functions using the same quantizer as in \cite{Alistarh-17}. Past quantization errors in the algorithm of \cite{Wu-18} accumulate from one iteration to another and are weighted by time-decaying factors. The momentum correction used in \cite{Lin-18}, as well as the distributed SGD with skipped communication rounds in \cite{sun2020lazily}, are also variants of error compensation.  Qian et al. \cite{qian2021error} propose an error-compensated accelerated SGD, while Richr\'arik et al. \cite{richtarik2021ef21} propose an error-compensated SGD that achieves the same order of convergence as SGD with unbiased compressors \cite{gorbunov2021marina}. The analyses in \cite{Stich-18,Karimireddy-19,gorbunov2020linearly,richtarik2021ef21} assume that the compressor is a contraction operator, while \cite{Wu-18} also assumes its unbiasedness.  Hor\'vath and Richt\'arik \cite{horvath2020better} construct an unbiased compressor from a contractive compressor and employ the resulting unbiased compressor within SGD as an alternative to error feedback to overcome the divergence issue with biased compressors.

Although a number of works provide convergence analyses of their proposed methods, showing that convergence rates of quantized gradient methods depend on the bit rate $R$ \cite{Alistarh-17,RamezaniKebrya-19,Acharya-19,MayekarTyagi-19,MayekarTyagi-20}, there are few existing convergence lower bounds in terms of $R$ that apply to any algorithm within a specified class. For quantized projected SGD, \cite{MayekarTyagi-19,MayekarTyagi-20} give lower bounds to a minimax expected estimation error (i.e. difference between the output function value and the optimal one), which is in the same order of convergence as that of the unquantized SGD over convex functions.
However, the allowable quantizer input in \cite{MayekarTyagi-19,MayekarTyagi-20} is fixed to be the gradient of the current iterate, precluding the use of error compensation. 

The parameter server framework that we consider in this paper is somewhat different from the \emph{distributed estimation} or optimization setting \cite{Boyd-11}, where there has also been great interest in communication-efficient algorithms to account for the distributed nature of these problems. In such applications, all parties in a connected network communicate back and forth in order to estimate the mean of a distribution \cite{Zhang-13-communication} (or a population \cite{Suresh-17}) or to solve a convex optimization cooperatively with quantization effects \cite{Nedic-09,Reisizadeh-19}.
Information-theoretic lower bounds have also been established either in the minimax sense for distributed statistical estimation problem \cite{Zhang-13-information} or in terms of the communication complexity for the distributed convex learning problem \cite{TsitsiklisLuo-87,ArjevaniShamir-15,Suresh-17}.
}

%\todo{
%Naive: signSGD \cite{Bernstein-18}, \cite{Aleksandr-20} (convergence rate is linear only
%in the special case of an over-parameterized regime (i.e., the regime in which the loss functions on
%all nodes share a common minimizer)), \cite{magnusson2020maintaining}
%Aleksandr-20
%
%\cite{FriedlanderSchmidt-12,Alistarh-17,Wen-17,Bernstein-18}
%
%Omitted: \cite{StichKarimireddy-19}
%
%gradient is computed at the unquantized trajectory instead of at the quantized trajectory.  exiThe EF-SGD \cite{Karimireddy-19} differs from {\DQGD} in the gradient-compute point: {\DQGD} computes the gradient at the unquantized trajectory, whereas EF-SGD still computes the gradient at a quantized trajectory. Although \cite{Karimireddy-19} shows that EF-SGD achieves the same convergence rate as SGD in the limit of many iterations, such a shift in the gradient-compute point is detrimental for {\GD} as we discuss in Section~\ref{subsec:disc}.
%
%%Besides, the standard assumption of unbiasedness of the compressor \cite{Nemirovski-09,Bubeck-15,Bottou-16} is crucial for the convergence of SGD.
%
%}

\subsection{Contributions}

\hl{In this paper, we provide a lower bound on quantized non-stochastic gradient descent, and we show a single-worker algorithm that achieves the lower bound with equality, thereby establishing an information-theoretic fundamental limit of quantized gradient descent. In other words, we quantify exactly the minimum amount of information required to achieve a desired convergence speed (within a class of algorithms), and we show an algorithm that achieves it. Because the algorithm achieves the information-theoretic converse with equality, no other algorithm can surpass its performance. It is remarkable that only a finite bit rate is required to achieve the optimal convergence speed achievable with an infinite rate. Our analysis is sharper than existing analyses because we identify constants and not just the order of convergence. We focus on (nonstochastic) GD and not on SGD as most prior work. We do not assume that the quantizer is unbiased or is a contraction operator - our information-theoretic lower bound applies to any quantizer, and any quantizer can be inserted into our algorithm, although our achievability result suggests that picking a (scalar or vector) quantizer with good covering efficiency would perform best.  Our mechanism for error compensation (that we call ``differential quantization'') differs from prior works in the gradient-compute point, which is crucial for achieving our sharp information-theoretic lower bound. Although our information-theoretic lower bound applies to the multi-worker setting as well, our best achievability bound comes short of it at a finite $R$. Thus, it remains an open problem whether the lower bound is achievable in the multi-worker setting.  While our main results are presented in terms of a quantity that is asymptotic in the number of iterations $T$, the analyses that lead to these results are nonasymptotic. Incidentally, we discover two new results on the classical (unquantized) gradient methods: a slightly more general converse for the gradient descent, and a nonasymptotic global convergence bound on Polyak's heavy ball method. 
}

%However, since our scheme refines its estimate of the optimizer successively, it attains a better delay-reliability tradeoff: the server can decide to stop after any i rounds based on its target accuracy while attaining a close-to-optimal tradeoff between communication and convergence over those i rounds. The decision of when to stop is done by the server and not by the worker.
We consider the single-worker scenario of the parameter server framework \cite{Li-14-scaling,Li-14-communication,Alistarh-17,Wen-17,Khirirat-18,Bernstein-18,RamezaniKebrya-19} consisting of a worker that computes the gradients and a server that successively refines the model parameter (i.e. the iterate) and decides when to stop the distributed iterative algorithm based on its target accuracy or delay constraints.
See Fig.~\ref{fig:diagram}.
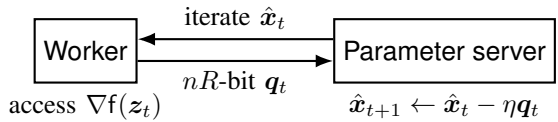
\begin{figure}[h!bt]
\centering
\begin{tikzpicture}[auto,node distance=2.6cm,>=Latex,thick]
    \tikzset{
        block/.style = {
            draw,
            % fill=blue!20,
            rectangle,
            minimum height=2.4em,
            minimum width=2.4em,
            font=\sffamily
        }
    }

    \node[block,label={below:access $\nabla\msf{f}(\bm{z}_t)$}] (worker) {Worker};
    \node[block,right=of worker,label={below:$\hat{\bm{x}}_{t+1}\la\hat{\bm{x}}_t-\eta\bm{q}_t$}] (center) {Parameter server};

    \draw[-Latex,thick] (center.west|-center.6) -- node[above,midway] {iterate $\hat{\bm{x}}_t$} (worker.east|-center.6);
    \draw[-Latex,thick] (worker.-11) -- (center.west|-worker.-11) node[below,midway] {$nR$-bit $\bm{q}_t$};
\end{tikzpicture}
\caption{Quantized gradient descent in a single-worker remote training setting. At each iteration $t$, the server first sends the current iterate $\hat{\bm{x}}_t$ to the worker noiselessly, who computes the gradient at some point $\bm{z}_{t}$ that is a function of (but not necessarily equal to) $\hat{\bm{x}}_t$. Then, the worker forms a descent direction $\bm{q}_{t}$ and pushes it back to the server under the $nR$ bits per iteration constraint.}
\label{fig:diagram}
\end{figure}

We study the fundamental tradeoff between the convergence rate and the communication rate of quantized gradient descent. 
We focus on the class $\set{F}_n$ of smooth and strongly convex objective functions $\fcn{f}\colon \mathbb R^n \mapsto  \mathbb R$ whose minimizers are bounded in the Euclidean norm.
%We do not restrict ourselves to the commonly used \emph{gain-shape} quantization that represents the norm of the gradient with machine-dependent (e.g. 32-bit or 64-bit) floating point numbers and quantizes the normalized gradient only \cite{Alistarh-17,Wen-17,RamezaniKebrya-19}. Instead, we consider the general class of fixed-rate quantizers with a cardinality constraint on the total number of quantizer outputs.
For a quantized iterative algorithm $\textrm A$, its worst-case linear \emph{contraction factor} over $\mathcal F_n$ at rate $R$ bits per problem dimension is defined as
\begin{align}
\sigma_{\textrm A}(n, R) \triangleq \inf_{R^\prime \leq R} \sup_{\msf f \in \mathcal F_n} \limsup_{T \to \infty} {\norm{\hat{\vec{x}}_T(R^\prime)-\vec{x}_{\fcn{f}}^*}}^{\frac 1 T}
\label{eq:lcr}
\end{align}
where $\vec{x}_{\fcn{f}}^*$ is the optimizer, and $\hat{\vec{x}}_0(R^\prime),\hat{\vec{x}}_1(R^\prime), \hat{\vec{x}}_2(R^\prime),\ldots$ is the sequence of iterates generated by $\textrm A$ in response to $\msf f \in \mathcal F_n$ when it operates at $R^\prime$ bits per problem dimension.

We consider three popular algorithms that converge linearly:\footnote{The term ``linear convergence'' is used in the literature as a synonym for convergence with the rate of  geometric progression. Note that SGD converges only sub-linearly over smooth and strongly convex functions \cite{Nemirovski-09,Bubeck-15,Bottou-16}.} the classical gradient descent (GD) with fixed step size, the accelerated gradient descent (AGD) \cite{Nesterov-14}, and the heavy ball method (HB) \cite{Polyak-87}. 
 We propose a principle for error feedback we call \emph{differential quantization} (DQ) that says that the quantizer input should be formed in such a way as to guide the descent trajectory of the quantized algorithm towards the descent trajectory of its unquantized counterpart. By applying the DQ principle to the GD, AGD, and HB algorithms, we construct three new quantized iterative optimization algorithms: DQ-GD, DQ-AGD, and DQ-HB. By analyzing them, we show achievability bounds of the form\footnote{The convergence result on DQ-HB in \eqref{eqn:rate-fcn-upper} requires that the function $\fcn f \in \set F_n$ is twice continuously differentiable.} 
 \begin{equation}
\label{eqn:rate-fcn-upper}
\sigma_{\textrm A}(n, R) \leq \max\bp{\sigma_{\textrm A},\rho_n 2^{-R}  \phi_{\textrm A}(n, R)  },
\end{equation}
where $\textrm A \in \left\{ \textrm{DQ-GD, DQ-AGD, DQ-HB} \right\}$, $\sigma_{\textrm A}$ is the contraction factor of the unquantized counterpart of $\textrm A$, $\rho_n \geq 1$ is the covering efficiency of the quantizer, and $\phi_{\textrm A}(n, R) \geq 1 $ is function that we specify; for example, 
 \begin{equation}
\sigma_{\textrm{DQ-GD}}(n, R) \leq \max\bp{\sigma_{\mathrm{GD}}, \rho_n 2^{-R}}.
\label{eq:DQ-GDintro}
\end{equation}
As \eqref{eqn:rate-fcn-upper} indicates, each of the novel DQ algorithms achieves the corresponding $\sigma_{\textrm A}$ once the rate passes a hard threshold. In other words, there is no loss at all due to quantization once the rate is high enough. 

We show an information-theoretic converse of the form
\begin{equation}
\label{eqn:rate-fcn-lower}
\sigma_{\textrm{A}}(n, R) \geq \max\bp{\sigma_{\mathrm{GD}}, 2^{-R}},
\end{equation}
which applies to any ``quantized gradient descent'' algorithm $ \textrm A$ (in the class of ``quantized gradient descent'' algorithms, summarized in \figref{fig:diagram}, the server can utilize only the last quantized input to form the next iterate). Recalling the classical result of Rogers \cite[Th. 3]{Rogers-63} that shows the existence of quantizers with covering efficiency $\rho_n \to 1$ as $n \to \infty$ and comparing \eqref{eqn:rate-fcn-upper} and   \eqref{eqn:rate-fcn-lower}, one can deduce the asymptotic optimality of DQ-GD within the class of ``GD-like'' algorithms. In contrast, the natural method that quantizes the gradient of its current iterate directly \cite{FriedlanderSchmidt-12,Alistarh-17,Wen-17,Bernstein-18} referred to as naively quantized (NQ) GD in this paper, has contraction factor \hl{(in the single-worker scenario; see Section~\ref{sec:multiworker} for the multi-worker result)}
\begin{equation}
\sigma_{\textrm {NQ-GD}}(n, R) \leq \sigma_{\mathrm{GD}} +  \frac {2\kappa}{\kappa + 1} \rho_n 2^{-R}
\label{eq:DQ-NGintro}
\end{equation}
where $\kappa \geq 1$ is the condition number of $\fcn f$. The guarantee \eqref{eq:DQ-NGintro} is significantly worse than \eqref{eq:DQ-GDintro}. 

Our numerical results indicate that the upper bounds \eqref{eqn:rate-fcn-upper} and  \eqref{eq:DQ-NGintro}  accurately represent the actual achieved contraction factors.

Within a wider class of quantized gradient methods (the server can utilize full memory of the past), the converse \eqref{eqn:rate-fcn-lower} can be surpassed. Once the rate passes the threshold mentioned earlier, DQ-HB attains the minimum possible contraction factor among all algorithms in that wider class, even unquantized ones. 
 
The rest of the paper is organized as follows. Differentially quantized algorithms are presented in Section~\ref{sec:algo}. Their convergence analyses and an experimental validation on least-squares problems are shown in Section~\ref{sec:convergence}. 
The converses are presented in Section~\ref{sec:converse}. The multi-worker setting is discussed in Section~\ref{sec:multiworker}.

\section{Differentially Quantized Algorithms}
\label{sec:algo}

\subsection{Quantizers employed in DQ algorithms}
A \emph{quantizer of dimension $n$ and rate $R$} is a function $\fcn{q}\colon\set{D}\to\R^n$, where $\set{D}\subseteq \R^n$ is the domain, such that the image of $\fcn{q}$ satisfies
\begin{equation}
\label{eqn:def-qnt-rate}
\abs{\mrm{Im}(\fcn{q})}= 2^{nR}.
\end{equation}
This is the classical general fixed-rate quantizer in the information theory literature. 
We fix a dimension-$n$, rate-$R$ quantizer $\fcn{q}$, and we set up quantizer $\fcn{q}_t$ to be used at iteration $t$ as
\begin{equation}
\label{eqn:quant-scale}
\fcn{q}_t(\cdot) = r_t\fcn{q}(\cdot/r_t)
\end{equation}
for a properly chosen sequence of shrinkage factors $\{r_t\}$ (see \eqref{eq:DR_memory}, \eqref{eqn:range-AGD}, and \eqref{eqn:range-HB}, below). Therefore, each quantizer $\fcn{q}_t$ has the same geometric structure but different resolution.

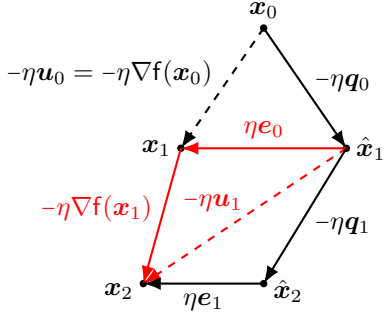
\begin{figure}[t]
\centering
\begin{tikzpicture}[auto,>=Latex,thick]
\filldraw[black] (0,0) circle (1pt) node[above] {$\bm{x}_0$};
\filldraw[black] (-1.1,-1.6) circle (1pt) node[left] {$\bm{x}_1$};
\filldraw[black] (1.1,-1.6) circle (1pt) node[right] {$\hat{\bm{x}}_1$};

\draw[->,dashed] (0,0) -- node[midway,left,yshift=2mm] {$\minus\eta\bm{u}_0=\minus\eta\nabla\msf{f}(\bm{x}_0)$} (-1.1,-1.6);
\draw[->] (0,0) -- node[midway,right,yshift=1mm] {$\minus \eta \bm{q}_0$} (1.1,-1.6);
\draw[->,red] (1.1,-1.6) -- node[midway,above] {$\eta\bm{e}_0$} (-1.1,-1.6);

\filldraw[black] (-1.6,-3.4) circle (1pt) node[left,yshift=-1mm] {$\bm{x}_2$};
\filldraw[black] (0,-3.4) circle (1pt) node[right,yshift=-1mm] {$\hat{\bm{x}}_2$};

\draw[->,red] (-1.1,-1.6) -- node[midway,left,yshift=1mm] {$\minus\eta\nabla\msf{f}(\bm{x}_1)$} (-1.6,-3.4);
\draw[->,red,dashed] (1.1,-1.6) -- node[midway,left,xshift=1mm,yshift=2mm] {$\minus\eta\bm{u}_1$} (-1.6,-3.4);
\draw[->] (1.1,-1.6) -- node[midway,right,yshift=-1mm] {$\minus \eta \bm{q}_1$} (0,-3.4);
\draw[->] (0,-3.4) -- node[midway,below] {$\eta\bm{e}_1$} (-1.6,-3.4);
\end{tikzpicture}
\caption{Illustration of the DQ-GD algorithm (Algorithm~\ref{alg:DQGD_short}).}
\label{fig:DQGD}
\end{figure}

\subsection{Differentially Quantized Gradient Descent}
The (unquantized) gradient descent algorithm searches along the direction of the negative gradient toward which the function value decreases:
\begin{equation}
\label{eq:GD}
\bm{x}_{t+1} = \bm{x}_t - \eta\nabla\msf{f}(\bm{x}_t),
\end{equation}
where $\eta>0$ is the constant stepsize chosen to minimize the function value along the search direction. 

In Fig.~\ref{fig:DQGD}, we illustrate an application of differential quantization (DQ) to GD \eqref{eq:GD}, which yields the DQ-GD algorithm (Algorithm~\ref{alg:DQGD_short}). At each iteration $t$, DQ-GD first determines the iterate $\bm{x}_t$ associated with the corresponding unquantized algorithm, i.e., GD, by compensating previous scaled quantization error $\eta\bm{e}_{t-1}$ (Line~4). It then computes the gradient at $\bm{z}_t=\bm{x}_t$ (\lemref{lma:recur_memory_full}) and sets the quantizer input as (Line~5)
\begin{equation}
\label{eq:input_memory_short}
\bm{u}_t = \nabla\msf{f}\lp \hat{\bm{x}}_t+\eta\bm{e}_{t-1}\rp - \bm{e}_{t-1},
\end{equation}
which in the absence of quantization error $\bm{e}_{t}$ would guide the iterate $\hat{\bm{x}}_t$ back to ${\bm{x}}_{t+1}$ (see Fig.~\ref{fig:DQGD}). The recorded scaled quantization error $\eta \bm{e}_t$ captures exactly the difference between $\hat{\bm{x}}_{t+1}$ and $\bm{x}_{t+1}$ for the next iteration.

See Appendix~\ref{sec:DQGDvary} for the DQ-GD algorithm with varying stepsize $\eta_t$. 

\begin{algorithm2e}[th]
\caption{DQ-GD}
\label{alg:DQGD_short}
   \DontPrintSemicolon
   \SetKw{KwW}{{\small\textsf{Worker}}:}
   \SetKw{KwS}{{\small\textsf{Server}}:}
   \SetKwIF{IfW}{ElseIfW}{ElseW}{if}{then}{else if}{\KwW}{endif}
   Initialize $\bm{e}_{-1} = \bm{0}$\;
   \For{$t = 0, 1, 2, \ldots$}{
      \SetAlgoVlined
      \ElseW{$\bm{z}_t=\hat{\bm{x}}_t+\eta\bm{e}_{t-1}$\;
         $\bm{u}_t=\nabla\msf{f}(\bm{z}_t)-\bm{e}_{t-1}$\;
         $\bm{q}_t=\msf{q}_t(\bm{u}_t)$\;
         $\bm{e}_t=\bm{q}_t-\bm{u}_t$\;
      }
      \KwS $\hat{\bm{x}}_{t+1}=\hat{\bm{x}}_t-\eta\bm{q}_t$\;
   }
\end{algorithm2e}

\subsection{Differentially Quantized Accelerated Gradient Descent}

Nesterov's Accelerated Gradient Descent (AGD) \cite{nesterov82method} keeps track of two iterate sequences
\begin{align}
\vec{y}_{t+1} &= \vec{x}_t - \eta\grad\fcn{f}(\vec{x}_t) \label{eqn:descent-AGD-y} \\
\vec{x}_{t+1} &= \vec{y}_{t+1} + \gamma \left( \vec{y}_{t + 1} - \vec{y}_t\right) \label{eqn:descent-AGD-x}.
\end{align}
It first performs the gradient descent step \eqref{eqn:descent-AGD-y}, and then adds the momentum term  $\gamma \left(\vec{y}_{t+1} - \vec{y}_t\right)$ \eqref{eqn:descent-AGD-x} to form  a projection $\vec{x}_{t+1}$ of the GD iterate $\vec{y}_{t+1}$ to its near future. The momentum term incorporates second-order effects by leveraging the past $\vec{y}_t$. The AGD is the first algorithm that achieved the contraction factor that is order-wise optimal (in terms of the condition number of $\mathsf f$) among all first-order (gradient) optimization methods \cite{Nesterov-14} (\lemref{lma:conv-AGD}). There are various interpretations of Nesterov's acceleration phenomenon. We refer the reader to \cite{AllenZhuOrecchia-14} for a connection to the mirror descent algorithm and to \cite{Su-16} for an interpretation in terms of differential equation. %Although $\sigma_a$ \eqref{eqn:conv-lin-AGD} is in the same order $O(\exp(-1/\sqrt{\kappa}))$ as the converse $\sigma_b$ \eqref{eqn:conv-lin-HB}, the gap $\sigma_a-\sigma_b$ is strictly positive unless $\kappa=1$ trivially.

%of order $O(\exp(-1/\sqrt{\kappa}))$ (GD achieves only $O(\exp(-1/\kappa))$), where $\kappa$ is the condition number (see \eqref{eq:condition}, below) $\sigma_b$ \eqref{eqn:conv-lin-HB} in Lemma~\ref{lma:lower-Nesterov} is possible \cite{Nesterov-14}. 

Differentially Quantized AGD algorithm is presented as Algorithm~\ref{alg:DQAGD}. At each iteration $t$, DQ-AGD uses the past two quantization errors $\vec{e}_{t-1},\vec{e}_{t-2}$ to determine the gradient-compute point $\vec{z}_t$ (Line~4) and the quantizer input $\vec{u}_t$ (Line~5). As dictated by the principle of differential quantization, DQ-AGD computes the gradient at the same point as unquantized AGD, i.e., $\vec{z}_t = \vec x_t$ (\lemref{lma:track-AGD}).

\begin{algorithm2e}[ht]
\caption{DQ-AGD}
\label{alg:DQAGD}
 
\DontPrintSemicolon
\SetKw{KwServer}{Server}
\SetKw{KwWorker}{Worker}
   Initialize $\vec{e}_{-2} = \vec{e}_{-1} = \vec{0}$, $\hat{\vec{y}}_0 = \hat{\vec{x}}_0$\;
\For{$t = 0, 1, 2, \ldots$}{
    \KwWorker:\;
    \Indp
        $\vec{z}_t=\hat{\vec{x}}_t+\eta \left[ \vec{e}_{t-1} + \gamma \left( \vec{e}_{t-1} - \vec{e}_{t-2} \right) \right]$\;
        $\vec{u}_t=\grad\fcn{f}(\vec{z}_t)-  \left[\vec{e}_{t-1} + \gamma \left( \vec{e}_{t-1} - \vec{e}_{t-2} \right) \right]$\;
        $\vec{q}_t=\fcn{q}_{t}(\vec{u}_t)$\;
        $\vec{e}_t=\vec{q}_t-\vec{u}_t$\;
    \Indm
    \KwServer:\;
    \Indp
        $\hat{\vec{y}}_{t+1}=\hat{\vec{x}}_t-\eta\vec{q}_t$\;
        $\hat{\vec{x}}_{t+1}= \hat{\vec{y}}_{t+1} + \gamma \left( \hat{\vec{y}}_{t+1} - \hat{\vec{y}}_t \right)$\;
    \Indm
}
\end{algorithm2e}

\subsection{Differentially Quantized Heavy Ball Method}

Polyak's Heavy Ball (HB) algorithm \cite{Polyak-87} iterates
\begin{equation}
\label{eqn:descent-HB}
\vec{x}_{t+1} = \vec{x}_t - \eta\grad\fcn{f}(\vec{x}_t) + \gamma \left( \vec{x}_{t} - \vec{x}_{t-1} \right),
\end{equation}
where $\gamma\left( \vec{x}_t - \vec{x}_{t-1}\right) $ is the momentum term that nudges $\vec{x}_{t+1}$ in the direction of the previous step,  and accelerates convergence to the optimizer.   In contrast to AGD, the HB method only uses the gradient at the current iterate. The HB method derives from the analogy with physics, since the continuous-time counterpart of \eqref{eqn:descent-HB} is a second-order ODE that models the motion of a body (``the heavy ball'') in a field with potential $\fcn{f}$ under the force of friction.  At the expense of requiring function $\mathsf f$ in $\mathcal F_n$ to be further twice continuously differentiable, the HB algorithm can be shown to converge with the optimal contraction factor achievable among all first-order optimization methods \cite[Th. 3.1]{Polyak-87} (\lemref{lma:HB}), \cite[Th.~2.1.13]{Nesterov-14} (\lemref{lma:lower-Nesterov}). In comparison, the AGD approaches it only order-wise, but it does not require the second derivative of $\mathsf f$ to exist, a significant restriction in practical applications.

Differentially Quantized HB algorithm is presented as Algorithm~\ref{alg:DQHB}. 
In accordance with the principle of differential quantization, the worker computes the gradient at $\bm{z}_t=\bm{x}_t$ (\lemref{lma:track-HB}).
Note that DQ-HB has the same expression for its quantizer input $\vec{u}_t$ (Line~4) as DQ-AGD (Line~5).

\begin{algorithm2e}[ht]
\caption{Differentially Quantized Heavy Ball Method (DQ-HB)}
\label{alg:DQHB}
\DontPrintSemicolon
\SetKw{KwServer}{Server}
\SetKw{KwWorker}{Worker}
   Initialize $\vec{e}_{-2} = \vec{e}_{-1} = \vec{0}, ~ \hat{\vec{x}}_{-1} = \hat{\vec{x}}_{0}$\;
\For{$t = 0, 1, 2, \ldots$}{
    \KwWorker:\;
    \Indp
        $\vec{z}_t=\hat{\vec{x}}_t+\eta\vec{e}_{t-1}$\;
        $\vec{u}_t=\grad\fcn{f}(\vec{z}_t)-\bk{ \vec{e}_{t-1} + \gamma \left( \vec{e}_{t-1}- \vec{e}_{t-2} \right) }$\;
        $\vec{q}_t=\fcn{q}_t(\vec{u}_t)$\;
        $\vec{e}_t=\vec{q}_t-\vec{u}_t$\;
    \Indm
    \KwServer: $\hat{\vec{x}}_{t+1}= \hat{\vec{x}}_t-\eta\vec{q}_t+ \gamma\left( \hat{\vec{x}}_{t} - \hat{\vec{x}}_{t-1}\right)$\;
}
\end{algorithm2e}

\section{Convergence rates of DQ algorithms}
\label{sec:convergence}

\subsection{Definitions}
\label{sec:functiondef}
We denote by $\lnorm\cdot\rnorm$ the Euclidean norm, and by $\set{B}(r)\triDef\bp{\vec{u}\in\R^n \colon \norm{\vec{u}}\leq r}$ the Euclidean ball in $\R^n$ with radius $r$ and center at $\vec 0$. 

We fix positive scalars $L$, and $\mu$, and  $D$, and we say that a continuously differentiable function $\mathsf f \colon \mathbb R^n \mapsto \mathbb R$ is in class $\mathcal F_n$ if
\begin{enumerate}[i)]
\item $\mathsf f$ is $L$-smooth, i.e., 
\begin{equation}
\label{eq:smooth}
\lnorm \nabla\msf{f}(\bm{v})-\nabla\msf{f}(\bm{w})\rnorm \leq L\lnorm \bm{v}-\bm{w}\rnorm;
\end{equation}
\item $\mathsf f$ is $\mu$-strongly convex, i.e., 
\begin{equation}
\label{eqn:str-cvx}
\text{function } \vec{v} \mapsto \fcn{f}(\vec{v})-\frac{\mu}{2}\norm{\vec{v}}^2 \text{ is convex};
\end{equation}
\item  the minimizer $\vec{x}_{\fcn{f}}^* \triDef \argmin_{\vec{x}\in \mathbb R^n} \fcn{f}(\vec{x})$ satisfies
\begin{equation}
\label{eq:iter_opt_range}
\lnorm \bm{x}_{\msf{f}}^*  - \hat{\vec{x}}_0 \rnorm\leq D,
\end{equation}
where $\hat{\vec{x}}_0$ is the starting location of iterative algorithms. 
\end{enumerate}

We say that $\mathsf f \colon \mathbb R^n \mapsto \mathbb R$ is in class $\mathcal F_n^2$ if it is in $\mathcal F_n$ and is in addition twice continuously differentiable. 

We denote the \emph{condition number} of an $\fcn{f} \in \mathcal F_n$ by
\begin{equation}
\kappa \triDef \frac{L}{\mu}.
\label{eq:condition}
\end{equation}
Note that $\kappa \geq 1$ due to \eqref{eq:smooth} and \eqref{eqn:str-cvx}.

For a bounded-domain quantizer $\fcn{q}\colon\set{D}\to\R^n$, we refer to
\begin{equation}
\label{eq:DR}
r(\fcn{q}) \eqDef \max\lbp \delta\colon \set{B}(\delta) \subseteq\mcal{D}\rbp
\end{equation}
as the \emph{dynamic range} of $\msf{q}$, 
to 
\begin{equation}
\mathsf d(\fcn{q}) \triangleq \min\lbp d \colon \forall\bm{x}\in\mcal{D},\;\lnorm \bm{x}-\msf{q}(\bm{x})\rnorm \leq d\rbp  
\label{eq:covradius}
\end{equation}
as its \emph{covering radius}, and to 
\begin{equation}
\label{eq:def_d}
\rho \lp \msf{q}\rp \eqDef \abs{\mrm{Im}(\fcn{q})}^{1/n}\frac{\msf{d}(\msf{q})}{r(\msf{q})}
\end{equation}
as its \emph{covering efficiency}.\footnote{Covering efficiency introduced in \eqref{eq:def_d}  extends the notion of covering efficiency of an infinite lattice \cite{Zamir-14}, which measures how well that lattice covers the whole space, to  bounded-domain quantizers.
} %Covering efficiency of a rate-$R$ quantizer $\fcn{q}$ measures the ratio of its covering radius to its prescribed dynamic range, normalized by its per-dimension cardinality.
A scalar uniform quantizer $\fcn{q}_{\mathsf u}$ has domain $[-r(\fcn{q}_{\mathsf u}),r(\fcn{q}_{\mathsf u})]^n$ and covering efficiency $\sqrt{n}$. This is wasteful: the classical result of Rogers \cite[Th. 3]{Rogers-63} implies that there exists a sequence of $n$-dimensional quantizers $\fcn{q}_{n}$ with $\rho \lp \msf{q}_n\rp \to 1$ as $n \to \infty$, while definition \eqref{eq:def_d} implies that $\rho \lp \msf{q}\rp \geq 1$ for any quantizer $\msf{q}$.

\subsection{DQ-GD: convergence analysis and simulation results}
\label{sec:DQ-GDa}
Unquantized gradient descent with the optimal stepsize given by
\begin{equation}
\label{eq:stepsize_gd}
\eta = \eta_{\mathrm{GD}} \triangleq \frac{2}{L+\mu}
\end{equation}
achieves contraction factor
\begin{equation}
\sigma_{\mathrm{GD}} \triangleq \frac{\kappa-1}{\kappa+1}
\label{eq:sigmaDQ}
\end{equation}
over $\mathcal F_n$ \cite[Th.~1.4]{Polyak-87}, \cite[Th.~2.1.15]{Nesterov-14}  (Lemma~\ref{lma:conv-GD}).
The following result provides a convergence guarantee for DQ-GD.
\begin{theorem}[Convergence of DQ-GD]
\label{thm:DQ-GD}
Fix a dimension-$n$, rate-$R$ quantizer $\fcn{q}$ with dynamic range $1$ and covering efficiency $\rho_n$. Then, Algorithm~\ref{alg:DQGD_short} with stepsize \eqref{eq:stepsize_gd}
and dynamic ranges $r_0 = LD$, 
\begin{equation}
\label{eq:DR_memory}
r_{t + 1} = \sigma_{\mathrm{GD}}^{t+1}\, LD + r_t\, \rho_n 2^{-R}, \quad t = 1, 2, \ldots
\end{equation}
in the definition of $\msf{q}_t$ \eqref{eqn:quant-scale} 
achieves the following contraction factor over $\mathcal F_n$ \eqref{eq:lcr}:
\begin{equation}
\sigma_{\mathrm{DQ-GD}}(n, R) \leq \max\bp{\sigma_{\mathrm{GD}}, \rho_n 2^{-R}}.
\label{eq:sigmaDQ-GD}
\end{equation}
\end{theorem}
\begin{proof}[Proof sketch]
 The path of DQ-GD and that of GD are related as (see Fig.~\ref{fig:DQGD}, \lemref{lma:recur_memory_full})
\begin{equation}
\hat{\bm{x}}_t = \bm{x}_t - \eta\bm{e}_{t-1} 
\label{eq:recur_memory_short}
\end{equation}
Comparing \eqref{eq:recur_memory_short} and Line~4 in Algorithm~\ref{alg:DQGD_short}, we see that $\bm{z}_t = \bm{x}_t$, i.e., {\DQGD} computes the gradient at the unquantized trajectory $\{\bm{x}_t\}$. The convergence guarantee of {\GD}  \cite[Th.~1.4]{Polyak-87}, \cite[Theorem~2.1.15]{Nesterov-14} (Lemma~\ref{lma:conv-GD}) controls the difference between the first term in the recursion \eqref{eq:recur_memory_short} and the optimizer $\bm x_{\mathsf f}^*$. To bound the second term in \eqref{eq:recur_memory_short}, we observe using \eqref{eq:def_d} that for any $r_t > 0$ in \eqref{eqn:quant-scale}, 
\begin{align}
\max_{\bm{u}\in\mcal{B}(r_t)}\lnorm \msf{q}_t(\bm{u})-\bm{u}\rnorm &= r_t\max_{\bm{u}\in\mcal{B}(1)}\lnorm \msf{q}(\bm{u})-\bm{u}\rnorm \\
 &= r_t \, \rho_n 2^{-R} \label{eq:e_d},
\end{align}
i.e. quantizer $\msf{q}_t$ used at iteration $t$ has dynamic range $r_t$ and covering radius \eqref{eq:e_d}. To complete the proof, we show by induction that with $r_t$ in \eqref{eq:DR_memory}, the input $\bm{u}_t$ to the quantizer $\msf{q}_t$ generated by Algorithm~\ref{alg:DQGD_short} always lies within $\set{B}(r_t)$. Since recurrence relation \eqref{eq:DR_memory} represents a geometric sequence, \eqref{eq:e_d} implies that the quantization error decays exponentially fast. The stepsize \eqref{eq:stepsize_gd} is optimal both for GD \cite[Theorem~2.1.15]{Nesterov-14} and for DQ-GD. See Appendix~\ref{sec:DQ-DGa} for details. 
\end{proof}

The bound in \eqref{eq:sigmaDQ-GD} exhibits a phase-transition behavior: at any $R \geq \log_2 \frac{\rho_n} {\sigma_{\mathrm{GD}}}$, achieving the contraction factor of unquantized {\GD} is possible, while at any $R< \log_2 \frac{\rho_n} {\sigma_{\mathrm{GD}}}$, the achievable contraction factor is only $\rho_n 2^{-R} = \frac{\msf{d}(\msf{q})}{r(\msf{q})}$. The algorithm converges linearly as long as $\rho_n 2^{-R} <1$.

A common approach to quantizing descent algorithms \cite{FriedlanderSchmidt-12,Alistarh-17,Wen-17,Bernstein-18,MayekarTyagi-19,RamezaniKebrya-19} we refer to as \emph{naive quantization} has the worker directly quantize the gradient of its current iterate. Applied to GD, it leads to the Naively Quantized Gradient Descent (NQ-GD) with the quantizer input (cf. \eqref{eq:input_memory_short})
\begin{equation}
\label{eq:input_naive}
\bm{u}_{t} = \nabla\msf{f}(\hat{\bm{x}}_t).
\end{equation}
In \thmref{thm:NQ-GD} in \secref{sec:multiworker} below, we show that 
\begin{equation}
\sigma_{\textrm {NQ-GD}}(n, R) \leq \sigma_{\mathrm{GD}} + \frac{2\kappa}{\kappa + 1}\,\rho_n2^{-R}.
\label{eq:sigmaNQ-GD}
\end{equation}
which is strictly greater than \eqref{eq:sigmaDQ-GD}.

In Fig.~\ref{fig:exp}, we numerically compare the contraction factor of {\DQGD} (Algorithm~\ref{alg:DQGD_short}), the {\NQGD}, and the unquantized {\GD} \eqref{eq:GD} on least-squares problems
\begin{equation}
\label{eq:obj_LS}
\msf{f}(\bm{x}) = \frac{1}{2}\lnorm \bm{y}-\mbf{A}\bm{x}\rnorm^2
\end{equation}
where $\bm{y}\in\mbb{R}^m,\mbf{A}\in\mbb{R}^{m\by n}$, with $m\geq n$.
% \topic{implemented quantizer}
\hl{We generate $500$ matrices $\mbf{A}$'s with i.i.d. standard normal entries, one for each $\bm{y}$, and rescale the spectrum of $\mbf{A}$ so that it has a prescribed condition number $\kappa$. We also run the algorithm on the real-world least-squares matrix \texttt{ash331} extracted from the online repository \emph{SuiteSpare} \cite{Kolodziej-19}. For each per-dimension quantization rate $R\geq1$, we generate $500$ instances of the vector $\bm{y}$ and $\hat{\bm{x}}_0$ with i.i.d. standard normal entries.  We run the iterative algorithms for as many iterations $T$ as possible until reaching the machine's floating point precision, and report the average contraction factor.
We use the uniform scalar quantizer for the ease of implementation and take as a consequence a space-filling loss of $\sqrt{n}$. For smaller values of the data rate $R$, quantized GD may not even converge as $\sqrt{n}2^{-R}\geq1$. In that case, we clip off the contraction factor at $1$ in the plots. We set the stepsize and the quantizer's dynamic range in the DQ-GD algorithm as prescribed by \thmref{thm:DQ-GD}, and in the NQ-GD algorithm as prescribed by \thmref{thm:NQ-GD} in \secref{sec:multiworker} below.}

% \topic{observation}
% \noindent
We observe that {\DQGD} has a significantly faster contraction factor than {\NQGD}, and that the empirical results closely track our analytical convergence bounds \eqref{eq:sigmaDQ-GD} and \eqref{eq:sigmaNQ-GD}.  The contraction factor of unquantized {\GD} serves as a lower bound to both quantized algorithms.

\hl{%An analysis of gradient descent with additive noise is provided in \cite{FriedlanderSchmidt-12}. Particularizing the noise of \cite{FriedlanderSchmidt-12} to quantization noise, one can apply the techniques of \cite{FriedlanderSchmidt-12} to study the contraction factor of NQ-GD (Appendix~\ref{subsec:FS12}). 

Applying the error feedback mechanism of \cite{Seide-14,Stich-18,Karimireddy-19}, developed for SGD, to GD results in an algorithm that forms the quantizer input as
\begin{equation}
\label{eq:input_EFGD}
\bm{u}_t = \nabla\msf{f}(\hat{\bm{x}}_t) - \bm{e}_{t-1}.
\end{equation}
Unlike DQ-GD \eqref{eq:input_memory_short}, error feedback in \eqref{eq:input_EFGD} results in computing the gradient along the quantized trajectory $\{\hat{\bm{x}}_t\}$, and it is unclear whether it can even improve upon NQ-GD \eqref{eq:sigmaNQ-GD} in the setting of our paper - nonstochastic GD with a worst-case performance criterion and without further assumptions on the quantizer (Appendix~\ref{subsec:EFGD}).}

\begin{figure}
     \centering
     \begin{subfigure}[b]{0.49\textwidth}
         \centering
\includegraphics[width=\textwidth]{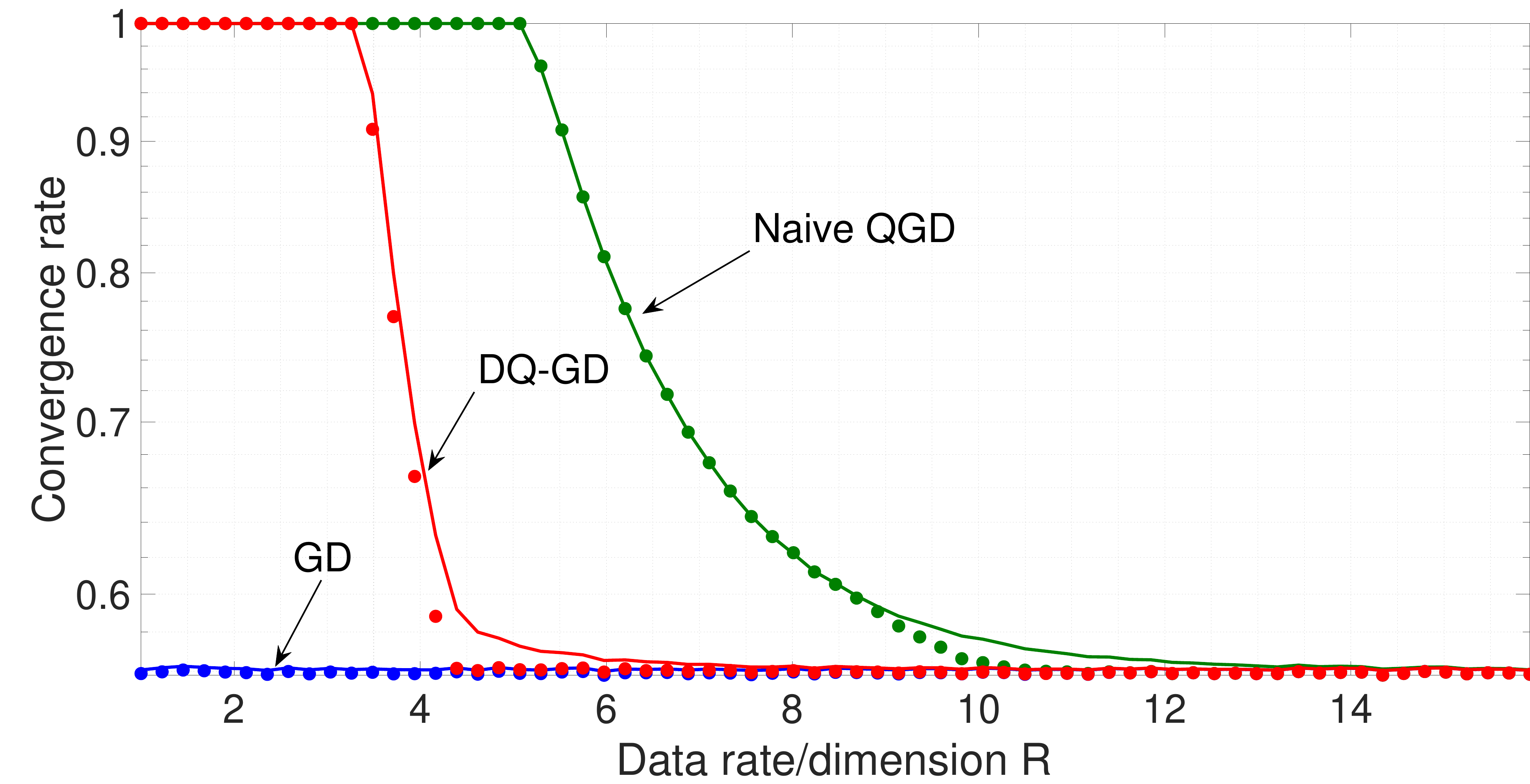}
     \end{subfigure}
     \hfill
     \begin{subfigure}[b]{0.49\textwidth}
         \centering
\includegraphics[width=\textwidth]{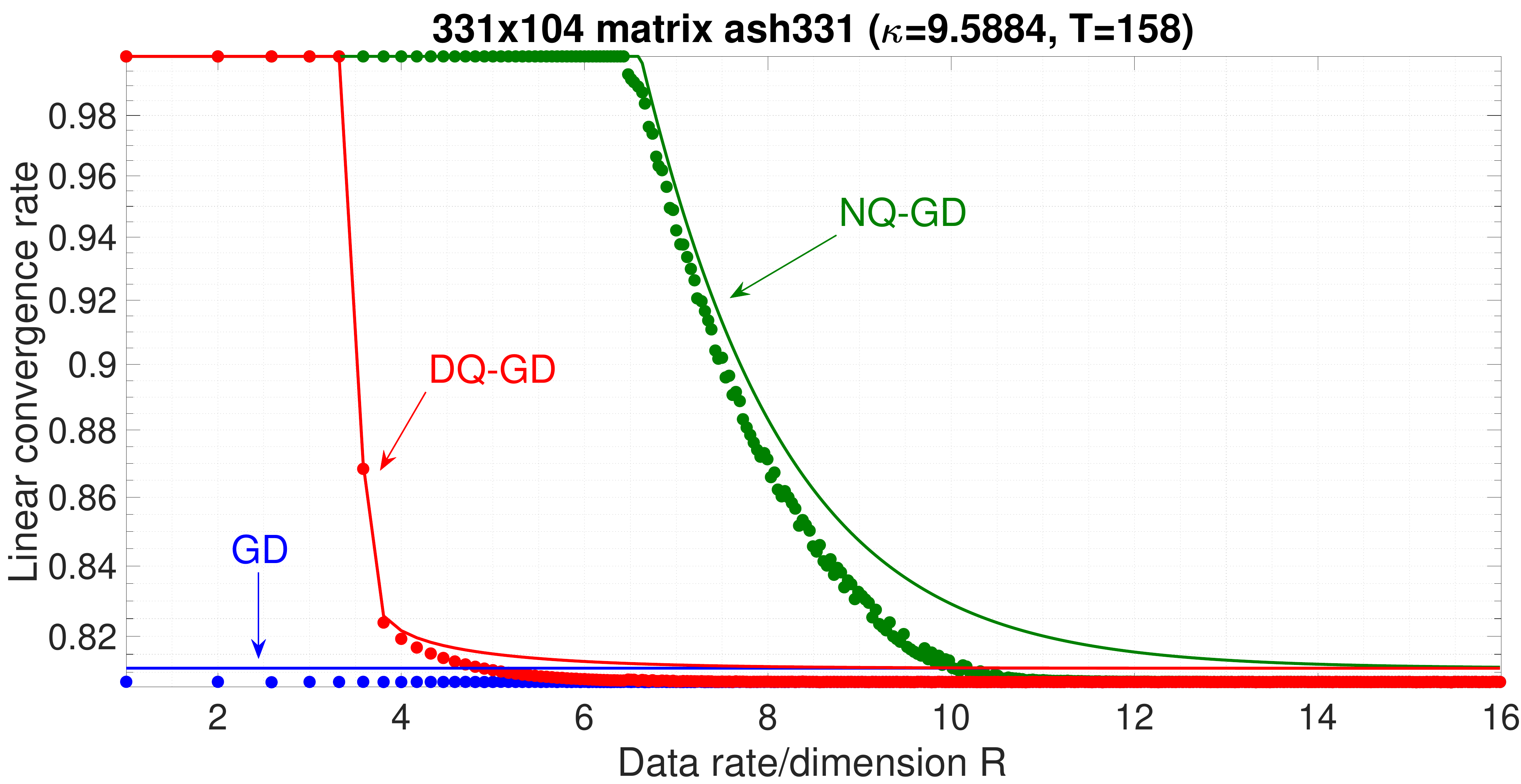}
     \end{subfigure}
\caption{Empirical contraction factors (as circles) and corresponding upper bounds \eqref{eq:sigmaDQ}, \eqref{eq:sigmaDQ-GD}, and \eqref{eq:sigmaNQ-GD} (as lines).
}
\label{fig:exp}
\end{figure}

\subsection{DQ-AGD: convergence analysis}

Unquantized accelerated gradient descent with stepsize
\begin{equation}
\label{eq:stepsize_agd}
\eta = \eta_{\mathrm{AGD}} \triangleq \frac{1}{L}
\end{equation}
and momentum coefficient
\begin{equation}
\gamma = \gamma_{\mathrm{AGD}} \triangleq \frac{\sqrt{\kappa}-1}{\sqrt{\kappa}+1}
\label{eq:gamma_agd}  
\end{equation}
achieves contraction factor
\begin{align}
\sigma_{\mathrm{AGD}} \triangleq &~ \sqrt{1-\frac{1}{\sqrt{\kappa}}}
\label{eq:sigmaADQ}
\end{align}
over $\mathcal F_n$  \eqref{eq:lcr} \cite[Th.~3.18]{Bubeck-15} (\lemref{lma:conv-AGD}), which improves the contraction factor of gradient descent $\sigma_{\mathrm{GD}} = 1 - \frac 1 {\kappa} + O \left(\frac 1 {\kappa^2} \right)$ \eqref{eq:sigmaDQ} to $\sigma_{\mathrm{AGD}} = 1 - \frac 1 {2 \sqrt{\kappa}} + O \left(\frac 1 {\kappa} \right)$, a significant improvement if $\kappa$ is large and optimal order-wise (the converse to the optimal contraction factor expands as $1 -\frac 4 {\sqrt{\kappa}} + O \left(\frac 1 {\kappa} \right)$ \cite{Nesterov-14} (\lemref{lma:lower-Nesterov}) and is attained in $\mathcal F_n^2$ by the heavy ball method \cite{Polyak-87} (\lemref{lma:HB}). 

Denote for brevity the constant
\begin{align}
\lambda &\triDef \left( 1+ \gamma_{\mathrm{AGD}}+ \gamma_{\mathrm{AGD}}\sigma_{\mathrm{AGD}}^{-1}\right) \sqrt{\kappa+1}. \label{eq:lambda}
\end{align}

The following result extends \eqref{eq:sigmaADQ} to DQ-AGD.
\begin{theorem}[Convergence of DQ-AGD]
\label{thm:DQ-AGD}
Fix a dimension-$n$, rate-$R$ quantizer $\fcn{q}$ with dynamic range $1$ and covering efficiency $\rho_n$. Then, Algorithm~\ref{alg:DQAGD}  
with stepsize \eqref{eq:stepsize_agd}, momentum coefficient \eqref{eq:gamma_agd},  
and dynamic ranges $r_{-2} = r_{-1} = 0$,
\begin{align}
\label{eqn:range-AGD}
\!\!\! r_t &=  \sigma_{\mathrm{AGD}}^t  LD \lambda +   \left( r_{t-1} +  \gamma_{\mathrm{AGD}} (r_{t-1} + r_{t-2}) \right) \rho_n 2^{-R},
\end{align}
 $t = 1, 2, \ldots$ in the definition of $\msf{q}_t$ \eqref{eqn:quant-scale} 
achieves the following contraction factor over $\mathcal F_n$ \eqref{eq:lcr}:
\begin{equation}
\!\!\!\!\!\! \sigma_{\mathrm{DQ-AGD}}(n, R) \leq \max\bp{\sigma_{\mathrm{AGD}}, \rho_n 2^{-R} \phi(n, R, \gamma_{\mathrm{AGD}}) }
\label{eq:sigmaDQ-AGD}
\end{equation}
where
\begin{align}
&~\phi(n, R, \gamma) \triangleq  \frac 1 2 ( 1 + \gamma)
+\frac 1 2 \sqrt{(1+\gamma)^2+\frac{4 \gamma}{\rho_n 2^{-R}}} . 
\label{eq:phi} 
\end{align}
\end{theorem}
\begin{proof}[Proof sketch]
The proof follows the roadmap of the proof of \thmref{thm:DQ-GD} with the following complication. Where in Algorithm~\ref{alg:DQGD_short} the quantizer input depends on the previous quantization error $\vec{e}_{t-1}$, the quantizer input in Algorithm~\ref{alg:DQAGD} depends on the past two quantization errors $\vec{e}_{t-1}$ and $\vec{e}_{t-2}$ (Line 5). The resulting recursion \eqref{eqn:range-AGD} is a second-order linear non-homogeneous recurrence relation, which unlike \eqref{eq:DR_memory} does not simply represent a geometric sequence. 
The characteristic polynomial of \eqref{eqn:range-AGD} is
\begin{equation}
\label{eqn:char}
\fcn{p}(r) \triDef r^2 - r \rho_n 2^{-R} (1+\gamma_{\mathrm{AGD}}) - \rho_n 2^{-R} \gamma_{\mathrm{AGD}},
\end{equation}
and $\rho_n 2^{-R} \phi_{\mathrm{DQ-AGD}}(n, R)$ in \eqref{eq:sigmaDQ-AGD} is its positive, larger-magnitude root. This implies that the quantization error decays with the contraction factor in the right side of \eqref{eq:sigmaDQ-AGD}. See Appendix~\ref{sec:DQ-AGD} for details. 
\end{proof}
Define the functions
\begin{align}
R_1(n, \gamma) &\triangleq \log_2 (1+2\gamma) + \log_2 \rho_n \label{eqn:suffi-conv-lin-rate} \\
R_2(n, \sigma, \gamma) &\triangleq \log_2 \frac{(1+\gamma)\sigma +\gamma}{\sigma^2} + \log_2 \rho_n
\label{eqn:suffi-same-as-unqtz}
\end{align}

The achievability bound \eqref{eq:sigmaDQ-AGD} exhibits two phase transitions. The first one is at $\rho_n 2^{-R} \phi(n, R, \gamma_{\mathrm{AGD}}) < 1$, which is equivalent to $\fcn{p}(1) > 0$: if
\begin{equation}
R > R_1(n, \gamma_{\mathrm{AGD}}) \quad \text{bits / dimension},
\end{equation}
then DQ-AGD enjoys linear convergence. The second one is at $\rho_n 2^{-R} \phi(n, R, \gamma_{\mathrm{AGD}}) \leq \sigma_{\mathrm{AGD}}$, which is equivalent to $\fcn{p}(\sigma_{\mathrm{AGD}}) \geq 0$: if
\begin{equation}
R \geq R_2(n, \sigma_{\mathrm{AGD}}, \gamma_{\mathrm{AGD}}) \quad \text{bits / dimension},
\end{equation}
then there is no loss in the long-term convergence behavior of the DQ-AGD compared to AGD. 

%Since $\gamma_{\mathrm{AGD}}$ \eqref{eq:gamma_agd} increases monotonically in $\kappa$,  taking $\kappa \to \infty$ in  $R_1(n, \gamma_{\mathrm{AGD}})$ gives the following sufficient condition for linear convergence:
%\begin{equation}
% R_1(n, \gamma_{\mathrm{AGD}}) > \frac{1}{3} + \log_2 \rho_n  \quad \text{bits / dimension},
%\label{eq:R1lb}
%\end{equation}

Curiously, $R_1(n, 0)$ and $R_2(n, \sigma_{\mathrm{GD}}, 0)$ express the two phase transitions of DQ-DG that were determined in \secref{sec:DQ-GDa}.

\subsection{DQ-HB: convergence analysis and numerical comparison}
Unquantized heavy ball method with stepsize
\begin{equation}
\label{eq:stepsize_hb}
\eta = \eta_{\mathrm{HB}} \triangleq \pr{\frac{2}{\sqrt{L}+\sqrt{\mu}}}^2
\end{equation}
and momentum coefficient
\begin{equation}
\gamma = \gamma_{\mathrm{HB}} \triangleq \pr{\frac{\sqrt{\kappa}-1}{\sqrt{\kappa}+1}}^2
\label{eq:gamma_hb}  
\end{equation}
achieves contraction factor 
\begin{equation}
\label{eqn:conv-lin-HB}
\sigma_{\mathrm{HB}} \triDef \frac{\sqrt{\kappa}-1}{\sqrt{\kappa}+1}
\end{equation}
over $\mathcal F_n^2$  \eqref{eq:lcr} \cite{Polyak-87} (\lemref{lma:HB}), which is optimal among all gradient methods 
\cite[Th.~2.1.13]{Nesterov-14} (\lemref{lma:lower-Nesterov}).

The following convergence analysis of DQ-HB applies to smooth and strongly convex functions that are in addition twice continuously differentiable.

\begin{theorem}[Convergence of DQ-HB]
\label{thm:DQ-HB}
Fix a dimension-$n$, rate-$R$ quantizer $\fcn{q}$ with dynamic range $1$ and covering efficiency $\rho_n$. Then, there exists a constant $\alpha > 0$ such that Algorithm~\ref{alg:DQHB} with stepsize \eqref{eq:stepsize_hb}, momentum coefficient \eqref{eq:gamma_hb} and dynamic ranges $r_{-1} = r_{-2} = 0$, 
\begin{align}
\label{eqn:range-HB}
r_t =  \sigma_{\mathrm{HB}}^t\, t^\alpha e^\alpha \sqrt 2\, L D +   \left( r_{t-1} + \gamma_{\mathrm{HB}} (r_{t-1} + r_{t-2}) \right) \rho_n 2^{-R},
\end{align}
$t = 1, 2, \ldots$ achieves the following contraction factor over $\mathcal F_{n}^2$:
\begin{equation}
\label{eqn:sigma-DQ-HB}
\begin{aligned}
\sigma_{\mathrm{DQ-HB}}(n, R) \leq \max\bp{\sigma_{\mathrm{HB}}, \rho_n 2^{-R} \phi(n, R, \gamma_{\mathrm{HB}}) }, 
\end{aligned}
\end{equation}
where $\phi(n, R, \gamma)$ is defined in \eqref{eq:phi}. 
\end{theorem}
\begin{proof}[Proof sketch]
The proof is similar to the proof of \thmref{thm:DQ-AGD}. The recurrence relation \eqref{eqn:range-HB} differs from \eqref{eqn:range-AGD} in only the presence of the subexponential factor $t^\alpha$, which does not matter when we take $t \to \infty$ to obtain \eqref{eqn:sigma-DQ-HB}. This factor arises from our nonasymptotic sharpening of Polyak's convergence result for the unquantized HB (Lemma~\ref{lma:HB}). See Appendix~\ref{sec:DQ-HB} for details. 
\end{proof}

DQ-HB exhibits two phase transitions, a behavior similar to DQ-HB and DQ-AGD. The two thresholds are given by $R_1$~\eqref{eqn:suffi-conv-lin-rate} and $R_2$~\eqref{eqn:suffi-same-as-unqtz} evaluated with $\sigma = \sigma_{\mathrm{HB}}$ and $\gamma = \gamma_{\mathrm{HB}}$. %The lower bound in \eqref{eq:R1lb} continues to hold for $R_1(n, \gamma_{\mathrm{HB}})$ since $\gamma_{\mathrm{HB}}$ \eqref{eq:gamma_hb} increases monotonically in $\kappa$. 

Plugging the parameters $\gamma$ and $\sigma$ into \eqref{eqn:suffi-same-as-unqtz}, we can infer that DQ-HB always has the largest $R_2$ \eqref{eqn:suffi-same-as-unqtz} for any condition number $\kappa\geq1$ among the three DQ schemes. On the other hand, whether $R_2$ of DQ-AGD is smaller than that of DQ-GD depends on whether $\kappa$ is smaller than a threshold that is roughly $2.18$. For the unquantized algorithms, contraction factor $\sigma_{\mathrm{HB}}$ of HB is always the smallest among the three for any condition number $\kappa\geq1$. On the other hand, whether $\sigma_{\mathrm{AGD}}$ of AGD is smaller than $\sigma_{\mathrm{GD}}$ of GD depends on whether $\kappa$ is greater than a threshold that is roughly $11.83$. For the differentially quantized algorithms, DQ-GD actually has the best convergence behavior in the transient regime where $R> \log_2 \rho_n$ so that GD converges linearly and $R$ is small enough so that DQ-HB does not yet outperform GD, i.e., $\rho_n 2^{-R} \phi(n, R, \gamma_{\mathrm{HB}}) > \sigma_{\mathrm{GD}}$. This is because DQ-GD is the first among the three DQ algorithms to pass $R_1$ \eqref{eqn:suffi-conv-lin-rate} above which it has linear convergence.

In Fig.~\ref{fig:exp-2}, we compare the performance of the differentially quantized algorithms DQ-HB (Algorithm~\ref{alg:DQHB}), DQ-AGD (Algorithm~\ref{alg:DQAGD}) and DQ-GD (Algorithm~\ref{alg:DQGD_short}) on least-squares problems \eqref{eq:obj_LS}. \hl{We use the same experimental setup as in \figref{fig:exp}, with the uniform scalar quantizer. The stepsize, the interpolation coefficient and the dynamic ranges are set to the values prescribed by Theorems~\ref{thm:DQ-GD} (DQ-GD),~\ref{thm:DQ-AGD} (ADQ-GD), and~\ref{thm:DQ-HB} (ADQ-HB).} We set $\alpha=0$ for Algorithm~\ref{alg:DQHB}, and DQ-HB still converges empirically for this parameter. We also record the performance of the corresponding unquantized gradient methods HB, AGD and GD. The curves exhibit the two phase transitions and comparative performance as discussed above. 
The level lines that the contraction factors of these DQ schemes rest on for $R\geq R_2$ are almost the same as the corresponding linear convergence rates $\sigma$ of their unquantized counterparts. We observe that there is a gap between the worst-case linear convergence rate $\sigma_{\mathrm{AGD}}$ that we design DQ-AGD to follow and the empirical convergence rate of AGD. This is because AGD applies for functions that are not necessarily twice continuously differentiable, and the least-squares problems \eqref{eq:obj_LS} happen not to be a worst-case problem class for AGD.

%For the condition numbers $\kappa$'s in Fig.~\ref{fig:exp-2}, we see that DQ-GD has the smallest first phase-transition threshold \eqref{eqn:suffi-conv-lin-rate}, above which it enjoys linear convergence. DQ-GD also has the smallest second phase-transition threshold \eqref{eqn:suffi-same-as-unqtz}, above which it has the same linear convergence rate as the unquantized GD.
%However, when $R$ is large enough so that all these DQ algorithms have the same convergence behavior as their unquantized counterparts, DQ-HB achieves the smallest linear convergence rate, whereas whether or not DQ-AGD converges faster than DQ-GD depends on whether or not $\kappa$ is greater than a threshold that is roughly $11.83$. %Moreover, these thresholds are shifted a bit to the right since we use the scalar uniform quantizer with covering efficiency $\rho_n=\sqrt{n}$ instead of $\rho_n=1+o_n(1)$ achieved by Rogers' vector quantizer that is hard to implement.

\begin{figure*}[t]
\begin{minipage}{0.5\textwidth}
\centering
\includegraphics[width=0.96\textwidth]{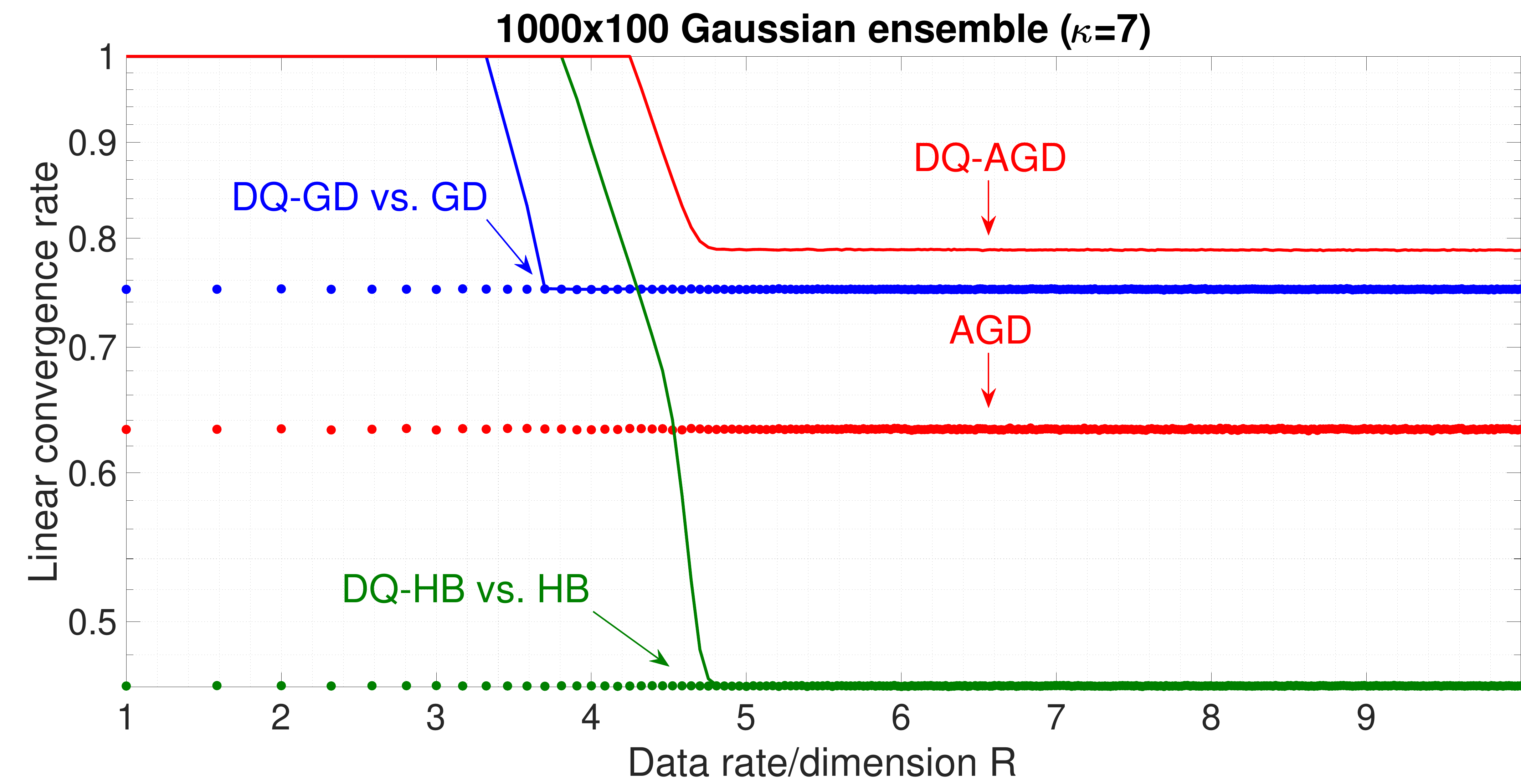}
\end{minipage}%
\begin{minipage}{0.5\textwidth}
\centering
\includegraphics[width=0.96\textwidth]{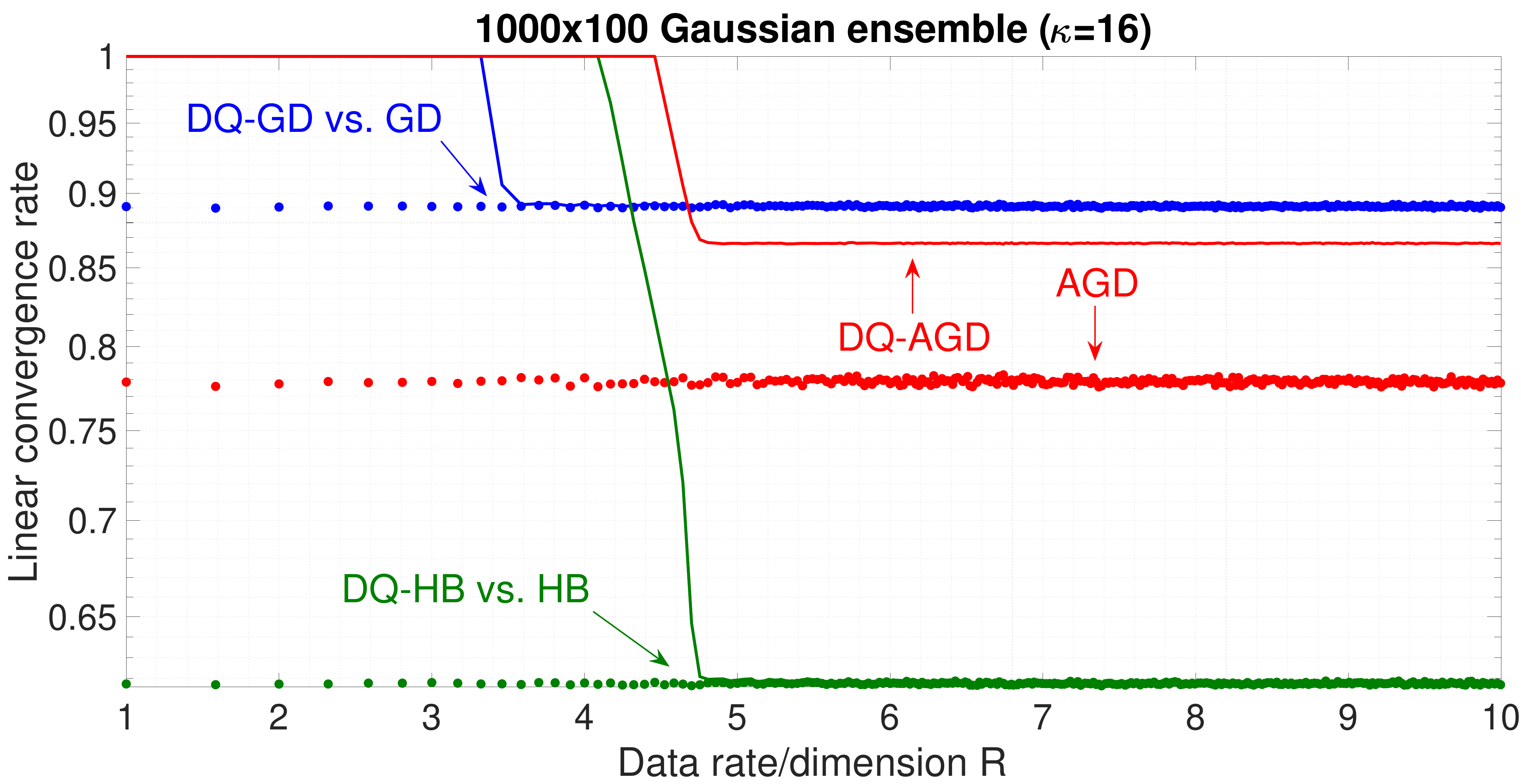}
\end{minipage}
%\caption{Linear convergence rates of various DQ algorithms (plotted as lines) and their corresponding unquantized counterparts (plotted as circles).}
%\label{fig:exp-1}
\end{figure*}

\begin{figure*}[t]
\begin{minipage}{0.5\textwidth}
\centering
\includegraphics[width=0.96\textwidth]{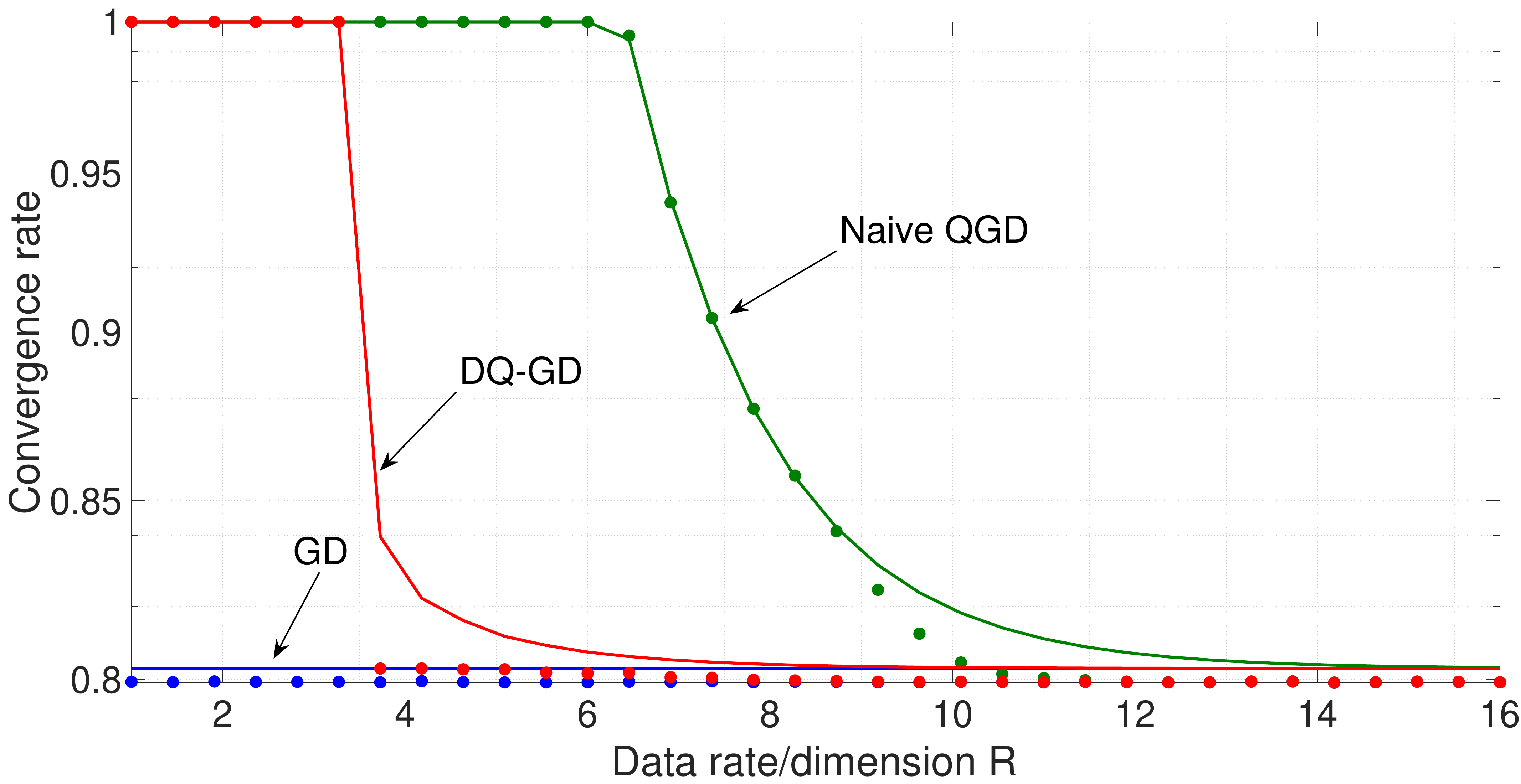}
\end{minipage}%
\begin{minipage}{0.5\textwidth}
\centering
\includegraphics[width=0.96\textwidth]{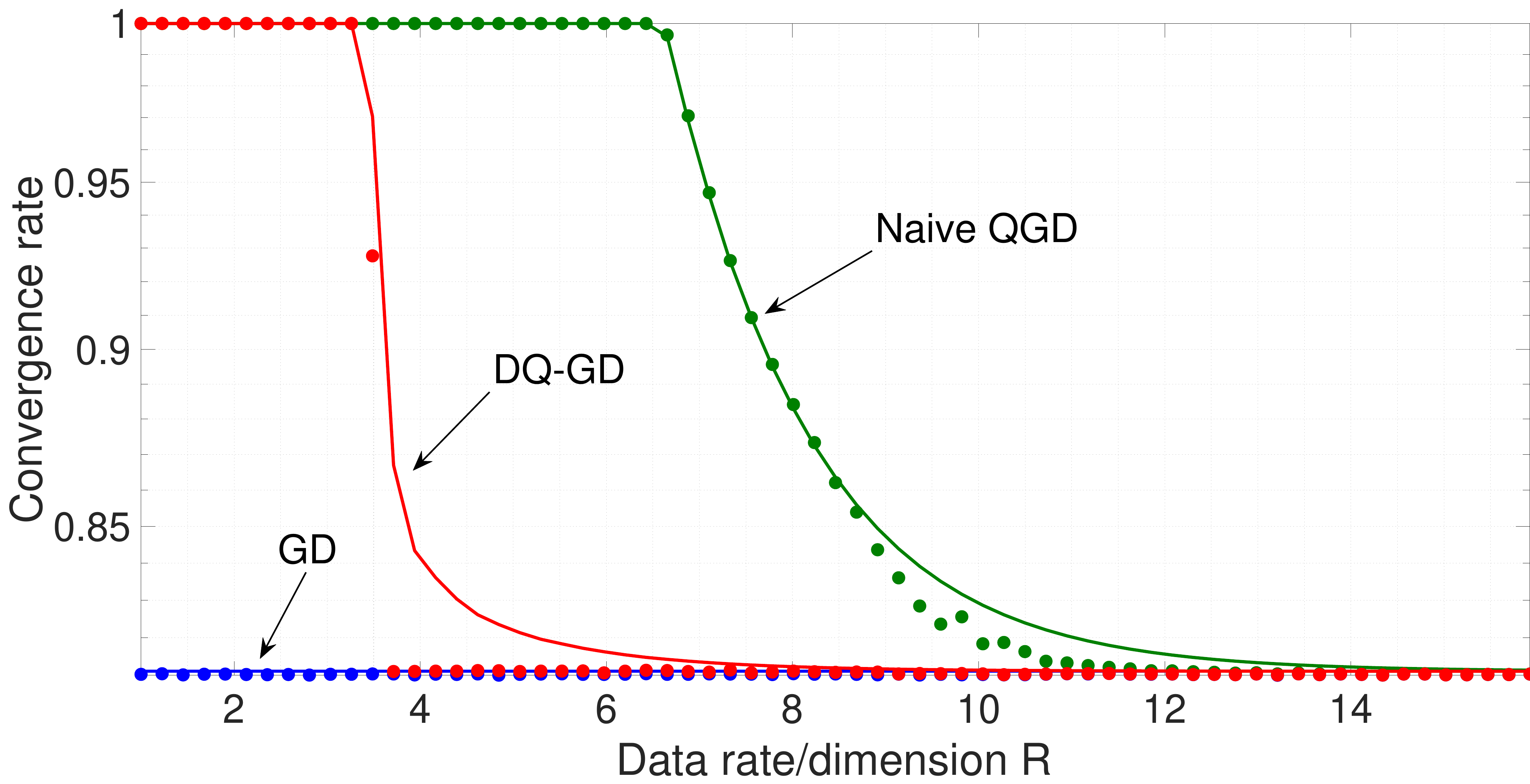}
\end{minipage}
\caption{Empirical contraction factors of various DQ algorithms (plotted as lines) and their corresponding unquantized counterparts (plotted as circles).}
\label{fig:exp-2}
\end{figure*}

\section{Converses}
\label{sec:converse}

\subsection{Quantized gradient descent algorithms}
In this section, we characterize the optimal contraction factor
achievable within class $\mathcal A_{\mathrm{GD}}$ of quantized gradient descent algorithms, formally defined next.

\begin{definition}[Class $\mathcal A_{\mathrm{GD}}$ of quantized gradient descent algorithms]
\label{def:QGD}
A \emph{quantized gradient descent} algorithm $\textrm A \in \mathcal A_{\mathrm{GD}}$ consists of a central server and an end worker. The algorithm is initialized with a collection of quantizers $\fcn{q}$ indexed by rate $R$ such that   $d(\fcn{q}) \to 0$ \eqref{eq:DR} as $R \to \infty$ and a sequence of dynamic ranges $\{r_t\}_{t = 1}^\infty$.  The worker has access to the function~$\msf{f}$. At each iteration $t$, the server first sends $\hat{\bm{x}}_t$ to the worker noiselessly, starting from some $\hat{\bm{x}}_0\in\mbb{R}^n$. The worker then determines its gradient-access point $\bm{z}_{t}$ and its quantizer input $\bm{u}_{t}$ under the structural constraints
\begin{align}
\bm{z}_{t} &\in \hat{\bm{x}}_t + \spn\lbp \bm{e}_{0}\,,\ldots,\bm{e}_{t-1}\rbp \label{eq:query_condi} \\
\bm{u}_{t} &\in \nabla\msf{f}(\bm{z}_{t}) + \spn\lbp \bm{e}_{0}\,,\ldots,\bm{e}_{t-1}\rbp, \label{eq:input_condi}
\end{align}
where $\bm{e}_i \triangleq \bm{q}_i-\bm{u}_i, i = 0, \ldots, t - 1$ are the past quantization errors before iteration $t$, and $+$ denotes Minkowski's sum. Upon receiving $\bm{q}_{t} = \fcn{q}_t(\bm{u}_t)$ \eqref{eqn:quant-scale} from the worker, the server performs the update
\begin{equation}
\label{eq:QGD}
\hat{\bm{x}}_{t+1} = \hat{\bm{x}}_t - \eta\bm{q}_t
\end{equation}
with a fixed stepsize $\eta>0$.
\end{definition}
Due to conditions \eqref{eq:query_condi} and \eqref{eq:input_condi},  if there is no quantization error at each iteration (i.e., if $R = \infty$), then any quantized algorithm in $\mathcal A_{\mathrm{GD}}$ reduces to the unquantized gradient descent. Both DQ-GD and NQ-GD fall in the class $\mathcal A_{\mathrm{GD}}$. 

\begin{theorem}[Converse within class $\mathcal A_{\mathrm{GD}}$]
\label{thm:converseGD}
The contraction factor achievable over functions $\mathsf f \in \mathcal F_n$ within class $\mathcal A_{\mathrm{GD}}$ of algorithms satisfies
\begin{equation}
 \inf_{\textrm A \in \mathcal A_{\mathrm{GD}}}\sigma_{\textrm A}(n, R) \geq \max\bp{\sigma_{\mathrm{GD}}, 2^{-R}}
 \label{eq:converseGD}
 \end{equation}
\end{theorem}
\begin{proof}[Proof sketch]
We fix an $\textrm A \in \mathcal A_{\mathrm{GD}}$, and we lower-bound the contraction factor it achieves at rate $R$ in two different ways. On one hand, we show that $\textrm A$ cannot converge faster than the unquantized GD.  Then, we use an argument similar to \cite{Klerk-17} to craft a worst-case problem instance $\msf{g} \in\set{F}_n$ for which the iterates of the unquantized GD satisfy $\lnorm \bm{x}_{t+1}-\bm{x}_{\msf{g}}^*\rnorm = \sigma_{\mathrm{GD}}\lnorm \bm{x}_t-\bm{x}_{\msf{g}}^*\rnorm$, which ensures that $ \inf_{\textrm A \in \mathcal A_{\mathrm{GD}}}\sigma_{\textrm A}(n, R) \geq \sigma_{\mathrm{GD}}$ \hl{(Lemma~\ref{lma:GD_opt})}. On the other hand, we notice that if $\textrm A$ is applied at dimension $n$ and rate $R$, then the set $\mcal{S}_{\text{A}} \subseteq \mathbb R^n$ of all possible locations of the iterate $\hat{\bm{x}}_T$ after $T$ iterations of $\textrm A$ has cardinality at most $2^{nRT}$, and we apply a volume-division argument to claim that $ \inf_{\textrm A \in \mathcal A_{\mathrm{GD}}}\sigma_{\textrm A}(n, R) \geq 2^R$. See Appendix~\ref{sec:DQ-GDc} for details.
\end{proof}
Applying Theorem~\ref{thm:DQ-GD} with Rogers-optimal quantizers with $\rho_n \to 1$ \cite[Th. 3]{Rogers-63} and juxtaposing with Theorem~\ref{thm:converseGD},  we characterize the optimal contraction factor achievable by quantized gradient descent in the limit of large problem dimension as 
\begin{equation}
\lim_{n \to \infty} \inf_{\textrm A \in \mathcal A_{\mathrm{GD}}}\sigma_{\textrm A}(n, R) = \max\bp{\sigma_{\mathrm{GD}}, 2^{-R}}.
\label{eq:ratefnGD}
\end{equation}
In other words, DQ-DG achieves the best possible contraction factor within $\mathcal A_{\mathrm{GD}}$, in the limit of large problem dimension. This is rather remarkable: it means not only that DQ-DG compensates previous quantization errors optimally so that no rate is wasted, but that our convergence analysis in Theorem~\ref{thm:DQ-GD} is tight enough to capture this optimality. Furthermore, notice that the right side of \eqref{eq:ratefnGD} is $< 1$ at any $R>0$. This means that at any $R > 0$ however small, DQ-DG with Rogers-optimal quantizers converges linearly at a large enough problem dimension $n$, i.e. the first phase transition \eqref{eqn:suffi-conv-lin-rate} dissappears.  

\hl{Although the notion of a contraction factor \eqref{eq:lcr} and thus the result in \eqref{eq:ratefnGD} are asymptotic in the number of iterations $T$, the achievability results in Appendix~\ref{sec:DQ-DGa} and converse results in Appendix~\ref{sec:DQ-DGa} used to derive \eqref{eq:ratefnGD} are nonasymptotic. They show that gap between the achievability and converse bounds on the finite-$T$ counterpart of \eqref{eq:ratefnGD} is $O \left(\frac 1 T \right)$. Whether DQ-GD remains optimal at finite $T$ remains an open problem.}

\subsection{Quantized gradient methods}
All quantized algorithms considered in this paper fall in the following class. 
\begin{definition}[Class $\mathcal A_{\mathrm{GM}}$ of quantized gradient methods]
\label{def:QGM}
A \emph{quantized gradient method} $\textrm A \in \mathcal A_{\mathrm{GM}}$ follows Definition~\ref{def:QGD} with \eqref{eq:QGD} relaxed to
\begin{equation}
\label{eq:QGM}
\hat{\bm{x}}_{t+1} \in \hat{\bm{x}}_0 + \mrm{span}\lbp \bm{q}_0,\ldots,\bm{q}_t\rbp.
\end{equation}
 \end{definition}

In the absence of rate constraints, there are no quantization errors, i.e. $\vec{e}_{t}=\vec{0}$ for all $t$, and the class of quantized gradient methods reduces to the class of unquantized gradient methods satisfying
\begin{equation}
\label{eqn:descent-GM}
\vec{x}_{t+1} \in \vec{x}_0 + \mrm{span}\bp{\grad\fcn{f}(\vec{x}_0),\ldots,\grad\fcn{f}(\vec{x}_t)}.
\end{equation}

To present our converse result for $\mathcal A_{\mathrm{GM}}$, we consider functions $\mathsf f$ defined on the square-summable Hilbert space\footnote{We do so to take advantage of the sharpest converse in the literature on the convergence of unquantized gradient methods \eqref{eqn:descent-GM} \cite{Nesterov-14} (\lemref{lma:lower-Nesterov}), which applies to functions on $\mathbb L_2$. Convergence lower bounds for smooth and strongly convex functions on $\R^n$ (rather than $\mathbb L_2$) are also known \cite{Arjevani-16}. However, \cite{Arjevani-16} considers only quadratic functions as objectives, and the considered class of iterative algorithms is more restrictive than \eqref{eqn:descent-GM} in that the next iterate $\vec{x}_{t+1}$ depends on the past $p$ terms $\vec{x}_t,\ldots,\vec{x}_{t-p+1}$ for some fixed $p\in\N$.
}
\begin{equation}
\label{eqn:class-iter}
\textstyle
\mathbb L_2 \triDef \bp{\vec{x}=[\vec x(1),\vec x(2),\ldots] \colon \sum_{i=1}^{\infty}\vec x(i)^2 < \infty}.
\end{equation}
We say that continuously differentiable function $\mathsf f \colon \mathbb L_2 \mapsto \mathbb R$ is in class $\mathcal F_{\infty}$ if it is $L$-smooth, $\mu$-strongly convex and its minimizer is bounded, i.e., $\mathsf f$ satisfies i)-iii) in \secref{sec:functiondef}. 

To quantize an infinitely long vector $\vec{u}\in \mathbb L_2$ to $\vec{q}\in \mathbb L_2$, we fix a free parameter $n \in \mathbb N$, apply a rate-$R$ quantizer $\fcn{q}$ in $\R^n$ \eqref{eqn:def-qnt-rate} to the first $n$ coordinates of $\vec{u}$, and set the remaining coordinates to $0$, i.e., 
\begin{equation}
\begin{cases}
[ \vec q (1),\ldots,\vec q (n) ] &= \fcn{q}\left( [ \vec u(1),\ldots, \vec u(n)] \right) \\
~~~~~~~~~~~~~~\vec q (i) &= 0 \quad\forall i>n, \label{eqn:quant-l2-null}
\end{cases}
\end{equation}
where $\vec u(i)$ denotes $i$-th coordinate of vector $\vec u \in \mathbb L_2$. 

Although only $n$ coordinates $\vec{u}\in\ell_2$ are quantized, we can still control the overall quantization error since in $\mathbb L_2$,
\begin{equation}
\label{eqn:l2-tail}
\sum_{i>n}\vec u (i)^2 = o_n(1)
\end{equation}
due to the Cauchy convergence criterion. Here $o_n(1)$ denotes a function that vanishes as $n \to \infty$. Thus, \eqref{eq:e_d} continues to hold for quantization in $\mathbb L_2$ with $\rho_n$ replaced by $\rho_n + o_n(1)$. It follows that the achievability bounds in Theorems~\ref{thm:DQ-AGD} and~\ref{thm:DQ-HB} with with $\rho_n$ replaced by $\rho_n + o_n(1)$ apply to functions $\mathsf f \in \mathcal F_{\infty}$. 

Contraction factor $\sigma_{\textrm A}(n, R)$ over $\mathcal F_{\infty}$ is defined in the same way as that over $\mathcal F_n$ \eqref{eq:lcr} except that $n$ is now a parameter of the employed quantizer (like $\rho_n$) rather than the dimension of the problem, and the total number of bits sent per iteration is $n R$, where $R$ is the quantizer's rate \eqref{eqn:def-qnt-rate}.

\begin{theorem}[Converse within class $\mathcal A_{\mathrm{GM}}$]
\label{thm:converseGM}
The contraction factor achievable over functions $\mathsf f \in \mathcal F_{\infty}$ within class $\mathcal A_{\mathrm{GM}}$ of algorithms satisfies
\begin{equation}
 \inf_{\textrm A \in \mathcal A_{\mathrm{GM}}}\sigma_{\textrm A}(n, R) \geq \max\bp{\sigma_{\mathrm{HB}}, 2^{-R}}
 \label{eq:converseGM}
 \end{equation}
 where $\sigma_{\mathrm{HB}}$ is given in \eqref{eqn:conv-lin-HB}.
\end{theorem}
\begin{proof}[Proof sketch]
The proof is similar to that of \thmref{thm:converseGD}: we apply a volume-division argument to recover the $2^{-R}$ in the right side of \eqref{eq:converseGM}, and we apply a known result on unquantized gradient methods that states that the best contraction factor achievable over $\mathcal F_{\infty}$ is that of the heavy ball method, $\sigma_{\mathrm{HB}}$ \cite{Nesterov-14} (\lemref{lma:lower-Nesterov}). See Appendix~\ref{sec:DQ-GMc}. 
\end{proof}

Together, Theorems~\ref{thm:converseGM} and~\ref{thm:DQ-HB} imply that for any $R \geq R_2(\infty, \sigma_{\mathrm{HB}}, \gamma_{\mathrm{HB}})$, DQ-HB attains the optimal contraction factor within $\mathcal A_{\mathrm{GM}}$ (under the additional assumption that $\mathsf f \in \mathcal F_{\infty}$ is twice continuously differentiable). 
\hl{The nonasymptotic achievability and converse results in Appendices~\ref{sec:DQ-HB} and~\ref{lma:lower-Nesterov} used to show this result leave open the question of whether DQ-HB is optimal at finite number of iterations $T$, as they determine the finite-length analog of the optimal contraction factor only with accuracy $O \left( \frac{\log T}{T}\right)$. }

\section{Multi-worker gradient methods}
\label{sec:multiworker}
\subsection{Problem setup}
In empirical risk minimization \cite{BoydVandenberghe-04,AbuMostafa-12}, the sample average of the loss function on the data points
\begin{equation}
\label{eq:obj}
\msf{f}(\bm{x}) = \frac{1}{K}\sum_{k=1}^K\msf{f}_k(\bm{x})
\end{equation}
arises as a substitute for the expected loss on the true data distribution that is often unknown. In multi-worker distributed empirical risk minimization, each worker has access to only one of the summands in \eqref{eq:obj}, and they communicate to the parameter server under rate constraints. See Figure~\ref{fig:diagramK}.  

\begin{figure}[t]
\centering
\includegraphics[width=0.40\textwidth]{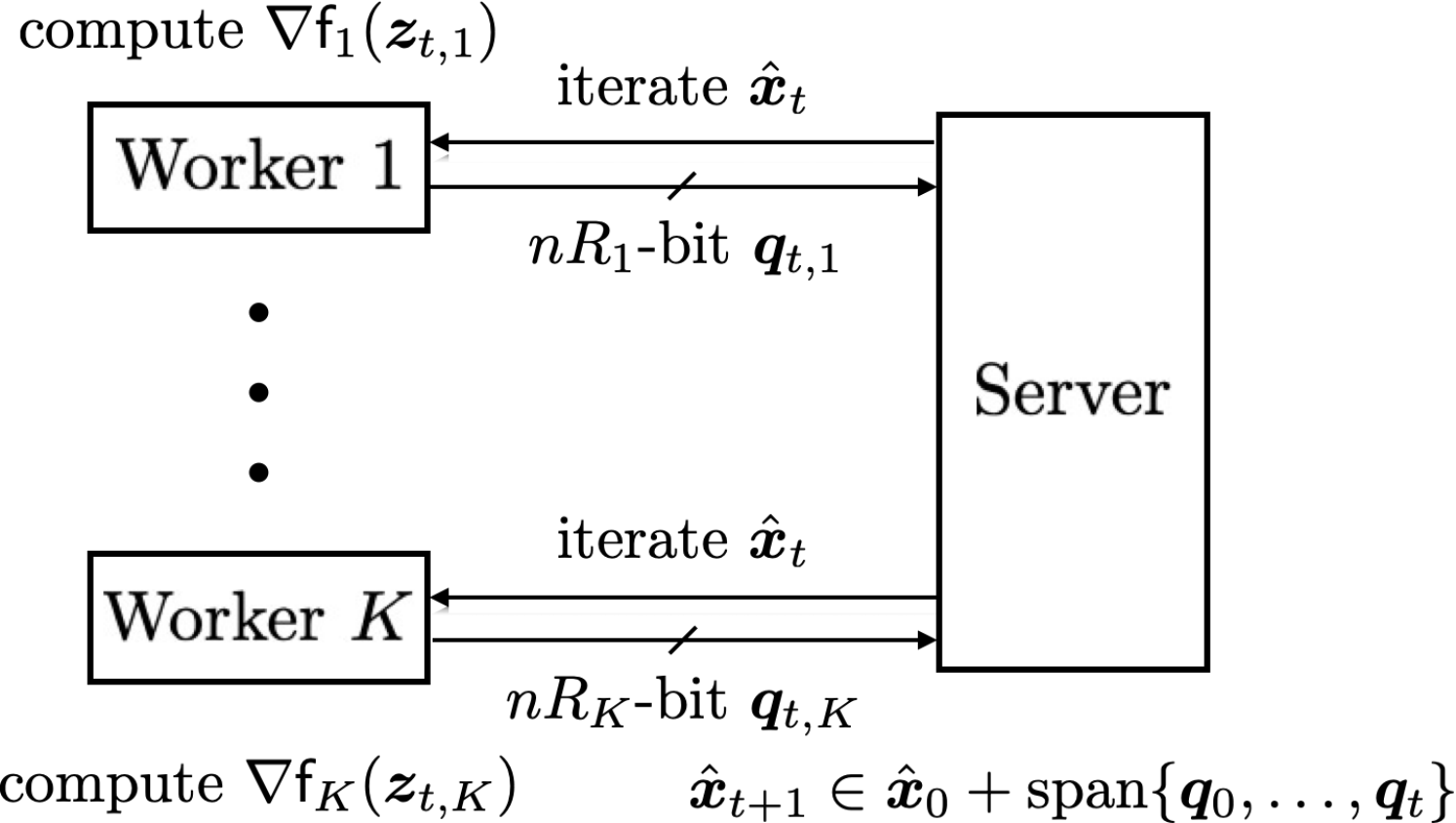}
\caption{$K$-worker quantized gradient method. At each iteration $t$, the server broadcasts the current iterate $\hat{\bm{x}}_t$. Worker $k$ computes the gradient at some point $\bm{z}_{t,k}$ that is a function of (but not necessarily equal to) $\hat{\bm{x}}_t$. Then, worker $k$ forms a descent direction $\bm{q}_{t,k}$ and pushes it back to the server under an $nR_k$-bit rate constraint.}
\label{fig:diagramK}
\end{figure}

\subsection{Converses}
Definitions~\ref{def:QGD} and~\ref{def:QGM} extend naturally to the $K$-worker setting.
Converses in Theorems~\ref{thm:converseGD} and ~\ref{thm:converseGM} extend verbatim to the multi-worker setting where the workers' rates satisfy the sum-rate constraint
\begin{equation}
\label{eq:rate_condi}
\sum_{k=1}^K R_k \leq R.
\end{equation}

\subsection{Differential quantization}
Differential quantization does not apply to $K$-worker quantized gradient methods since each worker does not know the local quantization errors stored by the others, and thus cannot guide the descent trajectory back to the unquantized path. Thus, whether \eqref{eq:converseGD} and \eqref{eq:converseGM} are attainable in the multiworker setting, and how each worker should optimally compensate its own past quantization errors, remain open problems.

\subsection{Naively Quantized Gradient Descent}
The Naively Quantized Gradient Descent applies a common method of quantizing distributed gradient algorithms \cite{FriedlanderSchmidt-12,Alistarh-17,Wen-17,Bernstein-18,MayekarTyagi-19,RamezaniKebrya-19} in which each worker quantizes the gradient of the current iterate, to GD. It is summarized as Algorithm~\ref{alg:NQGD}.     

\begin{algorithm2e}[ht]
\caption{$K$-worker NQ-GD}
\label{alg:NQGD}
   \DontPrintSemicolon
   \SetKw{KwW}{{\small\textsf{Worker $k$}}:}
   \SetKw{KwS}{{\small\textsf{Server}}:}
   \SetKwIF{IfW}{ElseIfW}{ElseW}{if}{then}{else if}{\KwW}{endif}
   \For{$t = 0, 1, 2, \ldots$}{
      \For{$k=1$ \KwTo $K$}{
         \SetAlgoVlined
         \ElseW{
         %$\bm{u}_{t,k}=\nabla\msf{f}_k(\hat{\bm{x}}_t)$\;
          % $\bm{q}_{t,k}=\msf{q}_{t,k}(\bm{u}_{t,k})$\;
             $\bm{q}_{t,k}=\msf{q}_{t,k}(\nabla\msf{f}_k(\hat{\bm{x}}_t))$\;
         }
         \KwS $\hat{\bm{x}}_{t+1}\la\hat{\bm{x}}_t-\frac{\eta}{K}\sum_{k=1}^K\bm{q}_{t,k}$\;
      }
   }
\end{algorithm2e}

Our convergence result for NQ-GD holds under the following assumptions. We assume that continuously differentiable summands $\msf{f}_k$ in \eqref{eq:obj} are (i) $L_k$-smooth and (ii) $\mu_k$-strongly convex, and we continue to assume that (iii) the optimizer of $\msf f$ is bounded as in \eqref{eq:iter_opt_range}. Note that $\msf{f}$ is $L$-smooth and $\mu$-strongly convex with
\begin{equation}
\label{eq:smooth_strcvx_ave}
L \eqDef \frac{1}{K}\sum_{k=1}^KL_k \quad\text{ and }\quad \mu \eqDef \frac{1}{K}\sum_{k=1}^K\mu_k.
\end{equation}
Further, we focus on the \emph{interpolation setting} \cite{SchmidtLeRoux-13,Needell-14,Ma-18} that assumes (iv)
\begin{equation}
\label{eqn:condi-str-grow}
\vec{x}_{\fcn{f}}^* = \vec{x}_{\fcn{f}_k}^* \quad\forall k=1,\ldots,K,
\end{equation}
where
\begin{equation}
\label{eqn:iter-opt-local}
\vec{x}_{\fcn{f}_k}^* \triDef \argmin_{\vec{x}\in\R^n} \fcn{f}_k(\vec{x}).
\end{equation}
The interpolation setting is motivated by the observation that almost all local minima are also global in an over-parametrized neural network with a very large data dimension $n$ \cite{Ma-18}, and is implied  We denote by $\mathcal G_n$ the class of functions $\fcn f$ \eqref{eq:obj} that satisfy the assumptions (i)-(iv). 

The minimum contraction factor achievable by $K$-worker NQ-GD under the sum rate constraint \eqref{eq:rate_condi} is given by
\begin{align}
&~\sigma_{\textrm {NQ-GD} }(n, R)  \notag\\
\triangleq &~ \inf_{\sum_{k=1}^K R_k \leq R} \sup_{\msf f \in \mathcal G_n} \limsup_{T \to \infty} {\norm{\hat{\vec{x}}_T(\{R_k\}_{k = 1}^K)-\vec{x}_{\fcn{f}}^*}}^{\frac 1 T},
\label{eq:lcrK}
\end{align}
where $\hat{\vec{x}}_0(\{R_k\}_{k = 1}^K),\hat{\vec{x}}_1(\{R_k\}_{k = 1}^K), \hat{\vec{x}}_2(\{R_k\}_{k = 1}^K),\ldots$ is the sequence of iterates generated by NQ-GD (Algorithm~\ref{alg:NQGD}) in response to $\msf f \in \mathcal G_n$ when the $k$-th worker operates at $R_k$ bits per problem dimension, $k =1, \ldots, K$. 

\begin{theorem}[Convergence of $K$-worker NQ-GD]
\label{thm:NQ-GD}
Fix a dimension-$n$, rate-$R_k$ quantizer $\fcn{q}_k$ with dynamic range $1$ and covering efficiency $\rho_n$. Set up the quantizer to be used by worker $k$ at iteration $t$ as
\begin{equation}
\label{eqn:quant-scalek}
\fcn{q}_{t, k}(\cdot) = r_{t, k}\fcn{q}_k(\cdot/r_{t, k}),
\end{equation} 
where the dynamic ranges are given by
\begin{equation}
\label{eq:range_NQGD}
r_{t,k} = \lp \sigma_{\mathrm{GD}} + \frac{\eta_{\mathrm{GD}}\, \rho_n }{K}\sum_{k=1}^K\min\left\{\nu, L_k \right\} \rp^t L_k D,
\end{equation}
and the optimum rate allocation is given by the waterfilling solution
\begin{equation}
\label{eq:rate_alloc_non_unif}
R_k = |\log_2 \left(L_k / \nu \right)|_+ \quad \text{bits,}
\end{equation}
where $\nu$ is the water level found from the sum rate constraint
\begin{equation}
\sum_{k = 1}^K  |\log_2 \left(L_k / \nu \right)|_+ =  R, 
\end{equation}
and $|\cdot|_+ \triangleq \max\{0, \cdot\}$. 
Then, Algorithm~\ref{alg:NQGD} with stepsize \eqref{eq:stepsize_gd} 
achieves the following contraction factor over $\mathcal G_n$ \eqref{eq:lcrK}:
\begin{equation}
\sigma_{\mathrm{NQ-GD}}(n, R) \leq \sigma_{\mathrm{GD}} + \frac{\eta_{\mathrm{GD}}\, \rho_n }{K}\sum_{k=1}^K\min\left\{\nu, L_k \right\}.
\label{eq:sigmaNQ-GDK}
\end{equation}
\end{theorem}
\begin{proof}
 Appendix~\ref{appx:NQ-GD}. 
\end{proof}

According to \eqref{eq:rate_alloc_non_unif}, higher rates are allocated to users whose function gradients  have higher Lipschitz constants, and if the Lipschitz constant is low enough in comparison to others no rate would be allocated at all. In the special case $L_k \equiv L$, \eqref{eq:sigmaNQ-GDK} reduces to
\begin{equation}
 \sigma_{\mathrm{NQ-GD}}(n, R) \leq \sigma_{\mathrm{GD}} + \frac{2\kappa}{\kappa + 1} \frac{\rho_n}{2^{R/K}}.
\end{equation}
 
 The bound in \thmref{thm:NQ-GD} approaches the converse only in the limit of large $R$.
 
\hl{Without assumption \eqref{eqn:condi-str-grow} that all summands share the minimizer, NQ-GD converges only to a neighborhood of $\vec x_{\fcn f}^*$; the size of the neighborhood is controlled by the quantization error and vanishes as $R \to \infty$ (\thmref{thm:NQ-GDprime} in  Appendix~\ref{appx:NQ-GD}).}

\section{Conclusion}

This paper formalizes the problem of finding the optimal contraction factor achievable within a class of rate-constrained iterative optimization algorithms (\eqref{eq:lcr}, Definitions~\ref{def:QGD} and~\ref{def:QGM}). We show information-theoretic converses to that fundamental limit (\thmref{thm:converseGD}, \thmref{thm:converseGM}). 

We introduce the principle of \emph{differential quantization} that posits that the quantizer's input shall be constructed in such a way as to guide the quantized algorithm's trajectory towards the unquantized trajectory. Applied to gradient descent (Algorithm~\ref{alg:DQGD_short}), differential quantization leads to the contraction factor that is optimal within the class of quantized gradient descent algorithms (\eqref{eq:ratefnGD}). Thus, differential quantization leverages the memory of past quantized inputs in an optimal way, removing the impact of past quantization errors. 

Beyond gradient descent, we apply differential quantization to gradient methods with momentum - the accelerated gradient descent (Algorithm~\ref{alg:DQAGD}) and the heavy ball method (Algorithm~\ref{alg:DQHB}). In all three cases, differentially quantized algorithms attain the contraction factor of their unquantized counterparts as long as the data rate exceeds the corresponding threshold $R_2$ \eqref{eqn:suffi-same-as-unqtz}. 

Incidentally, in the course of the analysis, we provide a sharper bound on the convergence of the unquantized HB algorithm than available in the literature (\lemref{lma:HB}). We also provide a slightly more general worst-case problem instance for the unquantized GD than available (\lemref{lma:GD_opt}). 

Quantizers employed at each step have the same geometry (covering efficiency, \eqref{eq:def_d}) but different resolution (covering radius, \eqref{eq:covradius}). The resolution is controlled by scaling the quantizer's dynamic range \eqref{eqn:quant-scale}. To attain the contraction factors in Theorems~\ref{thm:DQ-GD},~\ref{thm:DQ-AGD}, and~\ref{thm:DQ-HB}, the dynamic range is set to follow a recursion (\eqref{eq:DR_memory}, \eqref{eqn:range-AGD}, \eqref{eqn:range-HB}). That recursion shrinks the dynamic range at the fastest possible rate that still guarantees that the quantizer's input at each iteration falls within its dynamic range. This maximizes the usefulness of the bits exchanged at each iteration. While that recursion for DQ-GD \eqref{eq:DR_memory} is simply a geometric sequence, those for DQ-AGD \eqref{eqn:range-AGD} and DQ-HB \eqref{eqn:range-HB} are second-order linear non-homogeneous recurrence relations. 

While DQ-GD attains the optimal contraction factor among quantized gradient descent algorithms (Definition~\ref{def:QGD}) and DQ-HB attains the optimal contraction factor among all gradient methods, even unquantized, if $R \geq R_2(n, \sigma_{\mathrm{HB}}, \gamma_{\mathrm{HB}})$  \eqref{eqn:suffi-same-as-unqtz}, it remains an open problem whether the contraction factor of $2^{-R}$ dictated by the converse (\thmref{thm:converseGM}) is achievable in the regime $R_2(n, \sigma_{\mathrm{GD}}, 0)  < R <  R_2(n, \sigma_{\mathrm{HB}}, \gamma_{\mathrm{HB}})$ in the class of quantized gradient methods (Definition~\ref{def:QGM}). 

For multi-worker gradient descent, we provide a convergence result on naive quantization, in which the workers directly quantize their gradients, and show a waterfilling solution to optimize the allocation of data rates among the workers under the sum rate constraint (\thmref{thm:NQ-GD}). That result approaches the converse (\thmref{thm:converseGD}) only in the limit $R \to \infty$, leaving open a tighter characterization of the optimum convergence factor in that scenario. Differential quantization does not directly apply to multi-worker optimization since the workers cannot compute the unquantized path without the knowledge of the local quantization errors stored by the others. We leave as future work the question of how they should optimally compensate their own quantization errors. 

\section*{Acknowledgment}

The authors would like to thank Dr.~Himanshu~Tyagi for pointing out related works \cite{MayekarTyagi-19,MayekarTyagi-20}; Dr.~Vincent~Tan for bringing a known result on the worst-case contraction factor of unquantized GD \cite{Klerk-17} to our attention; Dr.~Victor~Kozyakin for a helpful discussion about joint spectral radius; and two anonymous reviewers for detailed comments. 

\appendices

\section{DQ-GD with varying stepsize}
\label{sec:DQGDvary}
See Algorithm~\ref{alg:DQGD_full}, below. 
\begin{algorithm2e}[!h]
\caption{DQ-GD with varying stepsize}
\label{alg:DQGD_full}
\DontPrintSemicolon
\SetAlgoSkip{\medskip}
\SetKw{KwInv}{{\color{white}Input:}}
\SetKw{KwW}{{\small\textsf{Worker}}:}
\SetKw{KwS}{{\small\textsf{Server}}:}
\SetKwIF{IfC}{ElseIfC}{ElseC}{if}{then}{else if}{\KwW}{endif}
%\KwIn{number of iterations $T$, a sequence of stepsizes $\{\eta_t\}_{t=0}^{T-1}$, and a sequence of quantizers $\{\msf{q}_t\}_{t=0}^{T-1}$
%}
   Initialize $\bm{e}_{-1}= \hat{\bm{x}}_0 = \bm{0}$\;
\For{$t=0, 1, 2, \ldots$}{
    \SetAlgoVlined
    \ElseC{$\bm{z}_t=\hat{\bm{x}}_t+\eta_{t-1}\bm{e}_{t-1}$\;
         $\bm{u}_t=\nabla\msf{f}(\bm{z}_t)-(\eta_{t-1}/\eta_t)\bm{e}_{t-1}$\;
         $\bm{q}_t=\msf{q}_t(\bm{u}_t)$\;
         $\bm{e}_t=\bm{q}_t-\bm{u}_t$\;
    }
    \KwS  $\hat{\bm{x}}_{t+1}=\hat{\bm{x}}_t-\eta_t\bm{q}_t$\;
}
%\KwOut{estimated optimizer $\hat{\bm{x}}_{T}$}
\end{algorithm2e}

\section{Convergence analyses of DQ algorithms}
\label{sec:DQ}
\subsection{Proof of \thmref{thm:DQ-GD}}
\label{sec:DQ-DGa}
As mentioned in the proof sketch,  relation \eqref{eq:recur_memory_short} is key to showing \thmref{thm:DQ-GD}. The next lemma, which applies to the more general version of the DQ--DG algorithm shown in Appendix~\ref{sec:DQGDvary}, establishes \eqref{eq:recur_memory_short}.
\begin{lemma}[DQ-GD trajectory]
\label{lma:recur_memory_full}
Consider descent trajectories $\{\hat{\bm{x}}_t\}$ of Algorithm~\ref{alg:DQGD_full} and $\{{\bm{x}}_t\}$ of unquantized GD \eqref{eq:GD} with the same sequence of stepsizes $\{\eta_t\}$ starting at the same location
$\hat{\bm{x}}_0 = \bm{x}_0$. 
Then, at each iteration $t$,
\begin{equation}
\label{eq:recur_memory_full}
\hat{\bm{x}}_t = \bm{x}_t - \eta_{t-1}\bm{e}_{t-1}.
\end{equation}
\end{lemma}

\begin{proof}%[Proof of {Lemma~\ref{lma:recur_memory_full}}]
We prove \eqref{eq:recur_memory_full} via mathematical induction.
\begin{itemize}
\item Base case: \eqref{eq:recur_memory_full} holds for $t=0$ since the starting location is the same. 
\item Inductive step: Suppose \eqref{eq:recur_memory_full} holds for iteration $t$. First, the induction hypothesis, the quantizer input at Line 5 and the quantizer output at Line 6 of Algorithm~\ref{alg:DQGD_full} together imply
\begin{equation}
\label{eq:input_memory_alt}
\bm{u}_t = \nabla f(\bm{x}_t) - \frac{\eta_{t-1}}{\eta_t}\bm{e}_{t-1}\,.
\end{equation}
(We define $0/0\eqDef0$ for the very first iteration when $\eta_{-1}=0$.)
We then have
\begin{align}
\hat{\bm{x}}_{t+1} &= \hat{\bm{x}}_t - \eta_t\bm{q}_t \\
 &= \hat{\bm{x}}_t - \eta_t(\bm{u}_t+\bm{e}_t) \\
 &= \hat{\bm{x}}_t - \eta_t\lp\nabla f(\bm{x}_t)-\frac{\eta_{t-1}}{\eta_t}\bm{e}_{t-1}\rp - \eta_t\bm{e}_t \\
 &= \big[ \bm{x}_t - \eta_t\nabla f(\bm{x}_t)\big] - \eta_t\bm{e}_t \label{eq:recur_memory_hypo}\\
 &= \bm{x}_{t+1} - \eta_t\bm{e}_t,
\end{align}
where \eqref{eq:recur_memory_hypo} is due to the induction hypothesis.
\end{itemize}
\end{proof}
Relation \eqref{eq:recur_memory_full} implies for the constant stepsize $\eta_t \equiv \eta$
\begin{align}
\lnorm \hat{\bm{x}}_t-\bm{x}_\msf{f}^*\rnorm &\leq \lnorm \bm{x}_t-\bm{x}_\msf{f}^*\rnorm + \eta\lnorm \bm{e}_{t-1}\rnorm \label{eq:pf_thm_upper_b}
\end{align}

We use the contraction factor of unquantized GD to bound the first term of \eqref{eq:pf_thm_upper_b}:
\begin{lemma}[{Convergence of GD \cite[Theorem~2.1.15]{Nesterov-14}}]
\label{lma:conv-GD}
For any $L$-smooth and $\mu$-strongly convex function $\fcn{f}$ on $\R^n$, GD \eqref{eq:GD} with stepsize \eqref{eq:stepsize_gd} satisfies
\begin{equation}
\label{eq:conv_GD}
\norm{\vec{x}_t-\vec{x}_{\fcn{f}}^*} \leq \sigma_{\mathrm{GD}}^t\norm{\vec{x}_0-\vec{x}_{\fcn{f}}^*},
\end{equation}
where $\sigma_{\mathrm{GD}}$ is defined in \eqref{eq:sigmaDQ}.
\end{lemma}

We use the following bound on quantization error to bound the second term in~\eqref{eq:pf_thm_upper_b}:
\begin{lemma}[DQ-GD quantization error]
\label{lma:XOD_memory}
Let $\msf{f}\in\set{F}_n$. Quantization errors $\{\bm{e}_t\}$ in Algorithm~\ref{alg:DQGD_short} with stepsize  \eqref{eq:stepsize_gd} and dynamic ranges \eqref{eq:DR_memory} satisfy
\begin{align}
\label{eq:egd}
\lnorm \bm{e}_t \rnorm &\leq  r_t \, \rho_n 2^{-R} \\
&\leq   \max\lbp \sigma_{\mathrm{GD}},\, \rho_n 2^{-R}\rbp^t LD\, \sigma_{\mathrm{GD}} b_t, 
\label{eq:egdsigma}
\end{align}
where
\begin{align}
b_t \triangleq 
 \begin{cases}
\frac{\rho_n 2^{-R}} {\left|  \sigma_{\mathrm{GD}}  -\rho_n 2^{-R}\right|}& \sigma_{\mathrm{GD}}\neq \rho_n 2^{-R} \\
t+1 & \sigma_{\mathrm{GD}}= \rho_n 2^{-R}
\end{cases}
\end{align}
\end{lemma}
\begin{proof}%[Proof of Lemma~\ref{lma:XOD_memory}]
Once we show that quantizer inputs $\{\bm{u}_t\}$ satisfy
\begin{equation}
\label{eq:XOD}
\bm{u}_t \in \mcal{B}(r_t),
\end{equation}
\eqref{eq:egd} will follow from \eqref{eq:e_d}.  

We prove \eqref{eq:XOD} via induction.
\begin{itemize}
\item Base case: \eqref{eq:XOD} holds for $t=0$ since
\begin{align}
\lnorm \nabla\msf{f}(\hat{\bm{x}}_0)-\bm{e}_{-1}\rnorm &= \lnorm \nabla\msf{f}(\hat{\bm{x}}_0)\rnorm \\
 &\leq L\lnorm \hat{\bm{x}}_0-\bm{x}_{\msf{f}}^*\rnorm \label{eq:pf_XOD_memory_grad} \\
 &\leq LD, \label{eq:pf_XOD_memory_init}
\end{align}
where \eqref{eq:pf_XOD_memory_grad} is due to $\grad\fcn{f}(\vec{x}_{\fcn{f}}^*)=\vec{0}$ and $L$-smoothness \eqref{eq:smooth}, and \eqref{eq:pf_XOD_memory_init} is due to assumption \eqref{eq:iter_opt_range}.
\item Inductive step: Suppose \eqref{eq:XOD} holds for iteration $t$. Applying triangle inequality and \eqref{eq:recur_memory_full} to the expression in Line~5 yields
\begin{equation}
\lnorm \bm{u}_{t+1}\rnorm \leq \lnorm \nabla\msf{f}({\bm{x}}_{t+1})\rnorm + \lnorm \bm{e}_{t}\rnorm. \label{eq:pf_XOD_memory_triangle}
\end{equation}
The first term is bounded as
\begin{align}
\lnorm \nabla\msf{f}({\bm{x}}_{t+1})\rnorm &\leq L\lnorm {\bm{x}}_{t+1}-\bm{x}_{\msf{f}}^*\rnorm \label{eqn:pf-range-GD-smooth}\\
 &\leq \sigma_{\mathrm{GD}}^{t+1} L\lnorm {\bm{x}}_0-\bm{x}_{\msf{f}}^*\rnorm \label{eq:pf_XOD_memory_GD} \\
 &\leq \sigma_{\mathrm{GD}}^{t+1} L D, \label{eq:pf_XOD_memory_range}
\end{align}
where \eqref{eqn:pf-range-GD-smooth} is due to $\grad\fcn{f}(\vec{x}_{\fcn{f}}^*)=\vec{0}$ and $L$-smoothness \eqref{eq:smooth}; \eqref{eq:pf_XOD_memory_GD} is due to \eqref{eq:conv_GD}; and \eqref{eq:pf_XOD_memory_range} is due to assumption \eqref{eq:iter_opt_range}.
Quantization error term $\lnorm \bm{e}_{t}\rnorm$ in \eqref{eq:pf_XOD_memory_triangle} is bounded by \eqref{eq:e_d} due to the induction hypothesis. Plugging \eqref{eq:pf_XOD_memory_range} and \eqref{eq:e_d} into \eqref{eq:pf_XOD_memory_triangle} gives
\begin{align}
\lnorm \bm{u}_{t+1}\rnorm &\leq \sigma_{\mathrm{GD}}^{t+1}  LD + r_t \rho_n 2^{-R}  \\
 &= r_{t+1}, \label{eq:DR_memory_u}
\end{align}
where \eqref{eq:DR_memory_u} is due to the choice of the dynamic ranges \eqref{eq:DR_memory}.
\end{itemize}
This concludes the proof of \eqref{eq:egd}. To establish \eqref{eq:egdsigma}, we unwrap recursion \eqref{eq:DR_memory} as the geometric sum 
\begin{align}
r_t &= LD \sum_{\tau=0}^t\sigma_{\mathrm{GD}}^{\tau}\lp \rho_n 2^{-R}\rp^{t-\tau} \\
&= LD \cdot 
\begin{cases}
 \sigma_{\mathrm{GD}}^t \frac{1-\left[ \rho_n 2^{-R} /\sigma_{\mathrm{GD}}\right]^{t+1}}{1- \rho_n 2^{-R} /\sigma_{\mathrm{GD}}} & \sigma_{\mathrm{GD}}\neq \rho_n 2^{-R} \\
 t+1 & \sigma_{\mathrm{GD}}= \rho_n 2^{-R} 
\end{cases} \\
 \label{eq:DR_memoryalt}
 &\leq LD \cdot 
 \begin{cases}
 \sigma_{\mathrm{GD}}^t \left(1-\frac{\rho_n 2^{-R}}{\sigma_{\mathrm{GD}}}\right)^{-1}  & \sigma_{\mathrm{GD}}> \rho_n 2^{-R} \\
  \left(\rho_n 2^{-R}\right)^t \left(\frac{\rho_n 2^{-R}}{\sigma_{\mathrm{GD}}} - 1\right)^{-1}   & \sigma_{\mathrm{GD}}< \rho_n 2^{-R} \\
 t+1 & \sigma_{\mathrm{GD}}= \rho_n 2^{-R} 
\end{cases},
 \end{align}
 where the bound for the case $\sigma_{\mathrm{GD}}> \rho_n 2^{-R}$ is obtained by lower-bounding $\left[ \rho_n 2^{-R} /\sigma_{\mathrm{GD}}\right]^{t+1}$ by 0, and the bound for $\sigma_{\mathrm{GD}}< \rho_n 2^{-R}$ is obtained by lower-bounding $\left[ \sigma_{\mathrm{GD}} / \rho_n 2^{-R} \right]^{t+1}$  by 0. 
\end{proof}

Putting together the results in Lemmas~\ref{lma:recur_memory_full},~\ref{lma:conv-GD}, and~\ref{lma:XOD_memory}, we show the following nonasymptotic (in the iteration number) convergence result for the DQ-GD:

\begin{theorem}[Convergence of DQ-GD]
\label{thm:upper}
In the setting of \thmref{thm:DQ-GD}, the difference between the iterate and the optimizer at step $t$ satisfies
\begin{align}
\label{eq:conv_memory_1st_case}
\!\!\!\! \lnorm \hat{\bm{x}}_t-\bm{x}_\msf{f}^*\rnorm \leq \max\lbp \sigma_{\mathrm{GD}},\, \rho_n 2^{-R}\rbp^t \lb 1+\eta_{\mathrm{GD}} L b_{t-1}\rb D.
\end{align}
\end{theorem}

\begin{proof}
Plugging \eqref{eq:conv_GD} and \eqref{eq:egdsigma}  into \eqref{eq:pf_thm_upper_b}, we obtain
\begin{align}
\lnorm \hat{\bm{x}}_t-\bm{x}_\msf{f}^*\rnorm \leq &~ \sigma_{\mathrm{GD}}^t D \\
&~ + \max\lbp \sigma_{\mathrm{GD}}, \rho_n 2^{-R}\rbp^{t-1} \eta_{\mathrm{GD}} LD\, \sigma_{\mathrm{GD}} b_{t-1}, \notag
\end{align}
which leads to \eqref{eq:conv_memory_1st_case} by an elementary weakening. 
\end{proof}

Bound \eqref{eq:sigmaDQ-GD} in~\thmref{thm:DQ-GD} follows immediately by applying \eqref{eq:conv_memory_1st_case} to definition \eqref{eq:lcr} of the contraction factor.

\subsection{Proof of Theorem~\ref{thm:DQ-AGD}}
\label{sec:DQ-AGD}
The proof follows steps similar to those in the proof of \thmref{thm:DQ-GD}. First, we prove that, as prescribed by the principle of differential quantization, DQ-AGD compensates quantization errors by directing the quantized trajectory back to the trajectory of AGD: 
\begin{lemma}[DQ-AGD trajectory]
\label{lma:track-AGD}
Iterate sequences $\bp{\hat{\vec{y}}_t, \hat{\vec{x}}_t}$ of Algorithm~\ref{alg:DQAGD} and $\bp{{\vec{y}}_t, {\vec{x}}_t}$ of AGD \eqref{eqn:descent-AGD-y} starting at the same location $(\hat{\vec{y}}_0, \hat{\vec{x}}_0) = ({\vec{y}}_0, {\vec{x}}_0)$  are related as, $\forall t = 0, 1, 2, \ldots$,
\begin{align}
\label{eqn:track-AGD-y}
\hat{\vec{y}}_t &= \vec{y}_t - \eta\vec{e}_{t-1}\\
\label{eqn:track-AGD-x}
\hat{\vec{x}}_t &= \vec{x}_t - \eta \vec{e}_{t-1} - \eta \gamma \left( \vec{e}_{t-1} - \vec{e}_{t-2} \right).
\end{align}
\end{lemma}

\begin{proof}
We prove \eqref{eqn:track-AGD-y} and \eqref{eqn:track-AGD-x} by induction.
\begin{itemize}
\item Base case: \eqref{eqn:track-AGD-y} and \eqref{eqn:track-AGD-x} hold for $t=0$ since both algorithms start at the same location.
\item Inductive step: Line~9 yields 
\begin{align}
 %&
 \hat{\vec{y}}_{t+1} %\notag\\
=\> 
&
\hat{\vec{x}}_t - \eta\vec{q}_t \\
=\> &\vec{x}_t - \eta \left[\vec{e}_{t-1} + \gamma \left( \vec{e}_{t-1} - \vec{e}_{t-2} \right) \right]\label{eqn:pf-track-AGD-y-hypo}  \\
 &- \eta\bk{\grad\fcn{f}(\vec{x}_t)-\left[ \vec{e}_{t-1} + \gamma \left( \vec{e}_{t-1} - \vec{e}_{t-2} \right)\right]+\vec{e}_t} \notag\\
=\> &\vec{y}_{t+1} - \eta\vec{e}_t, \label{eqn:pf-track-AGD-y-rule}
\end{align}
where \eqref{eqn:pf-track-AGD-y-hypo} is due to induction hypothesis \eqref{eqn:track-AGD-x} and Lines~5 and 7, and \eqref{eqn:pf-track-AGD-y-rule} is due to \eqref{eqn:descent-AGD-y}. On the other hand, plugging \eqref{eqn:track-AGD-y} and \eqref{eqn:pf-track-AGD-y-rule} into Line~10 yields
\begin{align}
 %&
 \hat{\vec{x}}_{t+1} %\\
=\> &\hat{\vec{y}}_{t+1} + \gamma \left( \hat{\vec{y}}_{t+1} - \hat{\vec{y}}_t \right) \\
=\> & \vec{y}_{t+1}+  \gamma \left( {\vec{y}}_{t+1} - {\vec{y}}_t \right)  - \eta \left[ \vec{e}_t + \gamma \left( \vec{e}_t - \vec{e}_{t-1}\right)\right] \notag\\
=\> &\vec{x}_{t+1} - \eta \left[ \vec{e}_t + \gamma \left( \vec{e}_t - \vec{e}_{t-1}\right)\right] \label{eq:agdindx},
\end{align}
where \eqref{eq:agdindx} is due to \eqref{eqn:descent-AGD-x}.
\end{itemize}
\end{proof}

Via triangle inequality, relation \eqref{eqn:track-AGD-y} implies 
\begin{equation}
\label{eqn:pf-DQ-AGD-triangle}
\norm{\hat{\vec{y}}_t-\vec{x}_{\fcn{f}}^*} \leq \norm{\vec{y}_t-\vec{x}_{\fcn{f}}^*} + \eta\norm{\vec{e}_{t-1}} 
\end{equation}

The following simple corollary to a known bound on the contraction factor of unquantized AGD controls the first term of \eqref{eqn:pf-DQ-AGD-triangle}:

\begin{lemma}[Convergence of AGD]
\label{lma:conv-AGD}
For any $L$-smooth and $\mu$-strongly convex function $\fcn{f}$ on $\R^n$, AGD (\eqref{eqn:descent-AGD-y}, \eqref{eqn:descent-AGD-x} starting at $\vec x_0 = \vec y_0$) with stepsize \eqref{eq:stepsize_agd}
and momentum coefficient \eqref{eq:gamma_agd}
satisfies
\begin{align}
\norm{\vec{y}_t-\vec{x}_{\fcn{f}}^*} &\leq \sigma_{\mathrm{AGD}}^t  \sqrt{\kappa+1} \norm{\vec{x}_0-\vec{x}_{\fcn{f}}^*} \label{eqn:conv-AGD-y} \\
\norm{\vec{x}_t-\vec{x}_{\fcn{f}}^*} &\leq \sigma_{\mathrm{AGD}}^t \lambda \norm{\vec{x}_0-\vec{x}_{\fcn{f}}^*} 
\label{eqn:conv-AGD-x}.
\end{align}
\end{lemma}
\begin{proof}
 According to \cite[Theorem~3.18]{Bubeck-15},  
 \begin{equation}
\label{eqn:conv-AGD-f}
\fcn{f}(\vec{y}_t) - \fcn{f}(\vec{x}_{\fcn{f}}^*) \leq \frac{L+\mu}{2}\sigma_{\textrm{AGD}}^{2t}\norm{\vec{x}_0-\vec{x}_{\fcn{f}}^*}^2.
\end{equation}
Convergence bound \eqref{eqn:conv-AGD-y} w.r.t. $\bp{\vec{y}_t}$ is due to \eqref{eqn:conv-AGD-f} and
\begin{equation}
\fcn{f}(\vec{x})-\fcn{f}(\vec{x}_{\fcn{f}}^*) \geq \frac{\mu}{2}\norm{\vec{x}-\vec{x}_{\fcn{f}}^*}^2 %\quad\forall \vec{x}\in\ell_2
\end{equation}
implied by $\mu$-strong convexity \cite[Theorem~2.1.8]{Nesterov-14}.
On the other hand, applying triangle inequality to \eqref{eqn:descent-AGD-x} %along $\bp{\vec{x}_t}$ 
yields
\begin{equation}
\norm{\vec{x}_{t+1}-\vec{x}_{\fcn{f}}^*} \leq (1+\gamma)\norm{\vec{y}_{t+1}-\vec{x}_{\fcn{f}}^*} + \gamma\norm{\vec{y}_t-\vec{x}_{\fcn{f}}^*},
\label{eq:AGD-aux}
\end{equation}
and \eqref{eqn:conv-AGD-x} then follows from \eqref{eqn:conv-AGD-y} and \eqref{eq:AGD-aux}.
\end{proof}

The following bound on the quantization error controls the second term in~\eqref{eq:pf_thm_upper_b}:
\begin{lemma}[DQ-AGD quantization error]
\label{lem:err-bound-AGD}
Let $\msf{f}\in\set{F}_n$. Quantization errors $\{\bm{e}_t\}$ in Algorithm~\ref{alg:DQAGD} with stepsize \eqref{eq:stepsize_agd}, momentum coefficient \eqref{eq:gamma_agd} and dynamic ranges \eqref{eqn:range-AGD} satisfy
\begin{align}
 \norm{\vec{e}_t} &\leq r_t \rho_n 2^{-R}  \label{eqn:err-bound-AGDrec} \\
 & \leq  \sigma_{\mathrm{AGD}}^t c_0 + \phi^t_+ c_+ + \phi_-^t c_-, \label{eqn:err-bound-AGD}
\end{align}
where $\phi_{\pm} \triangleq \phi_{\pm} (\gamma_{\mathrm{AGD}})$ with
\begin{align}
\!\!\!\! \phi_{\pm}(\gamma) \triangleq \rho_n 2^{-R}\Bigg( \frac 1 2 ( 1 + \gamma)
\pm \frac 1 2 \sqrt{(1+\gamma)^2+\frac{4 \gamma}{\rho_n 2^{-R}}} \Bigg)\!\!
 \label{eqn:recur-root}
\end{align}
 and $c_0, c_+, c_-$ are specified below in  \eqref{eq:c0}, \eqref{eq:c1} and \eqref{eq:c2} respectively.% ($c_0$ is a constant and $c_+, c_-$ are functions of $n, R$). 
\end{lemma}

\begin{proof}
Like in the proof of \lemref{lma:XOD_memory}, to establish \eqref{eqn:err-bound-AGDrec}, in view of \eqref{eq:e_d} it is enough to show \eqref{eq:XOD}. Applying triangle inequality and \eqref{eqn:track-AGD-x} to the expression in Line~5 yields
\begin{equation}
\label{eqn:pf-range-AGD-triangle}
\norm{\vec{u}_t} \leq \norm{\grad\fcn{f}(\vec{x}_t)} + \norm{\vec{e}_{t-1}} + \gamma \left( \norm{\vec{e}_{t-1}} + \norm{\vec{e}_{t-2}} \right).
\end{equation}
The first term in \eqref{eqn:pf-range-AGD-triangle} is bounded as
\begin{align}
 \norm{\grad\fcn{f}(\vec{x}_t)}&\leq L\norm{\vec{x}_t-\vec{x}_{\fcn{f}}^*} \label{eqn:pf-range-AGD-smooth} \\
 &\leq \sigma_{\mathrm{AGD}}^t L D \lambda \label{eqn:pf-range-AGD-unquant}
\end{align}
where 
\eqref{eqn:pf-range-AGD-smooth} is due to $L$-smoothness \eqref{eq:smooth}, and \eqref{eqn:pf-range-AGD-unquant} is due to \eqref{eqn:conv-AGD-x}. 
Plugging \eqref{eqn:pf-range-AGD-unquant} into \eqref{eqn:pf-range-AGD-triangle} and applying \eqref{eq:e_d} to bound quantization error terms in \eqref{eqn:pf-range-AGD-triangle}, we conclude via induction (similar to the proof of \lemref{lma:XOD_memory}) that setting the sequence of dynamic ranges recursively as \eqref{eqn:range-AGD} ensures \eqref{eq:XOD}. 
 
To show \eqref{eqn:err-bound-AGD}, we proceed to solve recursion \eqref{eqn:range-AGD}. This step is significantly different from the corresponding step in the proof of \lemref{lma:XOD_memory} since now $r_t$ depends not just on $r_{t-1}$ but also on $r_{t-2}$. More precisely, recursion \eqref{eqn:range-AGD} is a second-order linear non-homogeneous recurrence relation. 
\begin{itemize}
\item Particular solution: Plugging the candidate
\begin{equation}
p_t = \sigma_{\mathrm{AGD}}^t c_0
\end{equation}
into \eqref{eqn:range-AGD}, we solve for the constant
\begin{equation}
c_0 =  \frac{ \sigma_{\mathrm{AGD}}^2}{\fcn{p}(\sigma_{\mathrm{AGD}})} LD \lambda, 
\label{eq:c0}
\end{equation}
where $\fcn{p}(r)$ is the characteristic polynomial in \eqref{eqn:char} associated with recursion \eqref{eqn:range-AGD}.
\item Homogeneous solution: Since $\phi_+$ and $\phi_-$
are roots of the quadratic polynomial in \eqref{eqn:char}, the homogeneous solution is given by
\begin{equation}
\label{eqn:recur-homo}
\phi_t = \phi^t_+ c_+ + \phi_-^t c_-,
\end{equation}
\item General solution: constants $c_+,c_-$ in \eqref{eqn:recur-homo} are determined by plugging initial conditions $r_{-2} = r_{-1} = 0$ into the general solution
\begin{align}
r_t &= p_t + \phi_t \\
 &= \sigma_{\mathrm{AGD}}^t c_0 + \phi^t_+ c_+ + \phi_-^t c_-,
\end{align}
which are
\begin{align}
c_+ &= - \frac{c_0 \phi_+^2 }{\sigma_{\mathrm{AGD}}^2}  \frac{\sigma_{\mathrm{AGD}} - \phi_-}{\phi_+ - \phi_-} \label{eq:c1} \\
c_- &=  \frac{c_0 \phi_-^2 }{\sigma_{\mathrm{AGD}}^2}  \frac{\sigma_{\mathrm{AGD}} - \phi_+}{\phi_+ - \phi_-} \label{eq:c2}.
\end{align}
\end{itemize}
\end{proof}

Lemmas~\ref{lma:track-AGD},~\ref{lma:conv-AGD}, and~\ref{lem:err-bound-AGD} lead to the following finite-$t$ convergence bound for the DQ-AGD:

\begin{theorem}[Convergence of DQ-AGD]
\label{thm:DQ-AGDfinitet}
In the setting of \thmref{thm:DQ-AGD}, the difference between the iterate and the optimizer at step $t$ satisfies
\begin{equation}
\label{eqn:path-DQ-AGD}
\norm{\hat{\vec{y}}_t-\vec{x}_{\fcn{f}}^*} \leq \sigma_{\mathrm{AGD}}^t c + \phi^{t-1}_+ c_+ \eta + \phi_-^{t-1} c_- \eta
\end{equation}
where $c =  \sqrt{\kappa+1} D + \eta c_0 \sigma_{\mathrm{AGD}}^{-1}$, and $c_0, c_+, c_-$ are specified in  \eqref{eq:c0}, \eqref{eq:c1} and \eqref{eq:c2} respectively.
\end{theorem}

\begin{proof}
Plugging \eqref{eqn:conv-AGD-y} and \eqref{eqn:err-bound-AGD} into \eqref{eqn:pf-DQ-AGD-triangle} immediately leads to \eqref{eqn:path-DQ-AGD}.
\end{proof}

Note that if $\sigma_{\mathrm{AGD}} \geq \phi_+$, which is equivalent to $R \geq R_2$ \eqref{eqn:suffi-same-as-unqtz}, then $c_0 \geq 0$, $c_+ \leq 0$, $c_- \geq 0$; and $c_0 \leq  0$, $c_+ \geq 0$, $c_- \leq 0$ otherwise.  
Asymptotic convergence guarantee \eqref{eq:sigmaDQ-AGD} in~\thmref{thm:DQ-AGD} follows by plugging \eqref{eqn:path-DQ-AGD} into definition  \eqref{eq:lcr} of the contraction factor.

\subsection{Proof of \thmref{thm:DQ-HB}}
\label{sec:DQ-HB}
The proof of the DQ-HB convergence result in \thmref{thm:DQ-HB} follows the same recipe as the proofs of Theorems~\ref{thm:DQ-GD} and~\ref{thm:DQ-AGD}. 
Lemma~\ref{lma:track-HB} below states that the path of DQ-HB tracks that of the unquantized HB.

\begin{lemma}[DQ-HB trajectory]
\label{lma:track-HB}
Path $\{\hat{\vec{x}}_t\}$ of DQ-HB (Line~7) and path $\{\vec{x}_t\}$ of HB \eqref{eqn:descent-HB} starting at the same location $\hat{\vec{x}}_{-1} = {\vec{x}}_{-1} = \hat{\vec{x}}_{0} = \vec x_0 $ are related as, $\forall t = 0, 1, 2, \ldots$,
\begin{equation}
\label{eqn:track-HB}
\hat{\vec{x}}_t = \vec{x}_t - \eta\vec{e}_{t-1}.
\end{equation}
\end{lemma}
\begin{proof}
We prove \eqref{eqn:track-HB} via induction.
\begin{itemize}
\item Base case: \eqref{eqn:track-HB} holds for $t=0$ by the initialization $\vec{e}_{-2} = \vec{e}_{-1} = \vec{0}$ in Line 1 of Algorithm~\ref{alg:DQHB} and since the starting location is the same.
\item Inductive step: Plugging expressions on Line~5 and 7 into Line~8 yields
\begin{align}
 &\hat{\vec{x}}_{t+1} \\
=\> &\hat{\vec{x}}_t-\eta\vec{q}_t+ \gamma\left( \hat{\vec{x}}_{t} - \hat{\vec{x}}_{t-1}\right) \\
 =\> & \vec{x}_t - \eta \vec e_{t-1} \notag\\
 &- \eta \left[ \grad\fcn{f}(\vec{x}_t) - \left[ \vec{e}_{t-1} + \gamma \left( \vec{e}_{t-1}- \vec{e}_{t-2} \right) \right] + \vec e_t \right] \notag \\
 &+ \gamma \left( \vec x_t - \vec x_{t-1} - \eta \left( \vec e_{t-1} - \vec e_{t-2} \right) \right) \label{eqn:pf-track-HB-IH}  \\
=\>& \vec x_t - \eta \grad\fcn{f}(\vec{x}_t) + \gamma \left( \vec x_t - \vec x_{t-1} \right)  - \eta\vec{e}_t \\
=\> &\vec{x}_{t+1} - \eta\vec{e}_t, \label{eqn:pf-track-HB-rule}
\end{align}
where \eqref{eqn:pf-track-HB-IH} is due to the induction hypothesis and \eqref{eqn:pf-track-HB-rule} is due to \eqref{eqn:descent-HB}.
\end{itemize} 
\end{proof}

Applying triangle inequality to \eqref{eqn:track-HB} gives
\begin{equation}
\label{eqn:pf-DQHB-triangle}
\norm{\hat{\vec{x}}_t-\vec{x}_{\fcn{f}}^*} \leq \norm{\vec{x}_t-\vec{x}_{\fcn{f}}^*} + \eta\norm{\vec{e}_{t-1}}.
\end{equation}
The convergence result for unquantized HB in \lemref{lma:HB}, below,  controls the first term in the right side of \eqref{eqn:pf-DQHB-triangle}. \lemref{lma:HB} is a nonasymptotic refinement of Polyak's original convergence result  \cite[Th. 1, Sec. 3.2]{Polyak-87}. Unlike the original, it does not require that the algorithm starts ``sufficiently close" to the minimizer (i.e., it establishes global rather than local convergence), and it also clarifies that the subexponential factor in the bound is polynomial in $t$ (see \eqref{eqn:conv-HB}, below). This refinement is made possible by a result on joint spectral radius due to Wirth \cite{wirth1998calculation} that is more recent than \cite[Th. 1, Sec. 3.2]{Polyak-87}. 
\begin{lemma}[Convergence of HB]
\label{lma:HB}
For any $L$-smooth, $\mu$-strongly convex, twice continuously differentiable function $\fcn{f}$ on $\R^n$,  there exists a constant $\alpha > 0$ such that the HB algorithm \eqref{eqn:descent-HB} starting at $\vec x_{-1} = \vec x_0$ with stepsize \eqref{eq:stepsize_hb}
and momentum coefficient \eqref{eq:gamma_hb}
satisfies
\begin{equation}
\label{eqn:conv-HB}
\norm{\vec{x}_t-\vec{x}_{\fcn{f}}^*} \leq \sigma_{\mathrm{HB}}^t\, t^\alpha e^\alpha \sqrt 2 \| \vec x_0 - \vec x_{\mathsf f}^*\|.
\end{equation}
\end{lemma}
\begin{proof}
Iterative process \eqref{eqn:descent-HB} can be written in the form
\begin{align}
&~
\begin{bmatrix}
\vec x_{t+1} - \vec x_{\mathsf f}^* \\
\vec x_{t} - \vec x_{\mathsf f}^*
\end{bmatrix} \notag\\
= &~
\begin{bmatrix}
\vec x_{t} + \gamma(\vec x_t - \vec x_{t-1}) - \vec x_{\mathsf f}^* \\
\vec x_{t} - \vec x_{\mathsf f}^*
\end{bmatrix} 
- \eta
\begin{bmatrix}
\grad \mathsf f(\vec x_t) \\
0
\end{bmatrix} \\
= &~ 
\begin{bmatrix}
(1 + \gamma) \vec I & - \gamma \vec I \\
\vec I & 0
\end{bmatrix} 
\begin{bmatrix}
\vec x_{t} - \vec x_{\mathsf f}^* \\
\vec x_{t-1} - \vec x_{\mathsf f}^*
\end{bmatrix} 
- \eta
\begin{bmatrix}
\grad^2 \mathsf f(\vec v_t)  (\vec x_{t} - \vec x_{\mathsf f}^*) \\
0
\end{bmatrix}  
\label{eq:twiceder}\\
= &~ 
\begin{bmatrix}
(1 + \gamma) \vec I & - \gamma \vec I - \eta \grad^2 \mathsf f(\vec v_t)\\
\vec I & 0
\end{bmatrix} 
\begin{bmatrix}
\vec x_{t} - \vec x_{\mathsf f}^* \\
\vec x_{t-1} - \vec x_{\mathsf f}^*
\end{bmatrix},\label{eq:Amat}
\end{align}
where \eqref{eq:twiceder} holds for some $\vec v_t$ on the line segment between $\vec x_t$ and $\vec x_{\mathsf f}^*$ by the mean value theorem, since $\mathsf f$ is twice continuously differentiable by the assumption.
Denoting the matrix in \eqref{eq:Amat} by $\vec A_t$, we unroll the recursion as
\begin{align}
\begin{bmatrix}
\vec x_{t+1} - \vec x_{\mathsf f}^* \\
\vec x_{t} - \vec x_{\mathsf f}^*
\end{bmatrix} 
=
\vec A_t \cdot \vec A_{t-1} \cdot \ldots \cdot \vec A_0 
\begin{bmatrix}
\vec x_{0} - \vec x_{\mathsf f}^* \\
\vec x_{-1} - \vec x_{\mathsf f}^*
\end{bmatrix}.
\end{align}
It follows that
\begin{align}
\!\!\!\! \left\| 
\begin{bmatrix}
\vec x_{t+1} - \vec x_{\mathsf f}^* \\
\vec x_{t} - \vec x_{\mathsf f}^*
\end{bmatrix} 
\right\|_2
\leq
\left\| \vec A_t \cdot \ldots \cdot \vec A_0 \right\|_2
\left\|
\begin{bmatrix}
\vec x_{0} - \vec x_{\mathsf f}^* \\
\vec x_{-1} - \vec x_{\mathsf f}^*
\end{bmatrix}
\right\|_2. \label{eq:contrt}
\end{align}
It is shown in \cite[Lemma 2.3]{wirth1998calculation} that if matrices $\vec A_1, \ldots, \vec A_t$ all belong to a bounded set $\mathcal A$,  then there exists a constant $\alpha > 0$ such that $\forall t = 0, 1, \ldots$,
\begin{align}
\left\| \vec A_t \cdot \ldots \cdot \vec A_0 \right\|_2 &\leq \rho(\mathcal A)^{t+1} (t+1)^\alpha  e^\alpha \label{eq:rholimup}
\end{align}
where 
\begin{align}
\rho(\mathcal A) \triangleq \limsup_{t \to \infty} \sup_{t} \rho(\vec A_t),
\end{align}
where $\rho(\vec A_t)$ is the spectral radius of $\vec A_t$.
It is shown in \cite[Proof of Th. 1, Sec. 3.2]{Polyak-87} that if 
\begin{align}
\gamma = \max \left\{ \left( 1 - \sqrt{\eta L}\right)^2,~ \left( 1 - \sqrt{\eta \mu}\right)^2 \right\},  
\end{align}
then 
\begin{align}
\rho(\vec A_t) \leq \sqrt \gamma.
\label{eq:spectralub}
\end{align}
With the optimal choice of $\eta$ \eqref{eq:stepsize_hb} the right side of \eqref{eq:spectralub} is equal to  $\sigma_{\mathrm{HB}}$ \eqref{eqn:conv-lin-HB}.

In our setting $\mathcal A$ is bounded since for twice continuously differentiable functions, $L$-smoothness and $\mu$-strong convexity are equivalent to
\begin{equation}
\mu \vec I \preceq \grad^2 \mathsf f(\vec v) \preceq L \vec I,
\end{equation}
 in the positive semidefinite order, thus \eqref{eq:rholimup} applies.
Inequality \eqref{eqn:conv-HB} follows after applying \eqref{eq:spectralub} to \eqref{eq:rholimup} and the latter to~\eqref{eq:contrt}.
\end{proof}

\hl{
\begin{remark}
Polyak's convergence result \cite[Th. 1, Sec. 3.2]{Polyak-87} guarantees only the existence of $D > 0$ such that for all starting points $\vec x_{-1}, \vec x_0$ with $\lnorm \bm{x}_{\msf{f}}^*  - {\vec{x}}_{-1} \rnorm\leq D$, $\lnorm \bm{x}_{\msf{f}}^*  - {\vec{x}}_{0} \rnorm\leq D$ and all $0< \epsilon < 1 - \sigma_{\mathrm{HB}}$, there exists a $c > 0$ such that (cf. \eqref{eqn:conv-HB})
\begin{equation}
\label{eqn:conv-HBpolyak}
\norm{\vec{x}_t-\vec{x}_{\fcn{f}}^*} \leq c (\sigma_{\mathrm{HB}} + \epsilon)^t
\end{equation}
 under the same assumptions on $\mathsf f$ and the same stepsize and momentum coefficient as in \lemref{lma:HB}. This is a local convergence result because convergence is not guaranteed for any starting location but only for locations in a small enough neighbourhood of $\vec{x}_{\fcn{f}}^*$. Furthermore, \eqref{eqn:conv-HB} refines the subexponential factor in \eqref{eqn:conv-HBpolyak}.  
\end{remark}
}
We ensure that the quantization error term $\norm{\vec{e}_{t-1}}$ in \eqref{eqn:pf-DQHB-triangle} decays exponentially fast by adjusting the sequence of dynamic ranges \eqref{eqn:quant-scale} carefully:

\begin{lemma}[DQ-HB quantization error]
\label{lma:range-HB}
Let $\fcn{f}\in\set{F}_n^2$. Quantization errors $\{\bm{e}_t\}$ in Algorithm~\ref{alg:DQHB} with with stepsize \eqref{eq:stepsize_hb}, momentum coefficient \eqref{eq:gamma_hb} and dynamic ranges \eqref{eqn:range-HB} satisfy
\begin{align}
\norm{\vec{e}_t}  &\leq r_t \rho_n 2^{-R}  \label{eqn:err-bound-HBrec} \\ 
\ &\leq \left(\sigma_{\mathrm{HB}}^t c_0 + \phi^t_+ c_+ + \phi_-^t c_-\right)t^\alpha, \label{eqn:err-bound-HB}
\end{align}
where $\phi_{\pm} \triangleq \phi_{\pm} (\gamma_{\mathrm{HB}})$ \eqref{eqn:recur-root}, and $c_0, c_+, c_-$ are as in  \eqref{eq:c0}, \eqref{eq:c1} and \eqref{eq:c2} respectively, with $L D \lambda$ replaced by $ e^\alpha \sqrt 2\, L D$, and $\sigma_{\mathrm{AGD}}$ by $\sigma_{\mathrm{HB}}$.
\end{lemma}

\begin{proof}
Applying triangle inequality and \eqref{eqn:track-HB} to the expression in Line~5 yields the same expression as in the analysis of DQ-AGD \eqref{eqn:pf-range-AGD-triangle}: 
\begin{equation}
\label{eqn:pf-range-HB-triangle}
\norm{\vec{u}_t} \leq \norm{\grad\fcn{f}(\vec{x}_t)} + \norm{\vec{e}_{t-1}} + \gamma \left( \norm{\vec{e}_{t-1}} + \norm{\vec{e}_{t-2}} \right).
\end{equation}
The first term in \eqref{eqn:pf-range-HB-triangle} is bounded as
\begin{align}
\norm{\grad\fcn{f}(\vec{x}_t)} &\leq L\norm{\vec{x}_t-\vec{x}_{\fcn{f}}^*} \label{eqn:pf-range-HB-smooth} \\
 &\leq \sigma_{\mathrm{HB}}^t\, t^\alpha e^\alpha \sqrt 2\, L D  \label{eqn:pf-range-HB-unquant}
\end{align}
where \eqref{eqn:pf-range-AGD-smooth} is due to $L$-smoothness \eqref{eq:smooth}, and \eqref{eqn:pf-range-HB-unquant} is due to \eqref{eqn:conv-HB}. Applying the argument used to show \eqref{eqn:err-bound-AGDrec} in the proof of \lemref{lem:err-bound-AGD} leads to \eqref{eqn:err-bound-HBrec}. 

To show \eqref{eqn:err-bound-HB},  consider the recursion $r_{-1}^\prime = r_{-2}^\prime = 0$, 
\begin{align}
\label{eqn:range-HBprime}
r_t^\prime = \sigma_{\mathrm{HB}}^t\,  e^\alpha \sqrt 2\, L D +   \left( r_{t-1}^\prime + \gamma_{\mathrm{HB}} (r_{t-1}^\prime + r_{t-2}) \right) \rho_n 2^{-R},
\end{align}
We show by strong induction that 
\begin{align}
r_t \leq t^\alpha r_t^\prime, \label{eq:rtprime}
\end{align}
where $r_t$ solves \eqref{eqn:range-HB}. Base case $r_0 = r_0^\prime$ holds by the initial conditions. Assuming that \eqref{eq:rtprime} holds for $1, \ldots, t-1$, we establish \eqref{eq:rtprime} for $t$ using \eqref{eqn:range-HB} and the fact that $t^\alpha$ is increasing in $t$:
\begin{align}
 \!\!&~r_t \leq  \sigma_{\mathrm{HB}}^t\, t^\alpha e^\alpha \sqrt 2\, L D \\
 \!\!&+ \!  \left( (t-1)^\alpha r_{t-1}^\prime +  \gamma_{\mathrm{HB}} ((t-1)^\alpha r_{t-1}^\prime + (t-2)^\alpha r_{t-2}^\prime) \right) \! \rho_n 2^{-R} \notag\\
 \!\!&\leq t^\alpha r_t^\prime. 
\end{align}
Using \eqref{eqn:err-bound-HBrec} and \eqref{eq:rtprime}, we can show \eqref{eqn:err-bound-HB} by solving the recursion \eqref{eqn:range-HBprime}. But this is the same recursion as in \eqref{eqn:range-AGD}, up to the constants, thus the solution in the proof of \lemref{lem:err-bound-AGD} applies. 
\end{proof}

We now apply Lemmas~\ref{lma:track-HB},~\ref{lma:HB} and \ref{lma:range-HB} to state a finite-$t$ refinement of~\thmref{thm:DQ-HB}.
\begin{theorem}[Convergence of DQ-HB]
\label{thm:DQ-HBfinitet}
In the setting of \thmref{thm:DQ-HB}, the difference between the iterate and the optimizer at step $t$ satisfies
\begin{equation}
\label{eqn:conv-DQ-HB}
\norm{\hat{\vec{x}}_t-\vec{x}_{\fcn{f}}^*} \leq \left(\sigma_{\mathrm{HB}}^t c + \phi^{t-1}_+ c_+ + \phi_-^{t-1} c_-\right)t^\alpha
\end{equation}
where $c =  e^\alpha \sqrt 2 D + \eta c_0 \sigma_{\mathrm{HB}}^{-1}$, and $\phi_+, \phi_-, c_0, c_+, c_-$ are as in  \lemref{lma:range-HB}.
\end{theorem}

\begin{proof}
Plugging \eqref{eqn:conv-HB} and \eqref{eqn:err-bound-HB} into \eqref{eqn:pf-DQHB-triangle} and using $(t-1)^{\alpha} < t^\alpha$ leads to \eqref{eqn:conv-DQ-HB}.
\end{proof}

\section{Converses}
\subsection{Converse for unquantized GD}
\label{apx:cGD}
As explained in the proof sketch, the lower bound in \eqref{eq:converseGD} is a combination of an unquantized GD converse and a volume-division converse. The former relies on the following converse result, obtained by constructing a least-square problem instance that satisfies Nesterov's upper bound in \lemref{lma:conv-GD} with equality.  

\begin{lemma}[Optimality of $\sigma_{\mathrm{GD}}$]
\label{lma:GD_opt}
Consider GD \eqref{eq:GD} with starting point $\bm{x}_0$ and any constant stepsize $\eta$. Then, there exists a problem instance $\msf{f}\in\set{F}_n$ such that the distance to the optimizer at each iteration $t$ of GD satisfies
\begin{equation}
\label{eq:GD_opt}
\lnorm \bm{x}_{t+1}-\bm{x}_{\msf{f}}^*\rnorm \geq \sigma_{\mathrm{GD}}\lnorm \bm{x}_t-\bm{x}_{\msf{f}}^*\rnorm,
\end{equation}
with equality if $\eta$  is the optimal stepsize given in \eqref{eq:stepsize_gd}.
\end{lemma}

\begin{proof}[Proof of {Lemma~\ref{lma:GD_opt}}]
We first find an $\bm{x}_{\msf{f}}^*\in\mcal{B}(D)$ such that
\begin{equation}
\label{eq:iter_init_far}
\lnorm \bm{x}_0-\bm{x}_{\msf{f}}^* \rnorm \geq D
\end{equation}
and then construct a least-squares problem instance $\msf{f}\in\set{F}_n$ \eqref{eq:obj_LS} that admits $\bm{x}_{\msf{f}}^*$ as a unique minimizer and satisfies \eqref{eq:GD_opt}. Note that $\msf{f}$  is
\begin{equation}
\label{eq:smooth_strcvx_LS}
\sigma_1^2(\mbf{A}) \text{-smooth} \Mand \sigma_n^2(\mbf{A}) \text{-strongly convex}
\end{equation}
where we denote by $\sigma_i(\mbf{A})$ the $i$-th largest singular value of matrix $\mbf{A}\in\mbb{R}^{m\by n}$. The gradient of $\msf{f}$ at iteration $t$ is
\begin{equation}
\label{eq:grad_LS}
\nabla\msf{f}(\bm{x}_t) = \mbf{A}^\T\lp \mbf{A}\bm{x}_t-\bm{y}\rp.
\end{equation}
The first-order optimality condition $\nabla\msf{f}(\bm{x}_{\mathsf f}^*) = \bm{0}$
implies
\begin{equation}
\label{eq:grad_opt_LS}
\mbf{A}^\T\bm{y} = \mbf{A}^\T\mbf{A}\bm{x}_{\msf{f}}^*.
\end{equation}
To each $\bm{x}_{\msf{f}}^* \in\mcal{B}(D)$ that satisfies \eqref{eq:iter_init_far}, there corresponds a $\bm{y}\in\mbb{R}^m$ such that \eqref{eq:grad_opt_LS} holds. This is because $m\geq n$, i.e. we have more degrees of freedom than the problem dimension when selecting the vector $\bm{y}$. Since we can always select an $\mbf{A}$ with $\sigma_1(\mbf{A}) = \sqrt L$ and $\sigma_n(\mbf{A}) = \sqrt \mu$ and a $\bm{y}$ to ensure \eqref{eq:grad_opt_LS}, we have $\msf f \in \set F_n$. It remains to show how to set the right singular vectors of $\mbf{A}$ to ensure \eqref{eq:GD_opt}. 

Plugging \eqref{eq:grad_LS} into \eqref{eq:GD} yields
\begin{equation}
\label{eq:recur_LS}
\bm{x}_{t+1} = \bm{x}_t - \eta\mbf{A}^\T\mbf{A}(\bm{x}_t-\bm{x}_{\msf{f}}^*).
\end{equation}
Subtracting $\bm{x}_{\msf{f}}^*$ from both sides of \eqref{eq:recur_LS}, we conclude that the distance to the optimizer $\bm{x}_{\msf{f}}^*$ satisfies
\begin{equation}
\label{eq:conv_LS}
\lnorm \bm{x}_{t+1}-\bm{x}_{\msf{f}}^*\rnorm \leq \sigma_1\lp \mbf{I}-\eta\mbf{A}^\T\mbf{A}\rp \lnorm \bm{x}_t-\bm{x}_{\msf{f}}^*\rnorm,
\end{equation}
where equality is achieved if $\bm{x}_t-\bm{x}_{\msf{f}}^*$ points in the direction corresponding to the largest singular vector of the matrix $\mbf{I}-\eta\mbf{A}^\T\mbf{A}$. 
Since
\begin{align}
 \sigma_1\lp \mbf{I}-\eta \mbf{A}^\T\mbf{A}\rp 
= \max\lbp \lba 1-\eta \sigma_n^2(\mbf{A})\rba,\, \lba 1-\eta \sigma_1^2(\mbf{A})\rba\rbp \label{eq:pf_GD_opt_mono}, 
\end{align}
we designate the unit vector
\begin{equation}
\bm{v}_1 \eqDef \frac{\bm{x}_0-\bm{x}_{\msf{f}}^* }{\lnorm \bm{x}_0-\bm{x}_{\msf{f}}^* \rnorm}
\end{equation}
as the right singular vector of $\mbf A$ corresponding to either $\sigma_1(\mbf{A})$ if $\lba 1-\eta \sigma_1^2(\mbf{A})\rba$ achieves the maximum in \eqref{eq:pf_GD_opt_mono} or $\sigma_n(\mbf{A})$ otherwise. This is determined solely by the stepsize $\eta$;  the optimal stepsize \eqref{eq:stepsize_gd} ensures that $\lba 1-\eta \sigma_1^2(\mbf{A})\rba = \lba 1-\eta \sigma_n^2(\mbf{A})\rba = \sigma_{\mathrm{GD}}$. 
To complete the construction of $\mbf A$, we complement $\bm{v}_1$ with $n-1$ orthonormal vectors to form an orthonormal basis $\{\bm{v}_i\}_{i=1}^n$ of $\mbb{R}^n$. Then, 
\begin{equation}
\label{eq:pf_GD_opt_eq41}
 \bm{x}_1-\bm{x}_{\msf{f}}^*  = \sigma_1 \lp \mbf{I}-\eta\mbf{A}^\T\mbf{A}\rp \left( \bm{x}_0-\bm{x}_{\msf{f}}^*\right), 
\end{equation}
and using \eqref{eq:recur_LS} it is easy to show by induction that 
\begin{equation}
\label{eq:pf_GD_opt_eq41}
 \bm{x}_{t+1}-\bm{x}_{\msf{f}}^*  = \sigma_1 \lp \mbf{I}-\eta\mbf{A}^\T\mbf{A}\rp \left( \bm{x}_t-\bm{x}_{\msf{f}}^*\right), 
\end{equation}
which implies that \eqref{eq:conv_LS} holds with equality $\forall t = 0, 1, \ldots$. 
\end{proof}

\begin{remark}
While variants of Lemma~\ref{lma:GD_opt} are known in the literature (e.g. \cite[Example~1.3]{Klerk-17}), they are not directly applicable because of our need to satisfy \eqref{eq:iter_opt_range}. Furthermore, Lemma~\ref{lma:GD_opt} constructs a worst-case problem instance for any initial point and constant stepsize chosen by the GD, whereas the worst-case problem instance constructed in \cite[Example~1.3]{Klerk-17} is tailored to a particular starting point $\bm{x}_0\neq \bm{0}$ and the optimal $\eta$ \eqref{eq:stepsize_gd}.
\end{remark}

\subsection{Proof of \thmref{thm:converseGD}}
\label{sec:DQ-GDc}
On one hand, we have 
\begin{align}
 \inf_{\textrm A \in \mathcal A_{\mathrm{GD}}}\sigma_{\textrm A}(n, R) &\geq  \inf_{\textrm A \in \mathcal A_{\mathrm{GD}}}\sigma_{\textrm A}(n, \infty) \label{eq:nonincr}\\
 &\geq \sigma_{\mathrm{GD}},
 \label{eq:converseGDunq}
 \end{align}
 where \eqref{eq:nonincr} holds because the left side of \eqref{eq:nonincr} is non-increasing in the data rate $R$ by definition \eqref{eq:lcr}, and  \eqref{eq:converseGDunq} is by \lemref{lma:GD_opt} since the infinite-rate algorithm in $\mathcal A_{\mathrm{GD}}$ is the one incurring no quantization error at each iteration, i.e., the GD itself. 
 
 On the other hand, to show 
 \begin{align}
 \inf_{\textrm A \in \mathcal A_{\mathrm{GD}}}\sigma_{\textrm A}(n, R) &\geq  2^{-R}, \label{eq:converseGDq}
 \end{align}
we fix an algorithm $\msf{A} \in \mathcal A_{\mathrm{GD}}$ operating at $R^\prime \leq R$ bits per dimension. The set of possible outputs of $\msf{A}$ after $T$ iterations
\begin{equation}
\label{eqn:codebook}
\begin{multlined}
\mcal{S}_{\msf{A}} \eqDef \big\{ \hat{\bm{x}}_t \in\mbb{R}^n \mid \hat{\bm{x}}_t \text{ is the output} \\
\text{of } \msf{A} \text{ after } T \text{ iterations} \big\}
\end{multlined}
\end{equation}
has cardinality 
\begin{equation}
\lba \mcal{S}_{\msf{A}}\rba = 2^{nR^\prime T}.
\end{equation}
Given $\mcal{S}_{\msf{A}}$, consider the minimum-distance quantizer $\msf{q}_{\scriptscriptstyle\msf{A}}$
\begin{equation}
\msf{q}_{\scriptscriptstyle\msf{A}}(\bm{x}) = \argmin_{\bm{\hat{x}}\in\mcal{S}_{\msf{A}}}\lnorm \hat{\bm{x}}-\bm{x}\rnorm
\end{equation}
with dynamic range $D$ and covering radius $\msf{d}\lp \msf{q}_{\scriptscriptstyle\msf{A}}\rp$ \eqref{eq:covradius}. In other words, $2^{nR^\prime T}$ Euclidean balls of radius $\msf{d}\lp \msf{q}_{\scriptscriptstyle\msf{A}}\rp$ with centers in $\mcal{S}_{\msf{A}}$ cover $\set{B}(D)$; therefore 
\begin{align}
\frac{D}{\msf{d}\lp \msf{q}_{\scriptscriptstyle\msf{A}}\rp} \leq 2^{R^\prime T} \label{eq:pf_lower_volume}.
\end{align}
(This classical volume-division argument also shows that $\rho \lp \msf{q_\msf{A}}\rp \geq 1$ \eqref{eq:def_d}).

Since one can construct an $\msf{f}\in\set{F}_n$ such that
\begin{equation}
\lnorm \msf{q}_{\scriptscriptstyle\msf{A}}(\bm{x}_{\msf{f}}^*)-\bm{x}_{\msf{f}}^*\rnorm = \msf{d}\lp \msf{q}_{\scriptscriptstyle\msf{A}}\rp,
\end{equation}
\eqref{eq:converseGDq} follows by rearranging \eqref{eq:pf_lower_volume} and taking the limit $T \to \infty$.

\subsection{Proof of \thmref{thm:converseGM}}
\label{sec:DQ-GMc}
The following lemma shows that gradient iterative methods cannot achieve an arbitrarily low contraction factor.

\begin{lemma}[{\cite[Th.~2.1.13]{Nesterov-14}}]
\label{lma:lower-Nesterov}
For any gradient method \eqref{eqn:descent-GM}, there exists an $L$-smooth and $\mu$-strongly convex function $\fcn{f}\colon \mathbb L_2 \to \R$ such that, $\forall t = 0, 1, \ldots$,
\begin{equation}
\norm{\vec{x}_t-\vec{x}_{\fcn{f}}^*} \geq \sigma_{\mathrm{HB}}^t\norm{\vec{x}_0-\vec{x}_{\fcn{f}}^*}.
\end{equation}
\end{lemma}

The proof of \thmref{thm:converseGM} follows the same steps as the proof of \thmref{thm:converseGD}, with the replacement of the converse for unquantized GD (\lemref{lma:GD_opt}) by \lemref{lma:lower-Nesterov}.

\section{Convergence analysis of K-worker NQ-GD}
\label{appx:NQ-GD}

The path of NQ-GD satisfies the following recursive relation.

\begin{lemma}[NQ-GD trajectory]
\label{lma:recur_NQGD}
At each iteration $t = 0, 1, 2, \ldots$, the path of NQ-GD with stepsize \eqref{eq:stepsize_gd} satisfies
\begin{equation}
\label{eq:recur_NQGD}
\lnorm \hat{\bm{x}}_{t+1}-\bm{x}_{\msf{f}}^*\rnorm \leq \sigma_{\mathrm{GD}} \lnorm \hat{\bm{x}}_t-\bm{x}_{\msf{f}}^*\rnorm + \frac{\eta_{\mathrm{GD}}}{K}\sum_{k=1}^K\norm{\bm{e}_{t,k}},
\end{equation}
where $\bm{e}_{t,k} \triangleq \bm{q}_{t,k} - \nabla\msf{f}_k(\hat{\bm{x}}_t)$, and $\sigma_{\mathrm{GD}}$ is defined in \eqref{eq:sigmaDQ}.
\end{lemma}

\begin{proof}[Proof of {Lemma~\ref{lma:recur_NQGD}}]
The update rule (Line~5) and the quantizer inputs (Line~4) together imply
\begin{align}
\hat{\bm{x}}_{t+1} &= \hat{\bm{x}}_t - \frac{\eta_{\mathrm{GD}}}{K}\sum_{k=1}^K\lp \nabla\msf{f}_k(\hat{\bm{x}}_t)+\bm{e}_{t,k}\rp \\
 &= \hat{\bm{x}}_t - \eta_{\mathrm{GD}}\nabla\msf{f}(\hat{\bm{x}}_t) - \frac{\eta_{\mathrm{GD}}}{K}\sum_{k=1}^K\bm{e}_{t,k} \label{eq:pf_recur_NQGD_obj}
\end{align}
where \eqref{eq:pf_recur_NQGD_obj} is due to \eqref{eq:obj}. Triangle inequality now implies
\begin{align}
 &\lnorm \hat{\bm{x}}_{t+1}-\bm{x}_{\msf{f}}^*\rnorm \\
\leq\> &\norm{\hat{\bm{x}}_t-\eta_{\mathrm{GD}}\nabla\msf{f}(\hat{\bm{x}}_t)-\bm{x}_{\msf{f}}^*} + \frac{\eta_{\mathrm{GD}}}{K}\sum_{k=1}^K\norm{\bm{e}_{t,k}}. \label{eqn:pf-recur-NQGD-coer}
\end{align}
Applying the coercive property of smooth and strongly convex functions \cite[Th.~2.1.12]{Nesterov-14} in the same manner as in \cite[Proof of Th.~2.1.5]{Nesterov-14} to further upper-bound the first term in \eqref{eqn:pf-recur-NQGD-coer} gives \eqref{eq:recur_NQGD}.
\end{proof}

Denote for brevity 
\begin{align}
\sigma &\triangleq \sigma_{\mathrm{GD}} + C \label{eq:nq}, \\
C &\triangleq \frac{\eta_{\mathrm{GD}} \rho_n}{K}\sum_{k=1}^K \frac{L_k}{2^{R_k}}. \label{eq:C}
\end{align}
Minimizing \eqref{eq:nq} under the sum rate constrant \eqref{eq:rate_condi} is a convex optimization problem whose solution is given by \eqref{eq:rate_alloc_non_unif} and whose optimal value is given by the right side of \eqref{eq:sigmaNQ-GDK}. 
The following nonasymptotic result, which with the optimal rate allocation immediately yields \thmref{thm:NQ-GD}, uses an inductive argument to simultaneously bound both terms in the right-hand side of \eqref{eq:recur_NQGD}. 
\begin{theorem}[Convergence of NQ-GD]
\label{thm:naive}
In the setting of \thmref{thm:NQ-GD},
\begin{equation}
\label{eq:conv_naive}
\lnorm \hat{\bm{x}}_t-\bm{x}_\msf{f}^*\rnorm \leq \sigma^t D, 
\end{equation}
where $\sigma$ is defined in \eqref{eq:nq}.
\end{theorem}
\begin{proof}
We prove \eqref{eq:conv_naive} by induction. 

\begin{itemize}
\item Base case: for $t=0$, \eqref{eq:conv_naive} holds by \eqref{eq:iter_opt_range}.
\item Inductive step: Suppose \eqref{eq:conv_naive} holds for iteration $t$. Then, by $L_k$-smoothness of $\fcn f_k$, \eqref{eqn:condi-str-grow} and the inductive hypothesis, 
\begin{align}
\norm{\nabla\msf{f}_k(\hat{\bm{x}}_t)} &\leq L_k\norm{\hat{\bm{x}}_t-\bm{x}_{\msf{f}}^*} \label{eq:sharedmin} \\
 &\leq \sigma^t L_k D,
\end{align}
which due to the choice of dynamic ranges \eqref{eq:range_NQGD} and  \eqref{eq:e_d} leads to 
\begin{equation}
\label{eq:pf_NQGD_max_err}
\lnorm \bm{e}_{t,k}\rnorm \leq \frac{\rho_n}{2^{R_k}}\sigma^t L_k D.
\end{equation}
Applying \eqref{eq:pf_NQGD_max_err} and the inductive hypothesis to \eqref{eq:recur_NQGD} yields
\begin{align}
 \lnorm \hat{\bm{x}}_{t+1}-\bm{x}_{\msf{f}}^*\rnorm &\leq  \sigma^t \left[ \sigma_{\mathrm{GD}} + C  \right] D, 
\end{align}
which is exactly \eqref{eq:conv_naive} for $t+1$ if the optimal rate allocation is employed. 
\end{itemize}
\end{proof}

\hl{Without assumption (iv) \eqref{eqn:condi-str-grow} that the summands $\mathsf f_k$ share the minimizer, NQ-GD converges only to a neighborhood of  $\vec{x}_{\fcn{f}}^*$, albeit exponentially fast. The radius of this neighborhood vanishes as the data rate $R\to\infty$:

\begin{theorem} In the setting of \thmref{thm:NQ-GD} but without assumption (iv) \eqref{eqn:condi-str-grow} on the objective function, setting the dynamic ranges to $r_{t,k}^\prime$, where 
\begin{align}
\label{eq:range_NQGDprime}
r_{t,k}^\prime &= r_{t, k} + 2 L_k D\left(C \sum_{\tau=0}^{t-1}\sigma^{\tau}   +1 \right),
\end{align} 
$r_{t. k}$ is defined in \eqref{eq:range_NQGD}, and $C$ is defined in \eqref{eq:C}, NQ-GD achieves
\begin{equation}
\label{eqn:conv-NQ-GD}
\begin{multlined}
\norm{\hat{\vec{x}}_t-\vec{x}_{\fcn{f}}^*} \leq \sigma^tD \\
+ 2 D C \sum_{\tau=0}^{t-1}\sigma^{\tau},
\end{multlined}
\end{equation}
where $\sigma$ is defined in \eqref{eq:nq}.
\label{thm:NQ-GDprime}
\end{theorem}

\begin{proof}
 We follow the reasoning in the proof of \thmref{thm:naive} until \eqref{eq:sharedmin}, which we replace by
\begin{align}
\norm{\nabla\msf{f}_k(\hat{\bm{x}}_t)} &\leq L_k\norm{\hat{\bm{x}}_t-\bm{x}_{\msf{f}_k}^*}\\
 &\leq L_k\pr{\norm{\hat{\vec{x}}_t-\vec{x}_{\fcn{f}}^*}+\norm{\vec{x}_{\fcn{f}_k}^*-\vec{x}_{\fcn{f}}^*}} \\
 &\leq L_k\pr{\norm{\hat{\vec{x}}_t-\vec{x}_{\fcn{f}}^*}+2D} \label{eqn:3rd-inequ} \\
 &\leq L_k \left( \sigma^tD 
+ 2 D C \sum_{\tau=0}^{t-1}\sigma^{\tau} + 2D \right)\label{eq:indhypoprime}, 
\end{align}
where \eqref{eq:indhypoprime} applies inductive hypothesis \eqref{eqn:conv-NQ-GD}. Due to \eqref{eq:indhypoprime}, the choice of dynamic ranges \eqref{eq:range_NQGDprime} ensures that the quantizer input stays within its dynamic range, which limits the quantization error to $\rho_n 2^{-R_k}$ times the right-hand side of \eqref{eq:indhypoprime} (recall \eqref{eq:e_d}). Plugging this bound on quantization error in \eqref{eq:recur_NQGD} and applying inductive hypothesis \eqref{eqn:conv-NQ-GD} again establishes \eqref{eqn:conv-NQ-GD} for $t + 1$. 

\end{proof}
}

\hl{\section{Analysis of quantized GD with error feedback of~\cite{Seide-14,Stich-18,Karimireddy-19}}
\label{subsec:EFGD}

We derive the following convergence guarantee on quantized GD with the error feedback mechanism of \cite{Seide-14,Stich-18,Karimireddy-19}.

\begin{lemma}[{Trajectory of GD with error feedback of \cite{Seide-14,Stich-18,Karimireddy-19}}]
\label{lma:recur_EFGD}
Let $\msf{f}$ be an $L$-smooth and $\mu$-strongly convex function on $\mbb{R}^n$. Then, the distance to the optimizer at each iteration $t\in\mbb{N}$ of quantized gradient descent \eqref{eq:QGD} with quantizer input \eqref{eq:input_EFGD} and stepsize \eqref{eq:stepsize_gd} is bounded as
\begin{align}
\label{eq:recur_EFGD}
\lnorm \hat{\bm{x}}_{t+1}-\bm{x}^*\rnorm \leq&~ \sigma_{\GD}\lnorm \hat{\bm{x}}_t-\bm{x}^*\rnorm + \eta_{\GD}\lnorm \bm{e}_t\rnorm \notag\\
&~+ (\sigma_{\GD}+\eta_{\GD} L)\eta\lnorm \bm{e}_{t-1}\rnorm.
\end{align}
\end{lemma}
\begin{proof}
Using \eqref{eq:input_EFGD}, the quantizer output can be computed as
\begin{equation}
\bm{q}_t = \big[ \nabla\msf{f}(\hat{\bm{x}}_t)+\bm{e}_t\big] - \bm{e}_{t-1},
\label{eq:gef}
\end{equation}
and plugging \eqref{eq:gef} into \eqref{eq:QGD} gives 
\begin{align}
\hat{\bm{x}}_{t+1} &= \hat{\bm{x}}_t - \eta\nabla\msf{f}(\hat{\bm{x}}_t) - \eta\bm{e}_t + \eta\bm{e}_{t-1}. \label{eq:recur_shift}
\end{align}
Denoting a shifted trajectory by
\begin{equation}
\label{eq:iter_shift}
\tilde{\bm{x}}_t \eqDef \hat{\bm{x}}_t + \eta\bm{e}_{t-1},
\end{equation}
we rewrite \eqref{eq:recur_shift} as
\begin{equation}
\tilde{\bm{x}}_{t+1} = \tilde{\bm{x}}_t - \eta\nabla\msf{f}(\tilde{\bm{x}}_t-\eta\bm{e}_{t-1}),
\end{equation}
which via triangle inequality implies
\begin{align}
\lnorm \tilde{\bm{x}}_{t+1}-\bm{x}^*\rnorm \leq\> &\lnorm \tilde{\bm{x}}_t-\bm{x}^*-\eta\nabla\msf{f}(\tilde{\bm{x}}_t)\rnorm \\
 &+ \eta\lnorm \nabla\msf{f}(\tilde{\bm{x}}_t)-\nabla\msf{f}(\tilde{\bm{x}}_t-\eta\bm{e}_{t-1})\rnorm \\
\leq\> &\sigma_{\GD}\lnorm \tilde{\bm{x}}_t-\bm{x}^*\rnorm + \eta^2L\lnorm \bm{e}_{t-1}\rnorm, \label{eq:pf_recur_EFGD_triangle}
\end{align}
where the first term in the right side of \eqref{eq:pf_recur_EFGD_triangle} is obtained by the convergence guarantee of GD (\lemref{lma:conv-GD}), and the second term is due to $L$-smoothness \eqref{eq:smooth}. Since \eqref{eq:iter_shift} implies
\begin{align}
\lnorm \tilde{\bm{x}}_t-\bm{x}^*\rnorm - \eta\lnorm \bm{e}_{t-1}\rnorm &\leq 
\lnorm \hat{\bm{x}}_t-\bm{x}^*\rnorm  \label{eq:triangle_a} \\
 &\leq \lnorm \tilde{\bm{x}}_t-\bm{x}^*\rnorm + \eta\lnorm \bm{e}_{t-1}\rnorm, \label{eq:triangle_b}
\end{align}
we leverage \eqref{eq:pf_recur_EFGD_triangle} to control $\lnorm \hat{\bm{x}}_{t+1}-\bm{x}^*\rnorm$ as follows:
\begin{align}
 &\lnorm \hat{\bm{x}}_{t+1}-\bm{x}^*\rnorm \notag\\
\leq\> &\sigma_{\GD}\lnorm \tilde{\bm{x}}_t-\bm{x}^*\rnorm + \eta^2L\lnorm \bm{e}_{t-1}\rnorm + \eta\lnorm \bm{e}_t\rnorm \\
\leq\> &\sigma_{\GD} \lnorm \hat{\bm{x}}_t-\bm{x}^*\rnorm + \eta\lnorm \bm{e}_t\rnorm + (\sigma_{\GD}+\eta L)\eta\lnorm \bm{e}_{t-1}\rnorm, \label{eq:pf_recur_EFGD_shift}
\end{align}
where \eqref{eq:pf_recur_EFGD_shift} is due to \eqref{eq:triangle_a}.
\end{proof}
Compared to the guarantee on NQ-GD in \lemref{lma:recur_NQGD}, that on GD with error feedback of \cite{Seide-14,Stich-18,Karimireddy-19} in \lemref{lma:recur_EFGD} has an extra error term. Thus, it is unclear whether the error feedback mechanism of \cite{Seide-14,Stich-18,Karimireddy-19} can even improve upon NQ-GD in our nonstochastic problem setting.}

\ifCLASSOPTIONcaptionsoff
  \newpage
\fi



\bibliographystyle{IEEEtran}
\bibliography{Ref}

\begin{thebibliography}{10}
\providecommand{\url}[1]{#1}
\csname url@samestyle\endcsname
\providecommand{\newblock}{\relax}
\providecommand{\bibinfo}[2]{#2}
\providecommand{\BIBentrySTDinterwordspacing}{\spaceskip=0pt\relax}
\providecommand{\BIBentryALTinterwordstretchfactor}{4}
\providecommand{\BIBentryALTinterwordspacing}{\spaceskip=\fontdimen2\font plus
\BIBentryALTinterwordstretchfactor\fontdimen3\font minus
  \fontdimen4\font\relax}
\providecommand{\BIBforeignlanguage}[2]{{%
\expandafter\ifx\csname l@#1\endcsname\relax
\typeout{** WARNING: IEEEtran.bst: No hyphenation pattern has been}%
\typeout{** loaded for the language `#1'. Using the pattern for}%
\typeout{** the default language instead.}%
\else
\language=\csname l@#1\endcsname
\fi
#2}}
\providecommand{\BIBdecl}{\relax}
\BIBdecl

\bibitem{lin2021dq}
C.-Y. Lin, V.~Kostina, and B.~Hassibi, ``Differentially quantized gradient
  descent,'' in \emph{Proceedings 2021 IEEE International Symposium on
  Information Theory}, July 2021.

\bibitem{Zinkevich-10}
M.~Zinkevich, M.~Weimer, L.~Li, and A.~J. Smola, ``Parallelized stochastic
  gradient descent,'' in \emph{Advances in Neural Information Processing
  Systems 23}, Vancouver, British Columbia, Canada, Dec. 2010, pp. 2595--2603.

\bibitem{Recht-11}
B.~Recht, C.~Re, S.~Wright, and F.~Niu, ``Hogwild: A lock-free approach to
  parallelizing stochastic gradient descent,'' in \emph{Advances in Neural
  Information Processing Systems 24}, Granada, Spain, Dec. 2011, pp. 693--701.

\bibitem{Bekkerman-11}
R.~Bekkerman, M.~Bilenko, and J.~Langford, \emph{Scaling Up Machine Learning:
  Parallel and Distributed Approaches}.\hskip 1em plus 0.5em minus 0.4em\relax
  New York, NY, USA: Cambridge University Press, Dec. 2011.

\bibitem{Dean-12}
J.~Dean, G.~Corrado, R.~Monga, K.~Chen, M.~Devin, M.~Mao, M.~Ranzato,
  A.~Senior, P.~Tucker, K.~Yang, Q.~V. Le, and A.~Y. Ng, ``Large scale
  distributed deep networks,'' in \emph{Advances in Neural Information
  Processing Systems 25}, Lake Tahoe, NV, USA, Dec. 2012, pp. 1223--1231.

\bibitem{Chilimbi-14}
T.~Chilimbi, Y.~Suzue, J.~Apacible, and K.~Kalyanaraman, ``Project {A}dam:
  Building an efficient and scalable deep learning training system,'' in
  \emph{11th {USENIX} Symposium on Operating Systems Design and Implementation
  ({OSDI} 14)}, Broomfield, CO, Oct. 2014, pp. 571--582.

\bibitem{Sa-15}
C.~M. De~Sa, C.~Zhang, K.~Olukotun, C.~R\'{e}, and C.~R\'{e}, ``Taming the
  wild: A unified analysis of hogwild-style algorithms,'' in \emph{Advances in
  Neural Information Processing Systems 28}, C.~Cortes, N.~D. Lawrence, D.~D.
  Lee, M.~Sugiyama, and R.~Garnett, Eds.\hskip 1em plus 0.5em minus 0.4em\relax
  Curran Associates, Dec. 2015, pp. 2674--2682.

\bibitem{Konecny-16}
J.~Kone\v{c}n\'y, H.~B. McMahan, F.~X. Yu, P.~Richt\'arik, A.~T. Suresh, and
  D.~Bacon, ``Federated learning: Strategies for improving communication
  efficiency,'' in \emph{NIPS Workshop on Private Multi-Party Machine
  Learning}, Dec. 2016.

\bibitem{Scaman-17}
K.~Scaman, F.~Bach, S.~Bubeck, Y.~T. Lee, and L.~Massouli{\'e}, ``Optimal
  algorithms for smooth and strongly convex distributed optimization in
  networks,'' in \emph{Proceedings of the 34th International Conference on
  Machine Learning}, ser. Proceedings of Machine Learning Research, D.~Precup
  and Y.~W. Teh, Eds., vol.~70.\hskip 1em plus 0.5em minus 0.4em\relax
  International Convention Centre, Sydney, Australia: PMLR, 06--11 Aug 2017,
  pp. 3027--3036.

\bibitem{Seide-14}
F.~Seide, H.~Fu, J.~Droppo, G.~Li, and D.~Yu, ``1-bit stochastic gradient
  descent and application to data-parallel distributed training of speech
  {DNN}s,'' in \emph{Interspeech 2014}, Sep. 2014.

\bibitem{Li-14-communication}
M.~Li, D.~G. Andersen, A.~J. Smola, and K.~Yu, ``Communication efficient
  distributed machine learning with the parameter server,'' in \emph{Advances
  in Neural Information Processing Systems 27}, Montreal, QB, Canada, Dec.
  2014, pp. 19--27.

\bibitem{Strom-15}
N.~Strom, ``Scalable distributed {DNN} training using commodity {GPU} cloud
  computing,'' in \emph{INTERSPEECH}, Sep. 2015.

\bibitem{Zhang-15}
S.~Zhang, A.~E. Choromanska, and Y.~LeCun, ``Deep learning with elastic
  averaging {SGD},'' in \emph{Advances in Neural Information Processing Systems
  28}, Montreal, QB, Canada, Dec. 2015, pp. 685--693.

\bibitem{RobbinsSutton-51}
H.~Robbins and S.~Monro, ``A stochastic approximation method,'' \emph{Ann.
  Math. Statist.}, vol.~22, no.~3, pp. 400--407, Sep. 1951.

\bibitem{Wen-17}
W.~Wen, C.~Xu, F.~Yan, C.~Wu, Y.~Wang, Y.~Chen, and H.~Li, ``Tern{G}rad:
  Ternary gradients to reduce communication in distributed deep learning,'' in
  \emph{Advances in Neural Information Processing Systems 30}, Long Beach, CA,
  USA, Dec. 2017, pp. 1509--1519.

\bibitem{Bernstein-18}
J.~Bernstein, Y.-X. Wang, K.~Azizzadenesheli, and A.~Anandkumar, ``sign{SGD}:
  Compressed optimisation for non-convex problems,'' in \emph{Proceedings of
  the 35th International Conference on Machine Learning}, Long Beach, CA, USA,
  Jul. 2018, pp. 560--569.

\bibitem{Alistarh-17}
D.~Alistarh, D.~Grubic, J.~Li, R.~Tomioka, and M.~Vojnovic, ``{QSGD}:
  Communication-efficient {SGD} via gradient quantization and encoding,'' in
  \emph{Advances in Neural Information Processing Systems 30}, Long Beach, CA,
  USA, Dec. 2017, pp. 1709--1720.

\bibitem{RamezaniKebrya-19}
A.~Ramezani-Kebrya, F.~Faghri, and D.~M. Roy, ``{NUQSGD}: Improved
  communication efficiency for data-parallel {SGD} via nonuniform
  quantization,'' Aug. 2019.

\bibitem{MayekarTyagi-20}
P.~Mayekar and H.~Tyagi, ``Limits on gradient compression for stochastic
  optimization,'' \emph{arXiv}, vol. 2001.09032, Jan. 2020.

\bibitem{Gandikota-19}
V.~Gandikota, D.~Kane, R.~K. Maity, and A.~Mazumdar, ``vq{SGD}: Vector
  quantized stochastic gradient descent,'' Nov. 2019.

\bibitem{AjiHeafield-17}
A.~F. Aji and K.~Heafield, ``Sparse communication for distributed gradient
  descent,'' in \emph{Proceedings of the 2017 Conference on Empirical Methods
  in Natural Language Processing}, Copenhagen, Denmark, Sep. 2017, pp.
  440--445.

\bibitem{Stich-18}
S.~U. Stich, J.-B. Cordonnier, and M.~Jaggi, ``Sparsified {SGD} with memory,''
  in \emph{Advances in Neural Information Processing Systems 31}, Montr{\'e}al,
  Canada, Dec. 2018, pp. 4447--4458.

\bibitem{Wangni-18}
J.~Wangni, J.~Wang, J.~Liu, and T.~Zhang, ``Gradient sparsification for
  communication-efficient distributed optimization,'' in \emph{Advances in
  Neural Information Processing Systems 31}, Montr{\'e}al, Canada, Dec. 2018,
  pp. 1299--1309.

\bibitem{Wang-18}
H.~Wang, S.~Sievert, S.~Liu, Z.~Charles, D.~Papailiopoulos, and S.~Wright,
  ``{ATOMO}: Communication-efficient learning via atomic sparsification,'' in
  \emph{Advances in Neural Information Processing Systems 31}, Montr{\'e}al,
  Canada, Dec. 2018, pp. 9850--9861.

\bibitem{Lin-18}
Y.~Lin, S.~Han, H.~Mao, Y.~Wang, and B.~Dally, ``Deep gradient compression:
  Reducing the communication bandwidth for distributed training,'' in
  \emph{International Conference on Learning Representations}, Vancouver, BC,
  Canada, Apr. 2018.

\bibitem{Dryden-16}
N.~Dryden, S.~A. Jacobs, T.~Moon, and B.~Van~Essen, ``Communication
  quantization for data-parallel training of deep neural networks,'' in
  \emph{Proceedings of the Workshop on Machine Learning in High Performance
  Computing Environments}, 2016, pp. 1--8.

\bibitem{Alistarh-18}
D.~Alistarh, T.~Hoefler, M.~Johansson, N.~Konstantinov, S.~Khirirat, and
  C.~Renggli, ``The convergence of sparsified gradient methods,'' in
  \emph{Advances in Neural Information Processing Systems 31}, S.~Bengio,
  H.~Wallach, H.~Larochelle, K.~Grauman, N.~Cesa-Bianchi, and R.~Garnett,
  Eds.\hskip 1em plus 0.5em minus 0.4em\relax Curran Associates, Dec. 2018, pp.
  5973--5983.

\bibitem{yu2018gradiveq}
M.~Yu, Z.~Lin, K.~Narra, S.~Li, Y.~Li, N.~S. Kim, A.~Schwing, M.~Annavaram, and
  S.~Avestimehr, ``{GradiVeQ}: Vector quantization for bandwidth-efficient
  gradient aggregation in distributed {CNN} training,'' \emph{Advances in
  Neural Information Processing Systems}, vol.~31, pp. 5123--5133, Jan. 2018.

\bibitem{AmiriGunduz-19}
M.~M. {Amiri} and D.~{G{\"u}nd{\"u}z}, ``Machine learning at the wireless edge:
  Distributed stochastic gradient descent over-the-air,'' in \emph{2019 IEEE
  International Symposium on Information Theory (ISIT)}, 2019, pp. 1432--1436.

\bibitem{Zhu-19}
G.~{Zhu}, Y.~{Wang}, and K.~{Huang}, ``Broadband analog aggregation for
  low-latency federated edge learning,'' \emph{IEEE Transactions on Wireless
  Communications}, vol.~19, no.~1, pp. 491--506, 2020.

\bibitem{Yang-20}
K.~{Yang}, T.~{Jiang}, Y.~{Shi}, and Z.~{Ding}, ``Federated learning via
  over-the-air computation,'' \emph{IEEE Transactions on Wireless
  Communications}, vol.~19, no.~3, pp. 2022--2035, 2020.

\bibitem{Nemirovski-09}
A.~Nemirovski, A.~Juditsky, G.~Lan, and A.~Shapiro, ``Robust stochastic
  approximation approach to stochastic programming,'' \emph{SIAM Journal on
  Optimization}, vol.~19, no.~4, pp. 1574--1609, Jan. 2009.

\bibitem{Bubeck-15}
S.~Bubeck, ``Convex optimization: Algorithms and complexity,'' \emph{Found.
  Trends Mach. Learn.}, vol.~8, no. 3-4, pp. 231--357, Nov. 2015.

\bibitem{Bottou-16}
L.~Bottou, F.~E. Curtis, and J.~Nocedal, ``Optimization methods for large-scale
  machine learning,'' \emph{SIAM Review}, vol.~60, no.~2, pp. 223--311, May
  2018.

\bibitem{Mishchenko-19}
K.~Mishchenko, E.~Gorbunov, M.~Tak\'a\v{c}, and P.~Richt\'arik, ``Distributed
  learning with compressed gradient differences,'' Jan. 2019.

\bibitem{Horvath-19-stochastic}
S.~Horv\'ath, D.~Kovalev, K.~Mishchenko, S.~Stich, and P.~Richt\'arik,
  ``Stochastic distributed learning with gradient quantization and variance
  reduction,'' Apr. 2019.

\bibitem{Horvath-19-natural}
S.~Horv\'ath, C.-Y. Ho, L.~Horv\'ath, A.~N. Sahu, M.~Canini, and
  P.~Richt\'arik, ``Natural compression for distributed deep learning,'' May.
  2019.

\bibitem{philippenko2020bidirectional}
C.~Philippenko and A.~Dieuleveut, ``Bidirectional compression in heterogeneous
  settings for distributed or federated learning with partial participation:
  tight convergence guarantees,'' \emph{arXiv preprint arXiv:2006.14591}, June
  2020.

\bibitem{gorbunov2020unified}
E.~Gorbunov, F.~Hanzely, and P.~Richt\'arik, ``A unified theory of {SGD}:
  Variance reduction, sampling, quantization and coordinate descent,'' in
  \emph{International Conference on Artificial Intelligence and Statistics},
  June 2020, pp. 680--690.

\bibitem{gorbunov2020linearly}
E.~Gorbunov, D.~Kovalev, D.~Makarenko, and P.~Richt\'arik, ``Linearly
  converging error compensated {SGD},'' \emph{Advances in Neural Information
  Processing Systems}, vol.~33, pp. 20\,889--20\,900, Dec. 2020.

\bibitem{khaled2020unified}
A.~Khaled, O.~Sebbouh, N.~Loizou, R.~M. Gower, and P.~Richt\'arik, ``Unified
  analysis of stochastic gradient methods for composite convex and smooth
  optimization,'' \emph{arXiv preprint arXiv:2006.11573}, June 2020.

\bibitem{gorbunov2021marina}
E.~Gorbunov, K.~P. Burlachenko, Z.~Li, and P.~Richt{\'a}rik, ``{MARINA}: Faster
  non-convex distributed learning with compression,'' in \emph{International
  Conference on Machine Learning}.\hskip 1em plus 0.5em minus 0.4em\relax PMLR,
  July 2021, pp. 3788--3798.

\bibitem{islamov2021distributed}
R.~Islamov, X.~Qian, and P.~Richt\'arik, ``Distributed second order methods
  with fast rates and compressed communication,'' in \emph{International
  Conference on Machine Learning}.\hskip 1em plus 0.5em minus 0.4em\relax PMLR,
  July 2021, pp. 4617--4628.

\bibitem{Karimireddy-19}
S.~P. Karimireddy, Q.~Rebjock, S.~Stich, and M.~Jaggi, ``Error feedback fixes
  {S}ign{SGD} and other gradient compression schemes,'' in \emph{Proceedings of
  the 36th International Conference on Machine Learning}, vol.~97, Long Beach,
  CA, USA, June 2019, pp. 3252--3261.

\bibitem{Aleksandr-20}
A.~Beznosikov, S.~Horv\'ath, P.~Richt\'arik, and M.~Safaryan, ``On biased
  compression for distributed learning,'' \emph{arXiv:2002.12410}, Feb. 2020.

\bibitem{magnusson2020maintaining}
S.~Magn{\'u}sson, H.~Shokri-Ghadikolaei, and N.~Li, ``On maintaining linear
  convergence of distributed learning and optimization under limited
  communication,'' \emph{IEEE Transactions on Signal Processing}, vol.~68, pp.
  6101--6116, 2020.

\bibitem{Gray-89}
R.~M. Gray, \emph{Source Coding Theory}.\hskip 1em plus 0.5em minus 0.4em\relax
  Kluwer Academic Publishers, Oct. 1989.

\bibitem{Zheng-19}
S.~Zheng, Z.~Huang, and J.~Kwok, ``Communication-efficient distributed
  blockwise momentum {SGD} with error-feedback,'' in \emph{Advances in Neural
  Information Processing Systems 32}.\hskip 1em plus 0.5em minus 0.4em\relax
  Curran Associates, Dec. 2019, pp. 11\,450--11\,460.

\bibitem{Wu-18}
J.~Wu, W.~Huang, J.~Huang, and T.~Zhang, ``Error compensated quantized {SGD}
  and its applications to large-scale distributed optimization,'' in
  \emph{Proceedings of the 35th International Conference on Machine Learning},
  vol.~80, Stockholm, Sweden, July 2018, pp. 5325--5333.

\bibitem{sun2020lazily}
J.~Sun, T.~Chen, G.~B. Giannakis, Q.~Yang, and Z.~Yang, ``Lazily aggregated
  quantized gradient innovation for communication-efficient federated
  learning,'' \emph{IEEE Transactions on Pattern Analysis and Machine
  Intelligence}, 2020.

\bibitem{qian2021error}
X.~Qian, P.~Richt{\'a}rik, and T.~Zhang, ``Error compensated distributed sgd
  can be accelerated,'' \emph{Advances in Neural Information Processing
  Systems}, vol.~34, Dec. 2021.

\bibitem{richtarik2021ef21}
P.~Richt\'arik, I.~Sokolov, and I.~Fatkhullin, ``{EF21}: A new, simpler,
  theoretically better, and practically faster error feedback,'' \emph{Advances
  in Neural Information Processing Systems}, vol.~34, Dec. 2021.

\bibitem{horvath2020better}
S.~Horv\'ath and P.~Richt\'arik, ``A better alternative to error feedback for
  communication-efficient distributed learning,'' in \emph{Eighth International
  Conference on Learning Representations}, Apr. 2020.

\bibitem{Acharya-19}
J.~Acharya, C.~De~Sa, D.~Foster, and K.~Sridharan, ``Distributed learning with
  sublinear communication,'' in \emph{Proceedings of the 36th International
  Conference on Machine Learning}, ser. Proceedings of Machine Learning
  Research, K.~Chaudhuri and R.~Salakhutdinov, Eds., vol.~97.\hskip 1em plus
  0.5em minus 0.4em\relax Long Beach, CA, USA: PMLR, 09--15 Jun. 2019, pp.
  40--50.

\bibitem{MayekarTyagi-19}
P.~Mayekar and H.~Tyagi, ``{RATQ}: A universal fixed-length quantizer for
  stochastic optimization,'' \emph{arXiv}, vol. 1908.08200, Dec. 2019.

\bibitem{Boyd-11}
S.~Boyd, N.~Parikh, E.~Chu, B.~Peleato, and J.~Eckstein, ``Distributed
  optimization and statistical learning via the alternating direction method of
  multipliers,'' \emph{Found. Trends Mach. Learn.}, vol.~3, no.~1, pp. 1--122,
  Jan 2011.

\bibitem{Zhang-13-communication}
Y.~Zhang, J.~C. Duchi, and M.~J. Wainwright, ``Communication-efficient
  algorithms for statistical optimization,'' \emph{J. Mach. Learn. Res.},
  vol.~14, no.~1, pp. 3321--3363, Jan 2013.

\bibitem{Suresh-17}
A.~T. Suresh, F.~X. Yu, S.~Kumar, and H.~B. McMahan, ``Distributed mean
  estimation with limited communication,'' in \emph{Proceedings of the 34th
  International Conference on Machine Learning}, ser. Proceedings of Machine
  Learning Research, vol.~70.\hskip 1em plus 0.5em minus 0.4em\relax PMLR,
  06--11 Aug 2017, pp. 3329--3337.

\bibitem{Nedic-09}
A.~{Nedic}, A.~{Olshevsky}, A.~{Ozdaglar}, and J.~N. {Tsitsiklis}, ``On
  distributed averaging algorithms and quantization effects,'' \emph{IEEE
  Transactions on Automatic Control}, vol.~54, no.~11, pp. 2506--2517, Nov
  2009.

\bibitem{Reisizadeh-19}
A.~{Reisizadeh}, A.~{Mokhtari}, H.~{Hassani}, and R.~{Pedarsani}, ``An exact
  quantized decentralized gradient descent algorithm,'' \emph{IEEE Transactions
  on Signal Processing}, vol.~67, no.~19, pp. 4934--4947, Oct 2019.

\bibitem{Zhang-13-information}
Y.~Zhang, J.~Duchi, M.~I. Jordan, and M.~J. Wainwright, ``Information-theoretic
  lower bounds for distributed statistical estimation with communication
  constraints,'' in \emph{Advances in Neural Information Processing Systems
  26}, 2013, pp. 2328--2336.

\bibitem{TsitsiklisLuo-87}
J.~N. Tsitsiklis and Z.-Q. Luo, ``Communication complexity of convex
  optimization,'' \emph{Journal of Complexity}, vol.~3, no.~3, pp. 231 -- 243,
  1987.

\bibitem{ArjevaniShamir-15}
Y.~Arjevani and O.~Shamir, ``Communication complexity of distributed convex
  learning and optimization,'' in \emph{Advances in Neural Information
  Processing Systems 28}, Dec. 2015, pp. 1756--1764.

\bibitem{Li-14-scaling}
M.~Li, D.~G. Andersen, J.~W. Park, A.~J. Smola, A.~Ahmed, V.~Josifovski,
  J.~Long, E.~J. Shekita, and B.-Y. Su, ``Scaling distributed machine learning
  with the parameter server,'' in \emph{11th {USENIX} Symposium on Operating
  Systems Design and Implementation}, Broomfield, CO, Oct. 2014, pp. 583--598.

\bibitem{Khirirat-18}
S.~Khirirat, H.~R. Feyzmahdavian, and M.~Johansson, ``Distributed learning with
  compressed gradients,'' \emph{arXiv: 1806.06573}, June 2018.

\bibitem{Nesterov-14}
Y.~Nesterov, \emph{Introductory Lectures on Convex Optimization: A Basic
  Course}.\hskip 1em plus 0.5em minus 0.4em\relax Springer, Dec. 2014.

\bibitem{Polyak-87}
B.~T. Polyak, \emph{Introduction to optimization}, ser. Translations Series in
  Mathematics and Engineering.\hskip 1em plus 0.5em minus 0.4em\relax
  Optimization Software, May 1987.

\bibitem{Rogers-63}
C.~A. Rogers, ``Covering a sphere with spheres,'' \emph{Mathematika}, vol.~10,
  no.~2, pp. 157---164, Dec. 1963.

\bibitem{FriedlanderSchmidt-12}
M.~Friedlander and M.~Schmidt, ``Hybrid deterministic-stochastic methods for
  data fitting,'' \emph{SIAM Journal on Scientific Computing}, vol.~34, no.~3,
  pp. A1380--A1405, May 2012.

\bibitem{nesterov82method}
Y.~Nesterov, ``A method of solving a convex programming problem with
  convergence rate {$O \left( 1/k^2\right)$},'' in \emph{Soviet Mathematics
  Doklady}, vol.~27, no.~2, 1982, pp. 372--376.

\bibitem{AllenZhuOrecchia-14}
Z.~Allen-Zhu and L.~Orecchia, ``Linear coupling: An ultimate unification of
  gradient and mirror descent,'' \emph{arXiv: 1407.1537}, July 2014.

\bibitem{Su-16}
W.~Su, S.~Boyd, and E.~J. Cand{{\`e}}s, ``A differential equation for modeling
  {N}esterov's accelerated gradient method: Theory and insights,''
  \emph{Journal of Machine Learning Research}, vol.~17, no. 153, pp. 1--43,
  2016.

\bibitem{Zamir-14}
R.~Zamir, \emph{Lattice Coding for Signals and Networks: A Structured Coding
  Approach to Quantization, Modulation, and Multiuser Information
  Theory}.\hskip 1em plus 0.5em minus 0.4em\relax Cambridge University Press,
  Sep. 2014.

\bibitem{Kolodziej-19}
S.~Kolodziej, M.~Aznaveh, M.~Bullock, J.~David, T.~Davis, M.~Henderson, Y.~Hu,
  and R.~Sandstrom, ``The {SuiteSparse} matrix collection website interface,''
  \emph{Journal of Open Source Software}, vol.~4, no.~35, p. 1244, Mar. 2019.

\bibitem{Klerk-17}
E.~de~Klerk, F.~Glineur, and A.~B. Taylor, ``On the worst-case complexity of
  the gradient method with exact line search for smooth strongly convex
  functions,'' \emph{Optim. Lett.}, vol.~11, no.~7, pp. 1185--1199, Oct. 2017.

\bibitem{Arjevani-16}
Y.~Arjevani, S.~Shalev-Shwartz, and O.~Shamir, ``On lower and upper bounds in
  smooth and strongly convex optimization,'' \emph{Journal of Machine Learning
  Research}, vol.~17, no. 126, pp. 1--51, 2016.

\bibitem{BoydVandenberghe-04}
S.~Boyd and L.~Vandenberghe, \emph{Convex Optimization}.\hskip 1em plus 0.5em
  minus 0.4em\relax Cambridge University Press, 2004.

\bibitem{AbuMostafa-12}
Y.~S. Abu-Mostafa, M.~Magdon-Ismail, and H.-T. Lin, \emph{Learning From
  Data}.\hskip 1em plus 0.5em minus 0.4em\relax AMLBook, 2012.

\bibitem{SchmidtLeRoux-13}
M.~Schmidt and N.~L. Roux, ``Fast convergence of stochastic gradient descent
  under a strong growth condition,'' Aug. 2013.

\bibitem{Needell-14}
D.~Needell, R.~Ward, and N.~Srebro, ``Stochastic gradient descent, weighted
  sampling, and the randomized {K}aczmarz algorithm,'' in \emph{Advances in
  Neural Information Processing Systems}, vol.~27.\hskip 1em plus 0.5em minus
  0.4em\relax Curran Associates, Dec. 2014, pp. 1017--1025.

\bibitem{Ma-18}
S.~Ma, R.~Bassily, and M.~Belkin, ``The power of interpolation: Understanding
  the effectiveness of {SGD} in modern over-parametrized learning,'' in
  \emph{International Conference on Machine Learning}, vol.~80, Stockholm,
  Sweden, July 2018, pp. 3325--3334.

\bibitem{wirth1998calculation}
F.~Wirth, ``On the calculation of time-varying stability radii,''
  \emph{International Journal of Robust and Nonlinear Control: IFAC-Affiliated
  Journal}, vol.~8, no.~12, pp. 1043--1058, 1998.

\end{thebibliography}

\begin{IEEEbiographynophoto}
{Chung-Yi Lin} received the B.Sc. degree in electrical engineering and the M.Sc.
degree in communication engineering from the National Taiwan
University in 2015 and 2018, respectively.  He was a Ph.D. student in
electrical engineering at the California Institute of Technology between 2018 and 2020. He has been a quantitative researcher at the Kronos
Research based in Taiwan since 2021.
\end{IEEEbiographynophoto}

\begin{IEEEbiographynophoto}
   {Victoria Kostina}(S'12--M'14)
 received the bachelor's degree from Moscow Institute of Physics and Technology (MIPT) in 2004, the master's degree from University of Ottawa in 2006, and the Ph.D. degree from Princeton University in 2013.  During her studies at MIPT, she was affiliated with the Institute for Information Transmission Problems of the Russian Academy of Sciences. 
 
She is currently a Professor of electrical engineering and computing and mathematical sciences at California Institute of Technology. Her research interests include information theory, coding, control, learning, and communications. 
 She received the Natural Sciences and Engineering Research Council of Canada postgraduate scholarship during 2009--2012, Princeton Electrical Engineering Best Dissertation Award in 2013, Simons-Berkeley research fellowship in 2015 and the NSF CAREER award in 2017.  
 
 She is a Guest Editor for the IEEE Journal on Selected Areas in Information Theory 2022 special issue on Modern Compression. 
\end{IEEEbiographynophoto}

\begin{IEEEbiographynophoto}
    {Babak Hassibi}(Member, IEEE)
 was born in Tehran, Iran, in 1967. He received the B.S. degree from the University of Tehran in 1989, and the M.S. and Ph.D. degrees from Stanford University in 1993 and 1996, respectively, all in electrical engineering. 

He has been with the California Institute of Technology since January 2001, where he is currently the Mose and Lilian S. Bohn Professor of Electrical Engineering. From 2013-2016 he was the Gordon M. Binder/Amgen Professor of Electrical Engineering and from 2008-2015 he was Executive Officer of Electrical Engineering, as well as Associate Director of Information Science and Technology. From October 1996 to October 1998 he was a research associate at the Information Systems Laboratory, Stanford University, and from November 1998 to December 2000 he was a Member of the Technical Staff in the Mathematical Sciences Research Center at Bell Laboratories, Murray Hill, NJ. He has also held short-term appointments at Ricoh California Research Center, the Indian Institute of Science, and Linkoping University, Sweden. His research interests include communications and information theory, control and network science, and signal processing and machine learning. He is the coauthor of the books (both with A.H.~Sayed and T.~Kailath) {\em Indefinite Quadratic Estimation and Control: A Unified Approach to H$^2$ and H$^{\infty}$ Theories} (New York: SIAM, 1999) and {\em Linear Estimation} (Englewood Cliffs, NJ: Prentice Hall, 2000). He is a recipient of an Alborz Foundation Fellowship, the 1999 O. Hugo Schuck best paper award of the American Automatic Control Council (with H.~Hindi and S.P.~Boyd), the 2002 National Science Foundation Career Award, the 2002 Okawa Foundation Research Grant for Information and Telecommunications, the 2003 David and Lucille Packard Fellowship for Science and Engineering,  the 2003 Presidential Early Career Award for Scientists and Engineers (PECASE), and the 2009 Al-Marai Award for Innovative Research in Communications, and was a participant in the 2004 National Academy of Engineering ``Frontiers in Engineering'' program. 

He has been a Guest Editor for the IEEE Transactions on Information Theory special issue on ``space-time transmission, reception, coding and signal processing'' was an Associate Editor for Communications of the IEEE Transactions on Information Theory during 2004-2006, and is currently an Editor for the Journal ``Foundations and Trends in Information and Communication'' and for the IEEE Transactions on Network Science and Engineering. He is an IEEE Information Theory Society Distinguished Lecturer for 2016-2017 and was General Co-Chair if the 2020 IEEE International Symposium on Information Theory (ISIT 2020).
\end{IEEEbiographynophoto}

\end{document}